%% file: main_conic_blackwell.tex
\documentclass{article}

\usepackage[colorinlistoftodos,bordercolor=orange,backgroundcolor=orange!20,line
color=orange,textsize=scriptsize]{todonotes}
\catcode`\@=11
\catcode`\@=12
\usepackage{subcaption}

\usepackage{arxiv}
 \renewcommand{\citet}[1]{\cite{#1}}
 \usepackage{wrapfig}

\usepackage{amsmath, amsthm, amssymb, epsfig, sectsty, graphicx, float, url}
\usepackage{bm}
\usepackage{algorithm, algorithmicx}
\usepackage[noend]{algpseudocode}

\usepackage[utf8]{inputenc} 
\usepackage[T1]{fontenc}    
\usepackage{hyperref}       
\usepackage{url}            
\usepackage{booktabs}       
\usepackage{amsfonts}       
\usepackage{nicefrac}       
\usepackage{microtype}      
\usepackage{cleveref}

\usepackage{titlesec}
\titlespacing*{\paragraph} {0pt}{0.35ex plus 0.3ex minus .2ex}{1em}

\newcommand{\tb}[1]{{\color{black} #1}}
\newcommand{\tc}[1]{{\color{black} #1}}

\DeclareMathOperator*{\argmin}{arg\,min}

\newcommand{\R}{\mathbb{R}}

\newcommand{\V}{\mathcal{V}}
\newcommand{\N}{\mathbb{N}}

\newcommand{\XX}{\mathcal{X}}
\newcommand{\YY}{\mathcal{Y}}

\newcommand{\C}{\mathcal{C}}

\newcommand{\cbap}{{\sf CBA\textsuperscript{+}}}
\newcommand{\cba}{{\sf CBA}}
\newcommand{\rmp}{{\sf RM\textsuperscript{+}}}
\newcommand{\cfrp}{{\sf CFR\textsuperscript{+}}}
\newcommand{\rmm}{{\sf RM}}
\newcommand{\chp}{{\sf CHOOSEDECISION}}
\newcommand{\upp}{{\sf UPDATEPAYOFF}}

\theoremstyle{plain}
\newtheorem{theorem}{Theorem}[section]
\newtheorem{lemma}[theorem]{Lemma}

\newtheorem{proposition}[theorem]{Proposition}

\theoremstyle{definition}

\newtheorem{remark}[theorem]{Remark}

\title{Conic Blackwell Algorithm: Parameter-Free Convex-Concave Saddle-Point Solving}

\author{%
  Julien Grand-Cl{\'e}ment \\
  ISOM Department \& Hi!Paris\\
  HEC Paris\\
  \texttt{grand-clement@hec.fr} \\
   \And
   Christian Kroer \\
   IEOR Department \\
   Columbia University \\
   \texttt{christian.kroer@columbia.edu} \\
}

\begin{document}

\maketitle

\begin{abstract}
We develop new parameter-free and scale-free algorithms for solving convex-concave saddle-point problems. Our results are based on a new simple regret minimizer, the Conic Blackwell Algorithm$^+$ (\cbap), which attains $O(1/\sqrt{T})$ average regret. Intuitively, our approach generalizes to other decision sets of interest ideas from the Counterfactual Regret minimization (CFR$^+$) algorithm, which has very strong practical performance for solving sequential games on simplexes.
We show how to implement \cbap{} for the simplex, $\ell_{p}$ norm balls, and ellipsoidal confidence regions in the simplex, and we present numerical experiments for solving matrix games and distributionally robust optimization problems.
Our empirical results show that \cbap{} is a simple algorithm that outperforms state-of-the-art methods on synthetic data and real data instances, without the need for any choice of step sizes or other algorithmic parameters.
\end{abstract}

\input{text/intro}
\input{text/setup}
\input{text/cba}
\input{text/efficient_implementation}
\input{text/experiments}
\section{Conclusion}
We have introduced \cbap, a new algorithm for convex-concave saddle-point solving, that  is 1) simple to implement for many practical decision sets,  2) completely parameter-free and does not attempt to lear any step sizes, and 3) competitive with, or even better than, state-of-the-art approaches for the best choices of parameters,  both for matrix games, extensive-form games, and distributionally robust instances. \tb{Our paper is based on Blackwell approachability, which has been used to achieved important breakthroughs in poker AI in recent years, and we hope to generalize the use and implementation of this framework to other important problem instances.} Interesting future directions of research include developing a theoretical understanding of the improvements related to alternation in our setting,  designing efficient implementations for other widespread decision sets (e.g., based on Kullback-Leibler divergence or $\phi$-divergence),  and novel accelerated versions based on strong convex-concavity or optimistim.
\paragraph{Societal impact}
Our work enables faster  and simpler computation of  solutions of saddle-point problems, which have become widely used for applications in the industry.  There is a priori no direct negative societal consequence to this work, since our methods simply return the same solutions as previous algorithms but in a simpler and more efficient way.
{
\small 
\bibliographystyle{plainnat} 
\bibliography{FOM_RMDP}
}

\appendix

\input{text/appendix}
\end{document}

%% file: text/intro.tex
\section{Introduction}
\label{sec:intro}

We are interested in solving \emph{saddle-point problems} (SPPs) of the form
\begin{align}\label{eq:spp}
\min_{\bm{x} \in \XX}  \max_{\bm{y} \in \YY} F(\bm{x},\bm{y}),
\end{align}
where $\XX \subset \R^{n}, \YY \subset \R^{m}$ are convex, compact sets, and $F: \XX \times \YY \rightarrow \R$ is a subdifferentiable convex-concave function.
Convex-concave SPPs arise in a number of practical problems. For example, the problem of computing a Nash equilibrium of a zero-sum games can be formulated as a convex-concave SPP, and this is the foundation of most methods for solving sequential zero-sum games~\citep{stengel1996efficient,zinkevich2007regret,tammelin2015solving,kroer2018faster}. They also arise in imaging~\citep{ChambollePock2011}, $\ell_{\infty}$-regression \citep{sidford2018coordinate}, Markov Decision Processes \citep{Iyengar,Kuhn,sidford2018coordinate}, and in distributionally robust optimization, where the max term represents the distributional uncertainty~\citep{namkoong2016stochastic,ben2015oracle}. In this paper we propose efficient, \textit{parameter-free} algorithms for solving \eqref{eq:spp} in many settings, i.e., algorithms that do not require any tuning or choices of step sizes.

\paragraph{Repeated game framework}
One way to solve convex-concave SPPs is by viewing the SPP as a repeated game between two players, where each step $t$ consists of one player choosing $\bm{x}_t\in \XX$, the other player choosing $\bm{y}_t\in \YY$, and then the players observe the payoff $F(\bm{x}_t,\bm{y}_t)$.
If each player employs a regret-minimization algorithm, then a well-known folk theorem says that the uniform average strategy generated by two regret minimizers repeatedly playing an SPP against each other converges to a solution to the SPP.
We will call this the ``repeated game framework'' (see Section \ref{sec:setup}).
There are already well-known algorithms for instantiating the above repeated game framework for \eqref{eq:spp}. 
For example, one can employ the \emph{online mirror descent} (OMD) algorithm, which generates iterates as follows for the first player (and similarly for the second player):
\begin{align}\label{eq:omd}
       \bm x_{t+1} = \argmin_{\bm x \in \XX} \langle \eta \bm{f}_{t}, \bm{x} \rangle + D(\bm x \| \bm x_{t}),
\end{align}
where $\bm{f}_{t} \in \partial_{\bm{x}} F(\bm{x}_t, \bm{y}_t)$ ($\partial_{\bm{x}}$ denotes  here the set of subgradients as regards the variable $\bm{x}$), $D(\cdot \| \cdot)$ is a Bregman divergence which measures distance between pairs of points, and $\eta>0$ is an appropriate step size. By choosing $D$ appropriately for $\XX$, the update step \eqref{eq:omd} becomes efficient, and one can achieve an overall regret on the order of $O(\sqrt{T})$ after $T$ iterations. \tb{This regret can be achieved either by choosing a fixed step size $\eta = \alpha /L \sqrt{T}$, where an upper bound $L$ on the norms of the subgradients visited $\left( \bm{g}_{t} \right)_{t \geq 0}$, or by choosing adaptive step sizes $\eta_{t} = \alpha / \sqrt{t}$, for $\alpha>0$.  This is problematic,  as 1) the upper bound $L$ may be hard to obtain in many applications and may be too conservative in practice, and 2) adequately tuning the parameter $\alpha$ can be time- and resource-consuming, and even practically infeasible for very large instances, since we won't know if the step size will cause a divergence until late in the optimization process.} This is not just a theoretical issue,  as we highlight in our numerical experiments (Section \ref{sec:simu}) and in the appendices (Appendices \ref{app:OMD-etc}).
Similar results and challenges hold for the popular \emph{follow the regularized leader} (FTRL) algorithm (see Appendix \ref{app:OMD-etc}).

\tc{
The above issues can be circumvented by employing \emph{adaptive} variants of OMD or FTRL, which lead to parameter- and scale-free algorithms that estimate the parameters through the observed subgradients, e.g., AdaHedge for the simplex setting~\citep{de2014follow} or AdaFTRL for general compact convex decisions sets~\citep{orabona2015scale}.
Yet these adaptive variants have not seen practical adoption in large-sale game-solving, where regret-matching variants are preferred (see the next paragraph).
As we show in our experiments, adaptive variants of OMD and FTRL perform much worse than our proposed algorithms.
While these adaptive algorithms are referred to as parameter-free, this is only true in the sense that they are able to learn the necessary parameters. Our algorithm is parameter-free in the stronger sense that there are no parameters that even require learning. Formalizing this difference may be one interesting avenue for explaining the performance discrepancy on saddle-point problems.
}

\paragraph{Regret Matching}
In this paper, we introduce alternative regret-minimization schemes for instantiating the above framework. 
Our work is motivated by recent advances on solving large-scale zero-sum sequential games. 
In the zero-sum sequential game setting,  $\XX$ and $\YY$ are simplexes, the objective function becomes $F(\bm x, \bm y) = \langle \bm{x},\bm{Ay} \rangle$, and thus \eqref{eq:spp} reduces to a \emph{bilinear} SPP.
Based on this bilinear SPP formulation, the best practical methods for solving large-scale sequential games use the repeated game framework, where each player minimizes regret via some variant of \emph{counterfactual regret minimization} (CFR, \citep{zinkevich2007regret}).
Variants of CFR were used in every recent poker AI challenge, where poker AIs beat human poker players~\citep{bowling2015heads,moravvcik2017deepstack,brown2018superhuman,brown2019superhuman}.
The CFR framework itself is a decomposition of the overall regret of the bilinear SPP into local regrets at each decision point in a sequential game~\citep{farina2019online}.
The key to the practical performance of CFR-based algorithms seems to be three ingredients (beyond the CFR decomposition itself): 
(1) a particular regret minimizer called \emph{regret matching$^+$} (\rmp)~\citep{tammelin2015solving} which is employed at each decision point, 
(2) aggressive iterate averaging schemes that put greater weight on recent iterates (e.g. \emph{linear averaging}, which weights iterate at period $t$ by $2t / T(T+1)$), and 
(3) an alternation scheme where the updates of the repeated game framework are performed in an asymmetric fashion.
The CFR framework itself is specific to sequential bilinear games on simplexes, but these last three ingredients could potentially be generalized to other problems of the form \eqref{eq:spp}.
That is the starting point of the present paper.

The most challenging aspect of generalizing the above ingredients is that \rmp{} is specifically designed for minimizing regret over a simplex. 
However, many problems of the form \eqref{eq:spp} have convex sets $\XX,\YY$ that are not simplexes, e.g., box constraints or norm-balls for distributionally robust optimization~\citep{ben2015oracle}.
In principle, regret matching arises from a general theory called \emph{Blackwell approachability}~\citep{blackwell1956analog,hart2000simple}, and similar constructions can be envisioned for other convex sets.
However, in practice the literature has only focused on developing concrete implementable instantiations of Blackwell approachability for simplexes. 
A notable deviation from this is the work of \citet{abernethy2011blackwell}, who showed a general reduction between regret minimization over general convex compact sets and Blackwell approachability. However, their general reduction still does not yield a practically implementable algorithm: among other things,  their reduction relies on certain black-box projections that are not always efficient.
We show how to implement these necessary projections for the setting where $\XX$ and $\YY$ are simplexes, $\ell_p$ balls, and intersections of the $\ell_2$ ball with a hyperplane (with a focus on the case where an $\ell_2$ ball is intersected with a simplex,  which arises naturally as confidence regions). 
This yields an algorithm which we will refer to as the \emph{conic Blackwell algorithm} (\cba), which is similar in spirit to the regret matching algorithm, but crucially generalizes to other decision sets. Motivated by the practical performance of \rmp{}, we construct a variant of \cba{} which uses a thresholding operation similar to the one employed by {\rmp}. We call this algorithm \cbap{}.

\paragraph{Our contributions}
We introduce \cbap, a parameter-free algorithm which achieves $O(\sqrt{T})$ regret in the worst-case and generalizes the strong performances of \rmp{} for bilinear, simplex saddle-points solving to other more general settings.
A major selling point for \cbap{} is that it does not require any step size choices. Instead, the algorithm implicitly adjusts to the structure of the domains and losses by being instantiations of Blackwell's approachability algorithm.
After developing the \cbap{} algorithm, we then develop analogues of another crucial components for large-scale game solving. 
In particular, we prove a generalization of the folk theorem for the repeated game framework for solving \eqref{eq:spp}, which allows us to incorporate  polynomial averaging schemes such as linear averaging.
We then show that \cbap{} is compatible with linear averaging on the iterates. This mirrors the case of RM and \rmp{}, where only \rmp{} is compatible with linear averaging on the iterates.
We also show that both \cba{} and \cbap{} are compatible with polynomial averaging when simultaneously performed on the regrets and the iterates.
Combining all these ingredients, we arrive at a new class of algorithms for solving convex-concave SPPs. As long as efficient projection operations can be performed (which we show for several practical domains, including the simplex, $\ell_{p}$ balls and confidence regions in the simplex), one can apply the repeated game framework on \eqref{eq:spp}, where one can 
 use either \cba{} or \cbap{} as a regret minimizer for $\XX$ and $\YY$ along with polynomial averaging on the generated iterates to solve \eqref{eq:spp} at a rate of $O \left( 1/\sqrt{T} \right)$.
 
We highlight the practical efficacy of our algorithmic framework on several domains.
First, we solve two-player zero-sum matrix games and extensive-form games, where \rmp{} regret minimizer combined with linear averaging and alternation, and CFR$^+$, lead to very strong practical algorithms~\citep{tammelin2015solving}. We find that \cbap{} combined with linear averaging and alternation leads to a comparable performance in terms of the iteration complexity, and may even slightly outperform \rmp{} and CFR$^+$. On this simplex setting, we also find that \cbap{} outperforms both AdaHedge and AdaFTRL.
Second, we apply our approach to a setting where \rmp{} and CFR$^+$ do not apply: distributionally robust empirical risk minimization (DR-ERM) problems. Across two classes of synthetic problems and four real data sets, we find that our algorithm based on \cbap{} performs orders of magnitude better than online mirror descent and FTRL, as well as their optimistic variants, when using their theoretically-correct fixed step sizes.
Even when considering adaptive step sizes, or fixed step sizes that are up to $10,000$ larger than those predicted by theory, our \cbap{} algorithm performs better, with only a few cases of comparable performance (at step sizes that lead to divergence for some of the other non-parameter free methods). The fast practical performance of our algorithm, combined with its simplicity and the total lack of step sizes or parameters tuning, suggests that it should be seriously considered as a practical approach for solving  convex-concave SPPs in various settings.

Finally, we make a brief note on accelerated methods. Our algorithms have a rate of convergence towards a saddle point of  $O(1/\sqrt{T})$, similar to OMD and FTRL. 
In theory, it is possible to obtain a faster $O \left( 1/T \right)$ rate of convergence when $F$ is differentiable with Lipschitz gradients, for example via mirror prox~\citep{nemirovski2004prox} or other primal-dual algorithms~\citep{ChambollePock16}. 
However, our experimental results show that \cbap{} is faster than optimistic variants of FTRL and OMD~\citep{syrgkanis2015fast}, the latter being almost identical to the mirror prox algorithm, and both achieving $O(1/T)$ rate of convergence.
A similar conclusion has been drawn in the context of sequential game solving, where the fastest $O(1/\sqrt{T})$ CFR-based algorithms have better practical performance than the theoretically-superior $O \left( 1/T \right)$-rate methods~\citep{kroer2018faster,kroer2018solving}.
\tc{
In a similar vein, using \emph{error-bound conditions}, it is possible to achieve a linear rate, e.g., when solving bilinear saddle-point problems over polyhedral decision sets using the extragradient method~\citep{tseng1995linear} or optimistic gradient descent-ascent~\citep{wei2020linear}. However, these linear rates rely on unknown constants, and may not be indicative of practical performance.
}

%% file: text/setup.tex
\section{Game setup and Blackwell Approachability}\label{sec:setup}
As stated in \cref{sec:intro}, we will solve \eqref{eq:spp} using a repeated game framework. The first player chooses strategies from $\XX$ in order to minimize the sequence of payoffs in the repeated game, while the second player chooses strategies from $\YY$ in order to maximize payoffs.
There are $T$ iterations with indices $t=1,\ldots,T$. In this framework, each iteration $t$ consists of the following steps:
\begin{enumerate}
    \item Each player chooses strategies $\bm x_t\in \XX, \bm y_t \in \YY$
    \item First player observes $\bm{f_t} \in \partial_{\bm x} F(\bm x_t,\bm y_t)$ and uses $\bm{f}_{t}$ when computing the next strategy
    \item Second player observes $\bm{g_t} \in \partial_{\bm y} F(\bm x_t,\bm y_t)$ and uses $\bm{g}_{t}$ when computing the next strategy
\end{enumerate}

The goal of each player is to minimize their regret $R_{T,\bm{x}},R_{T,\bm{y}}$ across the $T$ iterations:
\begin{align*}
R_{T,\bm{x}}  = \sum_{t=1}^{T} \langle \bm f_t, \bm x_t \rangle - \min_{\bm{x} \in \XX} \sum_{t=1}^{T} \langle \bm f_t, \bm x \rangle ,\quad
R_{T,\bm{y}} = \max_{\bm{y} \in \YY} \sum_{t=1}^{T} \langle \bm g_t ,\bm{y} \rangle - \sum_{t=1}^{T} \langle \bm g_t,\bm{y}_{t} \rangle.
\end{align*}
The reason this repeated game framework leads to a solution to the SPP problem~\eqref{eq:spp} is the following folk theorem. Relying on $F$ being convex-concave and subdifferentiable,  it connects the regret incurred by each player to the duality gap in \eqref{eq:spp}.
\begin{theorem}[Theorem 1, \cite{kroer2020ieor8100}]\label{thm:folk theorem}
Let $\left( \bar{\bm{x}}_{T},\bar{\bm{y}}_{T} \right) = \dfrac{1}{T}\sum_{t=1}^{T} \left(\bm{x}_{t}, \bm{y}_{t} \right)$ for any $\left(\bm{x}_{t}\right)_{t \geq 1},\left(\bm{y}_{t}\right)_{t \geq 1}$.
Then 
\[\max_{\bm{y} \in \YY} F(\bar{\bm{x}}_{T},\bm{y}) - \min_{\bm{x} \in \XX} F(\bm{x},\bar{\bm{y}}_{T}) \leq (R_{T,\bm{x}} + R_{T,\bm{y}})/T. \]
\end{theorem}
Therefore,  when each player runs a regret minimizer that guarantees regret on the order of $O(\sqrt{T})$,   $\left(\bar{\bm{x}}_{T}, \bar{\bm{y}}_{T} \right)_{T \geq 0}$ converges to a solution to \eqref{eq:spp} at a rate of $O\left(1/\sqrt{T}\right)$.
Later we will show a generalization of Theorem \ref{thm:folk theorem} that will allow us to incorporate more aggressive averaging schemes that put additional weight on the later iterates.
Given the repeated game framework, the next question becomes which algorithms to employ in order to minimize regret for each player. As mentioned in Section \ref{sec:intro}, for zero-sum games, variants of regret matching are used in practice. 

\paragraph{Blackwell Approachability}
Regret matching arises from the \emph{Blackwell approachability} framework~\citep{blackwell1956analog}.  
In Blackwell approachability, a decision maker repeatedly takes decisions $\bm x_t$ from some convex decision set $\XX$ (this set plays the same role as $\XX$ or $\YY$ in \eqref{eq:spp}). After taking decision $\bm x_t$ the player observes a vector-valued affine payoff function $\bm u_t(\bm x) \in \mathbb R^n$. 
The goal for the decision maker is to force the average payoff $\frac{1}{t} \sum_{\tau = 1}^t \bm{u}_\tau(\bm x_\tau)$ to approach some convex target $\mathcal S$.
Blackwell proved that a convex target set $\mathcal S$ can be approached if and only if for every halfspace $\mathcal H\supseteq \mathcal S$, there exists $\bm x \in \XX$ such that for every possible payoff function $\bm u(\cdot)$, $\bm u(\bm x)$ is guaranteed to lie in $\mathcal H$. The action $\bm x$ is said to \emph{force} $\mathcal H$.
Blackwell's proof is via an algorithm: at iteration $t$, his algorithm projects the average payoff $\bar{\bm{u}} = \frac{1}{t-1} \sum_{\tau = 1}^{t-1} \bm{u}_\tau(\bm x_\tau)$ onto $\mathcal S$, and then the decision maker chooses an action $\bm x_t$ that forces the tangent halfspace to $\mathcal S$ generated by the normal $\bar{\bm{u}} - \pi_{\mathcal S}(\bar{\bm{u}})$, where $\pi_{\mathcal S}(\bar{\bm{u}})$ is the orthogonal projection of $\bar{\bm{u}}$ onto $\mathcal S$.
We call this algorithm \emph{Blackwell's algorithm}; it approaches $\mathcal S$ at a rate of $O(1/\sqrt{T})$.
It is important to note here that Blackwell's algorithm is rather a meta-algorithm than a concrete algorithm. Even within the context of Blackwell's approachability problem, one needs to devise a way to compute the forcing actions needed at each iteration, i.e., to compute $\pi_{\mathcal S}(\bar{\bm{u}})$.

\paragraph{Details on Regret Matching}
Regret matching arises by instantiating Blackwell approachability with the decision space $\XX$ equal to the simplex $\Delta(n)$, the target set $\mathcal S$ equal to the nonpositive orthant $\mathbb R_{-}^n$, and the vector-valued payoff function $\bm u_t(\bm x_t) = \bm f_t - \langle \bm f_t,\bm x_t \rangle \bm{e}$ equal to the regret associated to each of the $n$ actions (which correspond to the corners of $\Delta(n)$). 
Here $\bm{e} \in \R^{n}$ has one on every component.
 \citet{hart2000simple} showed that with this setup, playing each action with probability proportional to its positive regret up to time $t$ satisfies the forcing condition needed in Blackwell's algorithm.
 Formally, regret matching (\rmm) keeps a running sum $\bm r_t = \sum_{\tau = 1}^{t} \left( \bm f_{\tau} - \langle\bm f_{\tau}, \bm x_\tau \rangle \bm{e} \right)$, and then action $i$ is played with probability $\bm x_{t+1,i} = [\bm r_{t,i}]^+ / \sum_{i=1}^n [\bm r_{t,i}]^+$, where $[\cdot]^+$ denotes thresholding at zero.
 By Blackwell's approachability theorem, this algorithm converges to zero average regret at a rate of $O(1/\sqrt{T})$.
 In zero-sum game-solving, it was discovered that a variant of regret matching leads to extremely strong practical performance (but the same theoretical rate of convergence). In regret matching$^+$ (\rmp), the running sum is thresholded at zero at every iteration: $\bm r_t = [\bm r_{t-1} + \bm f_{t} - \langle\bm f_{t}, \bm x_t \rangle \bm{e} ]^+$, and then actions are again played proportional to $\bm r_t$.
In the next section, we describe a more general class of regret-minimization algorithms based on Blackwell's algorithm for general sets $\XX$, introduced in \citet{abernethy2011blackwell}.
Note that a similar construction of a general class of algorithms can be achieved through the \emph{Lagrangian Hedging} framework of \citet{gordon2007no}. It would be interesting to construct a \cbap-like algorithm and efficient projection approaches for this framework as well.

%% file: text/cba.tex
\section{Conic Blackwell Algorithm}\label{sec:CBA}
We present the Conic Blackwell Algorithm Plus (\cbap), a no-regret algorithm which uses a variation of Blackwell's approachability procedure~\citep{blackwell1956analog} to perform regret minimization on general convex compact decision sets $\XX$. We will assume that losses are coming from a bounded set; this occurs, for example, if there exists $L_x,L_y$ (that we do not need to know), such that
\begin{equation}\label{eq:definition-Lx-Lx}
\|  \bm{f} \| \leq L_x, \| \bm{g} \| \leq L_y, \; \forall \; \bm{x}\in \XX, \bm{y} \in \YY, \forall \; \bm{f} \in \partial_{\bm{x}} F(\bm{x}, \bm{y}), \forall \; \bm{g} \in \partial_{\bm{y}} F(\bm{x}, \bm{y}).
\end{equation}

{\cbap} is best understood as a combination of two steps. The first is the basic CBA algorithm, derived from Blackwell's algorithm, which we describe next.
To convert Blackwell's algorithm to a regret minimizer on $\XX$, we use the reduction from \citet{abernethy2011blackwell}, which considers the conic hull $\C = \textrm{cone}(\{\kappa\} \times \XX)$ where  $\kappa = \max_{\bm{x} \in \XX} \| \bm{x} \|_{2}$. 
The Blackwell approachability problem is then instantiated with $\XX$ as the decision set, target set equal to the polar  $\C^\circ = \{ \bm{z} : \langle \bm{z}, \bm{\hat z} \rangle \leq 0, \forall \bm{\hat z} \in \C\}$ of $\C$, and payoff vectors $(\langle\bm{f}_t,\bm{x}_t \rangle, -\bm{f}_t)$.
The conic Blackwell algorithm (\cba) is implemented by projecting the average payoff vector onto $\C$, calling this projection $\alpha(\kappa,\bm{x})$ with $\alpha \geq 0$ and $\bm{x} \in \XX$, and playing the action $\bm{x}$.

The second step in {\cbap} is to modify \cba{} to make it analogous to {\rmp} rather than to \rmm. To do this, the algorithm does not keep track of the average payoff vector. Instead, we keep a running aggregation of the payoffs, where we always add the newest payoff to the aggregate, and then project the aggregate onto $\C$.
More concretely, pseudocode for {\cbap} is given in Algorithm \ref{alg:CBAp}.
This pseudocode relies on two functions:
$\chp_{\cbap}: \R^{n+1} \rightarrow \R^{n}$, which maps the aggregate payoff vector $\bm u_t$ to a decision in $\XX$, and $\upp_{\cbap}$ which controls how we aggregate payoffs.
Given an aggregate payoff vector $\bm{u}=(\tilde{u},\hat{\bm{u}}) \in \R \times \R^{n}$,  we have
\[\chp_{\cbap}(\bm{u})=(\kappa/\tilde{u}) \hat{\bm{u}}.\]
If $\tilde{u}=0$, we just let $\chp_{\cbap}(\bm{u})=\bm{x}_{0}$ for some chosen $\bm{x}_{0} \in \XX$. 

The function $\upp_{\cbap}$ is implemented by adding the most recent payoff to the aggregate payoffs, and then projecting onto $\C$.
More formally, it is defined as
\[
\upp_{\cbap}(\bm{u}, \bm{x},\bm{f},\omega,S) =\pi_{\C} \left( \frac{S}{S+\omega}\bm{u} + \frac{\omega}{S+\omega} \left( \langle \bm{f},\bm{x} \rangle / \kappa, 
- \bm{f} \right) \right),
\]
where $\omega$ is the weight assigned to the most recent payoff and $S$ the weight assigned to the previous aggregate payoff $\bm{u}$. Because of the projection step in $\upp_{\cbap}$, we always have $\bm{u} \in \C$,  which in turn guarantees that $\chp_{\cbap}(\bm{u}) \in \XX$, since $\C = \textrm{cone}(\{\kappa\} \times \XX)$.
\begin{algorithm}[t]
\caption{Conic Blackwell Algorithm Plus (\cbap)}\label{alg:CBAp}
\begin{algorithmic}[1]
\State \textbf{Input} A convex, compact set $\XX \subset \R^{n}$, $\kappa = \max \{ \| \bm{x} \|_{2} \; | \; \bm{x} \in \XX \}$.
\State \textbf{Algorithm parameters} Weights $\left(\omega_{\tau} \right)_{\tau \geq 1} \in \R^{\N}$.
\State \textbf{Initialization} $t=1$,  $\bm{x}_{1} \in \XX$. 
\State Observe $\bm{f}_{1}$ then set
$\bm{u}_{1} = \left( \langle \bm{f}_{1},\bm{x}_{1} \rangle / \kappa,  - \bm{f}_{1} \right) \in \R \times \R^{n}$.
\For{$t \geq 1$}
\State Choose $\bm{x}_{t+1} = \chp_{\cbap}(\bm{u}_{t})$. \label{step:alg:projection}
\State Observe the loss $\bm{f}_{t+1} \in \R^{n}$.
\State Update $\bm{u}_{t+1} = \upp_{\cbap}(\bm{u}_{t}, \bm{x}_{t+1},\bm{f}_{t+1},\omega_{t+1},\sum_{\tau=1}^{t} \omega_{\tau}).$ \label{alg:step:update-u}
\State Increment $t \leftarrow t+1.$
\EndFor
\end{algorithmic}
\end{algorithm}
\tb{
Let us give some intuition on the effect of projection onto $\C$. In a geometric sense, 
it is easier to visualize things in $\R^{2}$ with $\C = \R^{2}_{+}$ and $\C^{\circ} = \R^{2}_{-}$. The projection on  $\C$ moves the vector along the edges of $\C^{\circ}$, maintaining the distance to $\C^{\circ}$ and moving toward the vector $\bm{0}$. This is illustrated in Figure \ref{fig:projection-cone} in Appendix \ref{app:figure-projection-cone}. From a game-theoretic standpoint, the projection on $\C=\R^{2}_{+}$ eliminates the components of the payoffs that are negative. It enables CBA+ to be less pessimistic than CBA, which may accumulate negative payoffs on actions for a long time and never resets the components of the aggregated payoff to $0$, leading to some actions being chosen less frequently. }
 
We will see in the next section that \rmp{} is related to \cbap{} but replaces the exact projection step $\pi_{\C}(\bm{u})$ in $\upp_{\cbap}$ by a suboptimal solution to the projection problem.
Let us also note the difference between \cbap{} and  the algorithm introduced in \cite{abernethy2011blackwell}, which we have called \cba.  
\cba{} uses different \upp{} and \chp{} functions. 
In \cba{} the payoff update is defined as \[\upp_{\cba}(\bm{u}, \bm{x},\bm{f},\omega,S) = \frac{S}{S+\omega} \bm{u} +  \frac{\omega}{S+\omega} \left( \langle \bm{f} / \kappa,\bm{x} \rangle, 
- \bm{f} \right).\]
Note in particular the lack of projection as compared to \cbap{}, analogous to the difference between \rmm{} and \rmp. The $\chp_{\cba}$ function then requires a projection onto $\C$: \[\chp_{\cba}(\bm{u}) = \chp_{\cbap}\left(\pi_{\C}(\bm{u}) \right).\]
Based upon the analysis in \cite{blackwell1956analog},  \cite{abernethy2011blackwell} show that \cba{} with uniform weights (both on payoffs and decisions) guarantees $O(1/\sqrt{T})$ average regret.
The difference between \cbap{} and \cba{} is similar to the difference between the \rmm{} and \rmp{} algorithms. 
In practice, \rmp{} performs significantly better than \rmm{} for solving matrix games, when combined with \emph{linear averaging} on the decisions (as opposed to the uniform averaging used in Theorem \ref{thm:folk theorem}).  In the next theorem, we show that \cbap{} is compatible with linear averaging on decisions only. We present a detailed proof in Appendix \ref{app:proof-lin-avg}.
\begin{theorem}\label{th:cbap-linear-averaging-only-policy}
 Consider $\left(\bm{x}_{t} \right)_{t \geq 0}$ generated by \cbap{} with uniform weights: $\omega_{\tau} =1, \forall \; \tau \geq 1$. \tb{Let $L = \max \{ \| \bm{f}_{t} \|_{2} \; | \; t \geq 1\}$ and $\kappa = \max \{ \| \bm{x} \|_{2} \; | \; \bm{x} \in \XX\}$.}
Then 
\[ \dfrac{\sum_{t=1}^{T} t \langle \bm{f}_{t},\bm{x}_{t} \rangle -  \min_{\bm{x} \in \XX} \sum_{t=1}^{T} t \langle \bm{f}_{t},\bm{x} \rangle}{T(T+1)} = O \left(\kappa L /\sqrt{T}\right).\]
\end{theorem}

Note that in Theorem~\ref{th:cbap-linear-averaging-only-policy}, we have \textit{uniform} weights on the sequence of payoffs $\left( \bm{u}_t \right)_{t \geq 0}$, but \textit{linearly increasing} weights on  the sequence of decisions.  
The proof relies on properties specific to \cbap{}, and it does not extend to \cba{}. Numerically it also helps \cbap{} but not \cba{}.
In Appendix~\ref{app:proof-lin-avg}, we show that both \cba{} and \cbap{} achieve $O \left(\kappa L /\sqrt{T}\right)$ convergence rates when using a weighted average on \textit{both} the decisions and the payoffs  (Theorems~\ref{th:cba-linear-averaging-both}-\ref{th:cbap-linear-averaging-both}). 
In practice,  using linear averaging only on the decisions, as in Theorem \ref{th:cbap-linear-averaging-only-policy}, performs vastly better than linear averaging on both decisions and payoffs.  We present empirical evidence of this in Appendix~\ref{app:proof-lin-avg}. 

\tb{
We can compare the $O\left(\kappa L /\sqrt{T}\right)$ average regret for \cbap{} with the $O\left(\Omega L /\sqrt{T}\right)$ average regret for \ref{alg:OMD}~\citep{nemirovsky1983problem,BenTal-Nemirovski} and \ref{alg:FTRL}~\citep{abernethy2009competing,mcmahan2011follow}, where $\Omega = \max \{ \| \bm{x} - \bm{x}'\|_{2} | \bm{x},\bm{x}' \in \XX\}.$ We can always recenter $\XX$ to contain $\bm{0}$, in which case the bounds for OMD/FTRL and \cbap{} are equivalent since $\kappa \leq \Omega \leq 2 \kappa$. Note that the bound on the average regret for Optimistic OMD (\ref{alg:pred-OMD}, \cite{chiang2012online}) and Optimistic FTRL (\ref{alg:pred-FTRL}, \cite{rakhlin2013online})  is $O\left(\Omega^{2} L /T\right)$ in the game setup, a priori better than the bound for \cbap{} as regards the number of iterations $T$. Nonetheless, we will see in Section \ref{sec:simu} that the empirical performance of \cbap{} is better than that of $O(1/T)$ methods.  
A similar situation occurs for \rmp{} compared to \ref{alg:pred-OMD} and \ref{alg:pred-FTRL} for solving poker games~\citep{farina2019optimistic,kroer2018faster}.}

The following theorem gives the convergence rate of \cbap{} for solving saddle-points~\eqref{eq:spp},  based on our convergence rate on the regret of each player (Theorem \ref{th:cbap-linear-averaging-only-policy}). The proof is in Appendix \ref{app:proof-folk-final}.
\begin{theorem}\label{th:folk-final}
Let $\left( \bar{\bm{x}}_{T},\bar{\bm{y}}_{T} \right) =2 \sum_{t=1}^{T} t \left(\bm{x}_{t}, \bm{y}_{t} \right)/(T(T+1)),$ where $\left(\bm{x}_{t} \right)_{t \geq 0}, \left(\bm{y}_{t}\right)_{t \geq 0}$ are generated by the repeated game framework with \cbap{} with uniform weights: $\omega_{\tau} =1, \forall \; \tau \geq 1$.
\tb{Let $L = \max \{ L_{x},L_{y}\}$ defined in \eqref{eq:definition-Lx-Lx} and $\kappa = \max \{ \max \{\| \bm{x} \|_{2},\| \bm{y} \|_{2}\} \; | \; \bm{x} \in \XX,\bm{y} \in \YY\}$.}
Then 
\[ \max_{\bm{y} \in \YY} F(\bar{\bm{x}}_{T},\bm{y}) - \min_{\bm{x} \in \XX} F(\bm{x},\bar{\bm{y}}_{T}) = O \left(\kappa L /\sqrt{T} \right). \]
\end{theorem}

%% file: text/efficient_implementation.tex
\subsection{Efficient implementations of \cbap{}}\label{sec:efficient-projection}
To obtain an implementation of \cbap{} and \cba, we need to efficiently resolve the functions $\chp_{\cbap}{}$ and $\upp_{\cbap{}}$. In particular, we need to compute $\pi_{\C}(\bm{u})$, the orthogonal projection of $\bm{u}$ onto the cone $\C$, where $\C  =\textrm{cone}(\{\kappa\} \times \XX)$:
\begin{equation}\label{eq:projection-onto-C}
\pi_{\C}(\bm{u}) \in \arg \min_{\bm{y} \in\C} \| \bm{y} - \bm{u} \|_{2}^{2}.
\end{equation}
Even for \cba{} this problem must be resolved, since \cite{abernethy2011blackwell} did not study whether~\eqref{eq:projection-onto-C} can be efficiently solved.
It turns out that \eqref{eq:projection-onto-C} can be computed in closed-form or quasi closed-form for many  decision sets $\XX$ of interest.  Interestingly,  parts of the proofs rely on \textit{Moreau's Decomposition Theorem}~\citep{combettes2013moreau}, which states that $\pi_{\C}(\bm{u})$ can be recovered from $\pi_{\C^{\circ}}(\bm{u})$ and vice-versa, because $\pi_{\C}(\bm{u})+\pi_{\C^{\circ}}(\bm{u})=\bm{u}.$ We present the detailed complexity results and the proofs in Appendix \ref{app:proofs-of-projections}.  

\paragraph{Simplex} \tb{$\XX = \Delta(n)$ is the classical setting used for matrix games. Also, for  extensive-form games, CFR decomposes the decision sets (treeplexes) into a set of regret minimization problems over the simplex~\citep{farina2019online}.} Here, $n$ is the number of actions of a player and $\bm{x} \in \Delta(n)$ represents a randomized strategy. In this case, $\pi_{\C}(\bm{u})$ can be computed in $O(n\log(n))$.  
\tb{
Note that \rmm{} and \rmp{} are obtained by choosing a suboptimal solution to \eqref{eq:projection-onto-C},  avoiding the $O(n\log(n))$ sorting operation, whereas \cba{} and \cbap{} choose optimally (see Appendix \ref{app:proofs-of-projections}). Thus, \rmm{} and \rmp{} can be seen as approximate versions of \cba{} and \cbap{}, where \eqref{eq:projection-onto-C} is solved approximatively at every iteration.}
 In our numerical experiments,  we will see that \cbap{} slightly outperforms \rmp{} and CFR$^+$ in terms of iteration count.
\paragraph{$\ell_{p}$ balls} This is when $\XX = \{ \bm{x} \in \R^{n} \; | \ ; \| \bm{x} \|_{p} \leq 1\}$ with $p  \geq 1$ or $p=\infty$. This is of interest for instance in distributionally robust optimization \citep{ben2015oracle,namkoong2016stochastic},  $\ell_{\infty}$ regression \citep{sidford2018coordinate} and saddle-point reformulation of Markov Decision Process \citep{jin2020efficiently}.  For $p=2$, we can compute $\pi_{\C}(\bm{u})$ in closed-form, i.e., in $O(n)$ arithmetic operations. For $p \in \{1,\infty\}$,  we can compute $\pi_{\C}(\bm{u})$ in $O(n \log(n))$ arithmetic operations using a sorting algorithm.
\paragraph{Ellipsoidal confidence region in the simplex} Here, $\XX$ is an \textit{ellipsoidal subregion of the simplex}, defined as 
$\XX = \{ \bm{x} \in \Delta(n) \; | \; \| \bm{x} - \bm{x}_{0} \|_{2} \leq \epsilon_{x} \}$. 
This type of decision set is widely used because it is associated with confidence regions when estimating a probability distribution from observed data~\citep{Iyengar,bertsimas2019probabilistic}. It can also be used in the Bellman update for robust Markov Decision Process \citep{Iyengar,Kuhn,GGC}. 
We also assume that the confidence region is ``entirely contained in the simplex'': $\{ \bm{x} \in \R^{n} | \bm{x}^{\top}\bm{e}=1 \} \bigcap \{ \bm{x} \in \R^{n} \; | \; \| \bm{x}- \bm{x}_{0} \|_{2} \leq \epsilon_{x} \} \subseteq \Delta(n)$, to avoid degenerate components. In this case, using a change of basis we show that it is possible to compute $\pi_{\C}(\bm{u})$ in closed-form, i.e., in $O(n)$ arithmetic operations.
\tb{
\paragraph{Other potential sets of interests} Other important decision sets include sets based on Kullback-Leibler divergence $\{ \bm{x} \in \Delta(n) \; | \; KL\left(\bm{x},\bm{x}_{0} \right) \leq \epsilon_{x} \}$, or, more generally, $\phi$-divergence~\citep{ben2013robust}. For these sets, we did not find a closed-form solution to the projection problem \eqref{eq:projection-onto-C}. Still, as long as the domain $\XX$ is a convex set, computing $\pi_{\C}(\bm{u})$ remains a convex problem, and it can be solved efficiently with solvers, although this results in a slower algorithm than with closed-form computations of $\pi_{\C}(\bm{u})$. 
}

%% file: text/experiments.tex
\section{Numerical experiments}\label{sec:simu}
In this section we  investigate the practical performances of our algorithms on several instances of saddle-point problems. We start by comparing \cbap{} with \rmp{} in the matrix  and extensive-form games setting.  We then turn to comparing our algorithms on instances from the distributionally robust optimization literature. The code for all experiments is available in the supplemental material.

   \vspace{-8mm}

\subsection{Matrix games on the simplex and Extensive-Form Games}
\label{sec:simu-bilinear-games-on-simplex}
   \vspace{-5mm}

Since the motivation for \cbap{} is to obtain the strong empirical performances of \rmp{} and CFR$^+$ on other decision sets than the simplex, we start by checking that \cbap{} indeed provide comparable performance on simplex settings. 
We compare these methods on matrix games
\[ \min_{\bm{x} \in \Delta(n) } \max_{\bm{y} \in \Delta(m) } \langle \bm{x},\bm{Ay} \rangle,\]
where $\bm{A}$ is the matrix of payoff, and on extensive-form games (EFGs). 
EFGs can also be written as SPPs with bilinear objective and $\XX,\YY$ polytopes encoding the players' space of sequential strategies~\citep{stengel1996efficient}.  
EFGs can be solved via simplex-based regret minimization by using the counterfactual regret minimization (CFR) framework to decompose regrets into local regrets at each simplex. 
Explaining CFR is beyond the scope of this work; we point the reader to~\citep{zinkevich2007regret} or newer explanations~\citep{farina2019regret,farina2019online}.
For matrix games, we generate 70 synthetic $10$-dimensional matrix games with $A_{ij} \sim U[0,1]$ and  compare the most efficient algorithms for matrix games with linear averaging: \cbap{} and \rmp. \tb{We also compare with two other scale-free no-regret algorithms, AdaHedge~\citep{de2014follow} and AdaFTRL~\citep{orabona2015scale}.}
Figure~\ref{fig:subfig:matrix-uniform-parameters} presents the duality gap of the current solutions vs.  the number of steps.
Here, both \cbap{} and \rmp{} use \emph{alternation}, which is a trick that is well-known to improve the performances of \rmp{}~\citep{tammelin2015solving}, where the repeated game framework is changed such that players take turns updating their strategies, rather than performing these updates simultaneously, see Appendix \ref{app:numerical-setup-details} for details.\footnote{We note that \rmp{} is guaranteed to retain its convergence rate under alternation. In contrast, we leave resolving this property for \cbap{} to future work.}

For EFGs, we compare \cbap and CFR$^+$ on many poker AI benchmark instances, including {\sf Leduc, Kuhn, search games} and {\sf sheriff} (see  \citet{farina2021faster} for game descriptions). We present our results in Figures \ref{fig:subfig:efg_leduc_2pl_9ranks}-\ref{fig:subfig:efg_leduc_2pl_5ranks}.
Additional details and experiments for EFGs are presented in Appendix~\ref{app:simu-EFG}.
Overall, we see in Figure~\ref{fig:matrix-games-simplex} that \cbap may slightly outperform \rmp{} and CFR$^+$, two of the strongest algorithms for matrix games and EFGs, which were shown to achieve the best empirical performances compared to a wide range of algorithms, including Hedge and other first-order methods~\citep{kroer2020ieor8100,kroer2018solving,farina2019optimistic}.  For matrix games, AdaHedge and AdaFTRL are both outperformed by \rmp{} and \cbap ; we present more experiments to compare \rmp{} and \cbap{} in Appendix \ref{app:comparing-rm-cba}, and more experiments with matrix games in Appendix \ref{app:simu-matrix-games}. Recall that our goal is to generalize these strong performance to other settings: we present our numerical experiments for solving distributionally robust optimization problems in the next section. 

\begin{figure}[H]
       \vspace{-2mm}
\centering
     \begin{subfigure}{0.24\textwidth}
\includegraphics[width=1.0\linewidth]{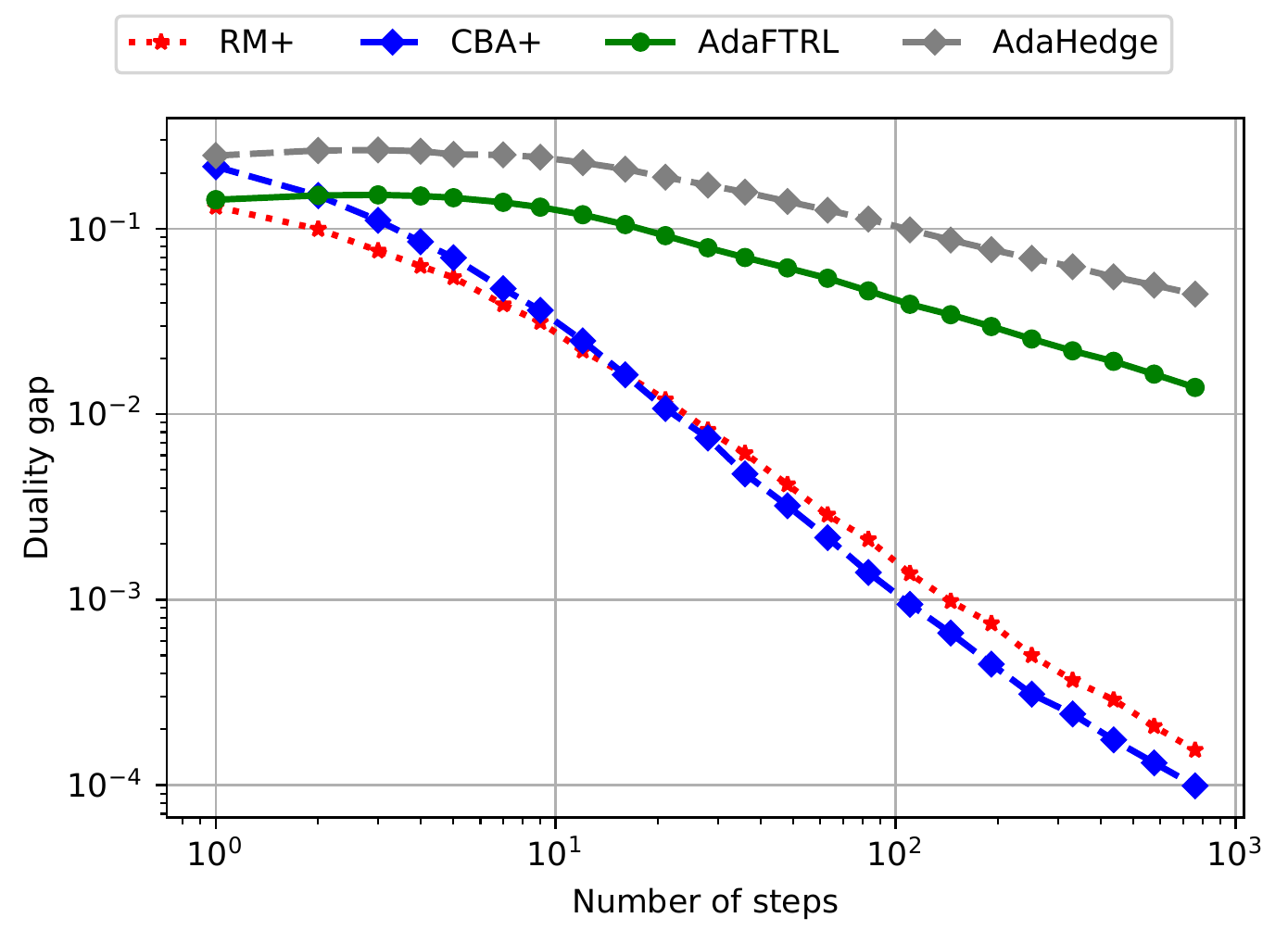}
         \caption{Matrix game.}\label{fig:subfig:matrix-uniform-parameters}
  \end{subfigure}
     \begin{subfigure}{0.24\textwidth}
         \includegraphics[width=1.0\linewidth]{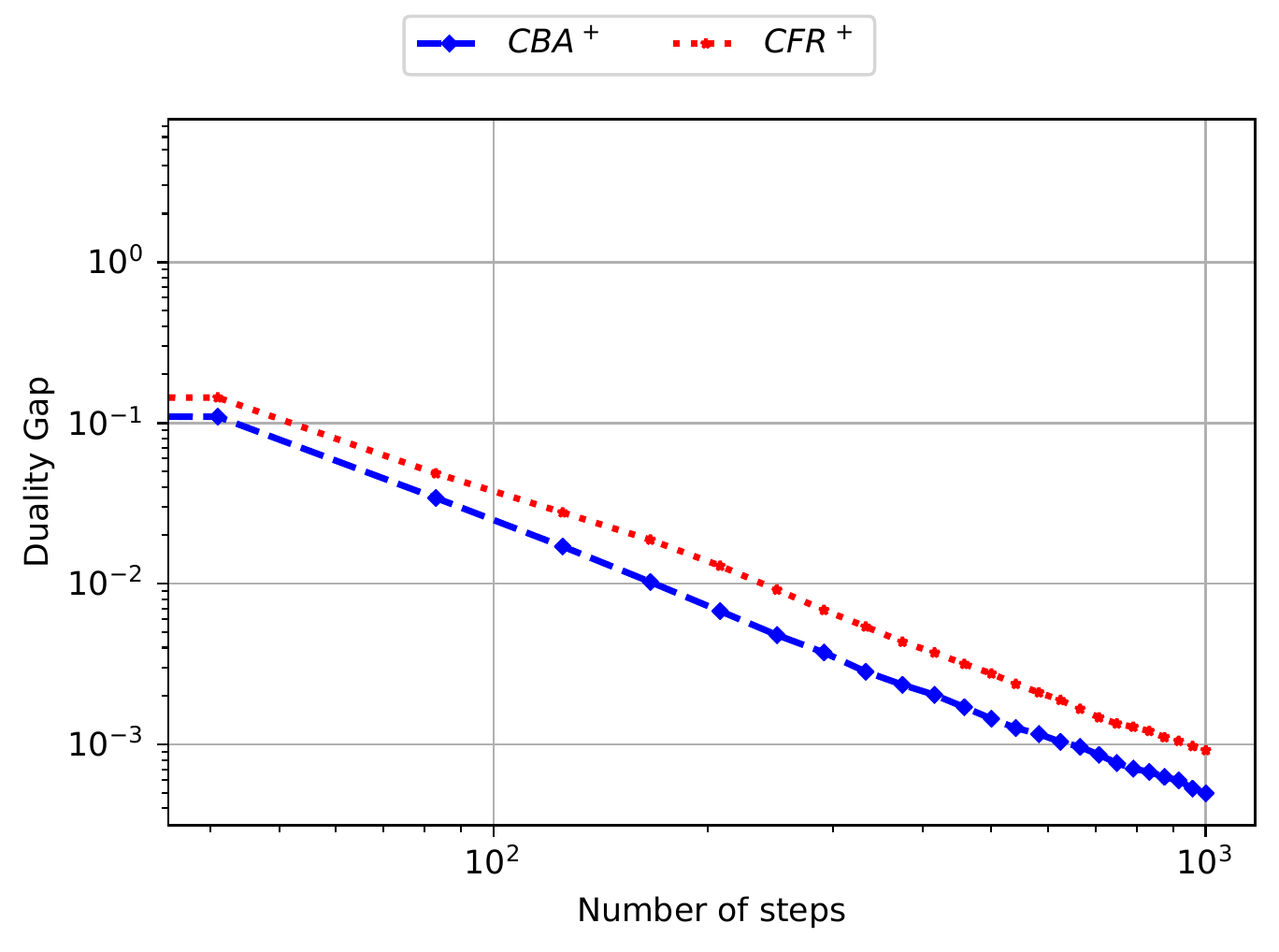}
         \caption{EFG: Leduc, 2 players, 9 ranks.}\label{fig:subfig:efg_leduc_2pl_9ranks}
  \end{subfigure}
     \begin{subfigure}{0.24\textwidth}
         \includegraphics[width=1.0\linewidth]{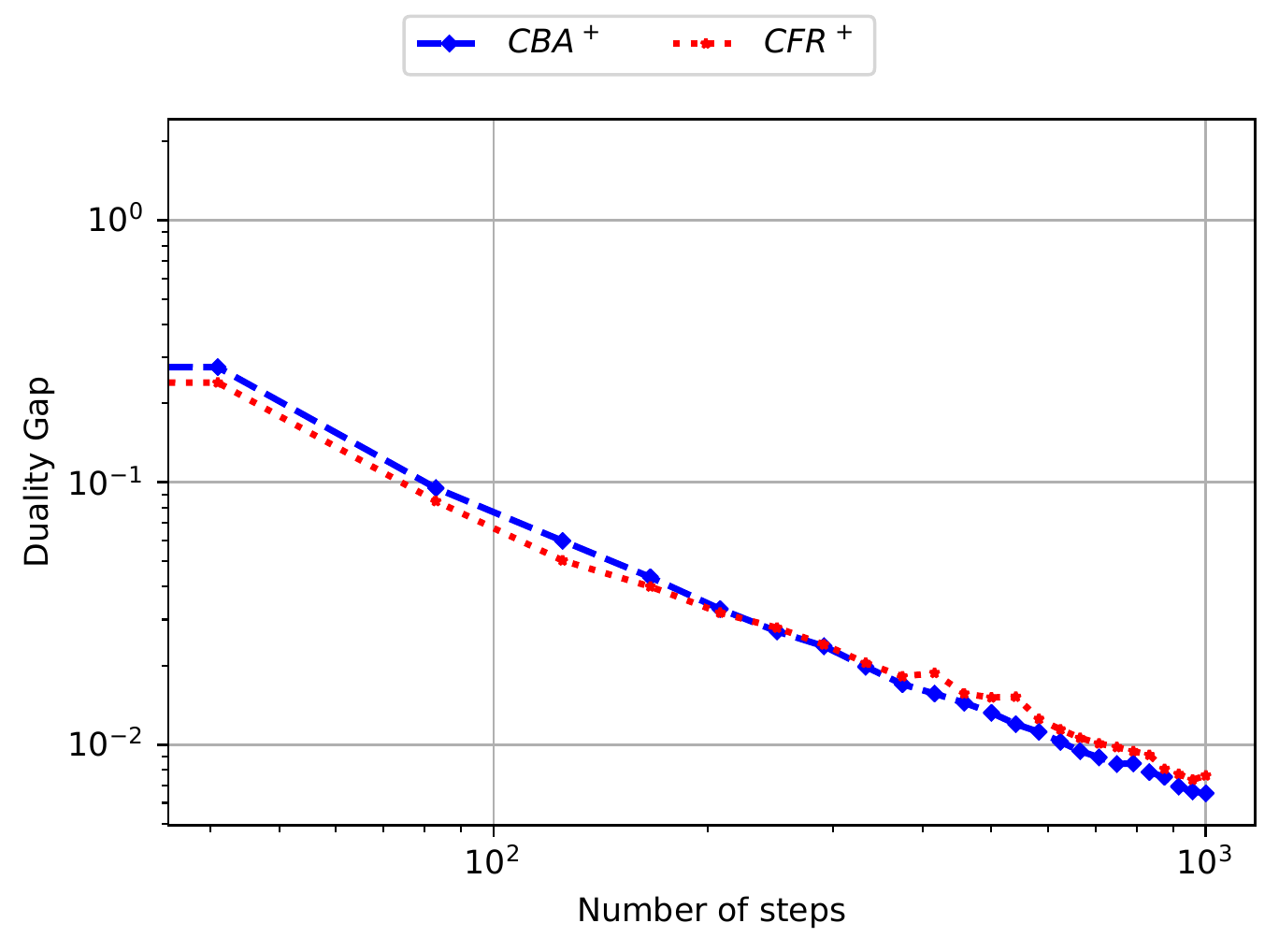}
         \caption{EFG: Battleship, 3 turns, 1 ship.}\label{fig:subfig:efg_bs_0sum_2x3_3turns_1shipL2v4}
  \end{subfigure}
       \begin{subfigure}{0.24\textwidth}
         \includegraphics[width=1.0\linewidth]{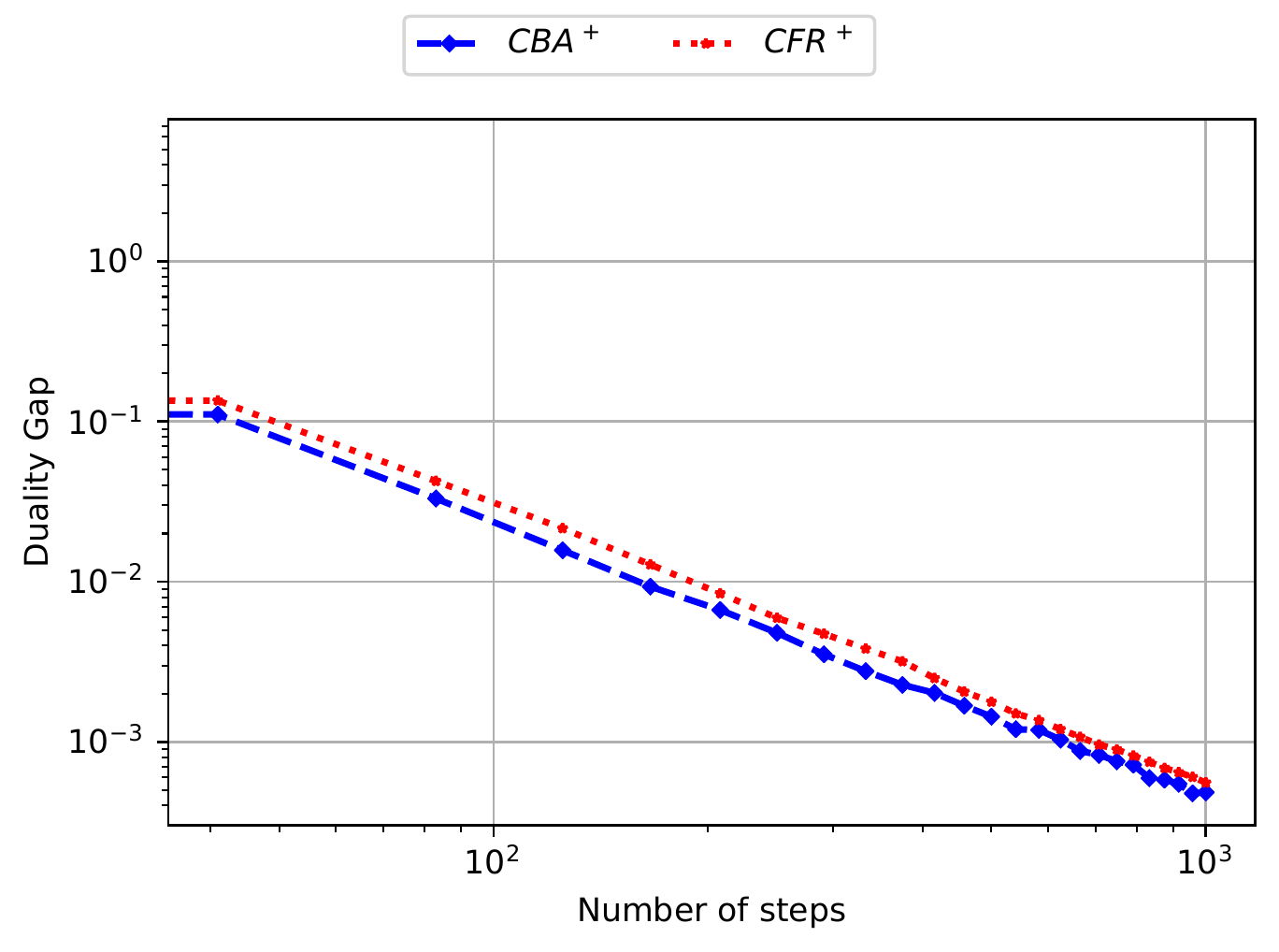}
         \caption{EFG: Leduc, 2 players, 5 ranks.}\label{fig:subfig:efg_leduc_2pl_5ranks}
  \end{subfigure}
       \vspace{-2mm}
  \caption{Comparison of \cbap{} with \rmp{} and CFR$^{+}$ on matrix games and EFGs. }
       \vspace{-4mm}
  \label{fig:matrix-games-simplex}
\end{figure}

\subsection{Distributionally Robust Optimization}
\paragraph{Problem setup}
Broadly speaking, DRO attempts to exploit partial knowledge of the statistical properties of the model parameters to obtain risk-averse optimal solutions~\citep{rahimian2019distributionally}.
We focus on the following instance of distributionally robust classification with logistic losses \citep{namkoong2016stochastic,ben2015oracle}.  There are $m$ observed feature-label pairs $\left( \bm{a}_{i},b_{i} \right) \in \R^{n} \times \{-1,1\}$,  and we want to solve 
\begin{equation}\label{eq:DRO-empirical}
\min_{\bm{x} \in \R^{n}, \| \bm{x} - \bm{x}_{0}\|_{2} \leq R} \max_{\bm{y} \in \Delta(m), \| \bm{y} - \bm{y}_{0} \|_{2}^{2} \leq \lambda} \sum_{i=1}^{m} y_{i}\ell_{i}(\bm{x}),
\end{equation}
where $\ell_{i}(\bm{x}) = \log(1+\exp(-b_{i}\bm{a}^{\top}_{i}\bm{x}))$. The formulation \eqref{eq:DRO-empirical} takes a worst-case approach to put more weight on misclassified observations and provides some statistical guarantees, e.g., it can be seen as a convex regularization of standard empirical risk minimization instances~\citep{duchi2021statistics}.

We compare \cbap{} (with linear averaging and alternation) with Online Mirror Descent (\ref{alg:OMD}), Optimistic OMD (\ref{alg:pred-OMD}), Follow-The-Regularized-Leader (\ref{alg:FTRL}) and Optimistic FTRL (\ref{alg:pred-FTRL}).  We provide a detailed presentation of our implementations of these algorithms in Appendix \ref{app:OMD-etc}. 
We compare the performances of these algorithms with \cbap{} on two synthetic data sets and four real data sets.  
We use linear averaging on decisions for all algorithms,  and parameters $\bm{x}_{0} = \bm{0},R=10,\bm{y}_{0} = \left(1,...,1\right)/m,\lambda = 1/2m$ in \cref{eq:DRO-empirical}.

\paragraph{Synthetic and real instances}
For the synthetic classification instances, we generate an optimal $\bm{x}^{*} \in \R^{n}$,  sample $\bm{a}_{i} \sim N(0,\bm{I})$  for $i \in \{1,...,m\}$,  set labels $b_{i} = \text{sign} (\bm{a}^{\top}_{i}\bm{x}^{*})$, and then  we flip $10 \%$ of them.  
For the real classification instances, we use the following data sets from the {\sf libsvm} website\footnote{https://www.csie.ntu.edu.tw/$\sim$cjlin/libsvmtools/datasets/}: \textit{adult}, {\em australian}, {\em splice}, {\em madelon}. Details about the empirical setting, the data sets and additional numerical experiments are presented in Appendix \ref{app:details-simu}. 

\paragraph{Choice of step sizes} One of the main motivation for \cbap{} is to obtain a \textit{parameter-free} algorithm. Choosing a {\it fixed} step size $\eta$ for the other algorithms requires knowing a bound $L$ on the norm of the instantaneous payoffs (see Appendix \ref{app:OMD-implementation} for our derivations of this upper bound).  
This is a major limitation in practice: these bounds may be very conservative, leading to small step sizes.  
We highlight this by showing the performance of all four algorithms, for various fixed step sizes $\eta = \alpha \times \eta_{\sf th}$, where $\alpha \in \{1,100,1,000,10,000\}$ is a multiplier and $\eta_{\sf th}$ is the theoretical step size which guarantees the convergence of the algorithms for each instance. \tb{We detail the computation of $\eta_{\sf th}$ in Appendix \ref{app:OMD-implementation}.}
We present the results of our numerical experiments on synthetic and real data sets in Figure \ref{fig:simu-DRO}. Additional simulations with adaptive step sizes $\eta_{t} = 1 / \sqrt{\sum_{\tau=1}^{t-1} \| \bm{f}_{\tau} \|_{2}^{2}}$~\citep{orabona2019modern} are presented in Figure \ref{fig:simu-DRO-adaptive} and in Appendix~\ref{app:details-simu}.
\begin{figure}
\includegraphics[width=1.0\linewidth]{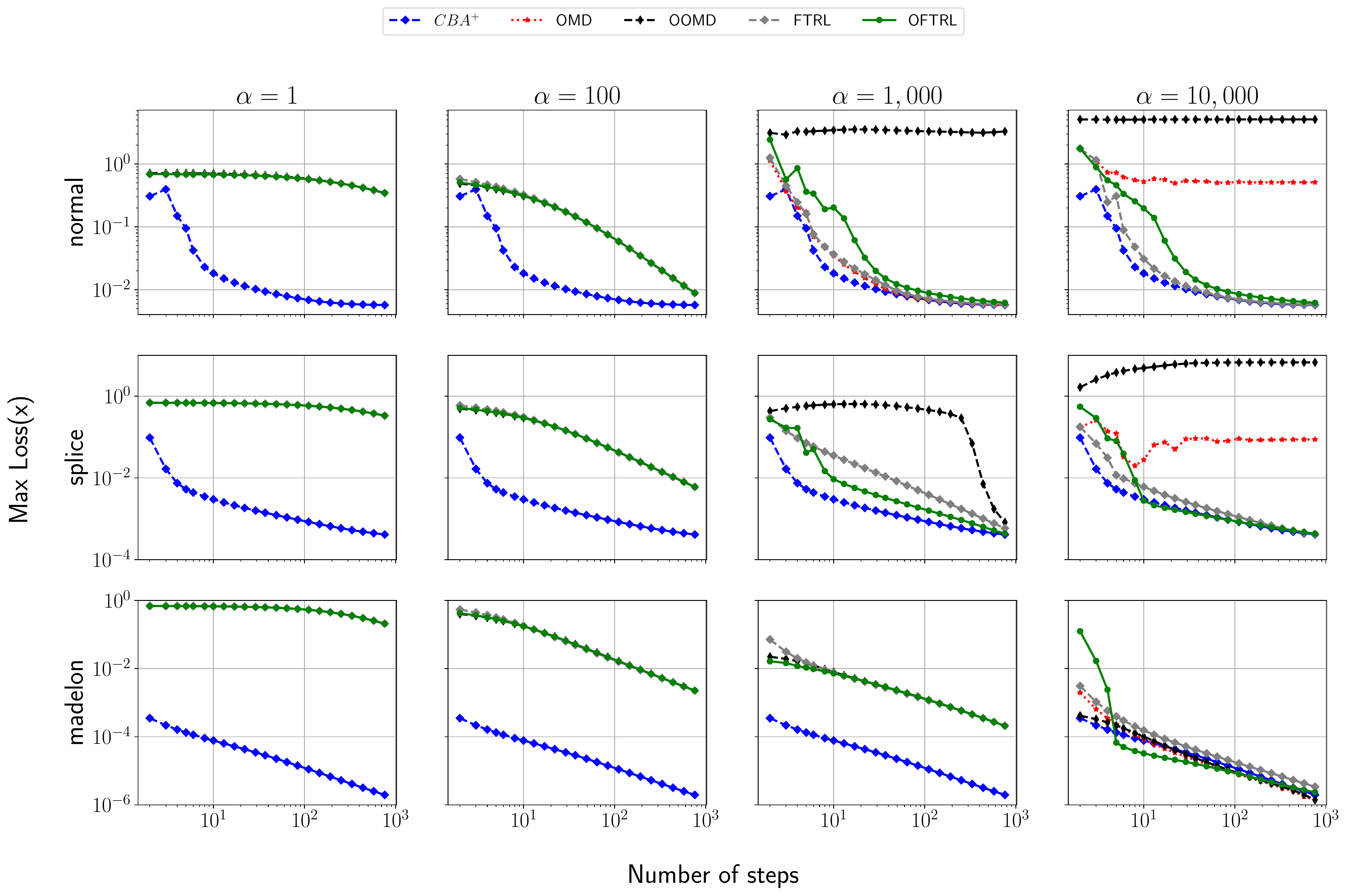}
       \vspace{-2mm}
\caption{Comparisons of the performances of \cbap{} with \ref{alg:OMD},\ref{alg:FTRL},\ref{alg:pred-OMD} and \ref{alg:pred-FTRL} with fixed step sizes, on synthetic (with \textit{normal} distribution) and real data sets (\textit{splice} and \textit{madelon}).}\label{fig:simu-DRO}
       \vspace{-4mm}
\end{figure}

\paragraph{Results and discussion} In Figure \ref{fig:simu-DRO},  we present the worst-case loss of the current solution $\bar{\bm{x}}_{T}$ in terms of the number of steps $T$.  We see that when the step sizes is chosen as the theoretical step sizes guaranteeing the convergence of the non-parameter free algorithms ($\alpha=1$), \cbap{} vastly outperforms all of the algorithms.  When we take more aggressive step sizes, the non-parameter-free algorithms become more competitive. For instance, when $\alpha=1,000$,  \ref{alg:OMD}, \ref{alg:FTRL} and \ref{alg:pred-FTRL} are competitive with \cbap{} for the experiments on synthetic data sets.  However, for this same instance and $\alpha=1,000$, \ref{alg:pred-OMD} diverges, because  the step sizes are far greater than the theoretical step sizes guaranteeing convergence.  At $\alpha=10,000$,  both \ref{alg:OMD} and \ref{alg:pred-OMD} diverge. 
The same type of performances also hold for the \textit{splice} data set.  Finally, for the \textit{madelon} data set, the non parameter-free algorithms start to be competitive with \cbap{} only when $\alpha = 10,000$.  Again, we note that this range of step sizes $\eta$ is completely outside the values $\eta_{\sf th}$ that guarantee convergence of the algorithms, \tb{ and fine-tuning the algorithms is time- and resource-consuming. In contrast, \cbap{} can be used without wasting time on exploring and finding the best convergence rates, and with confidence in the convergence of the algorithm.} Similar observations hold for adaptive step sizes (see Figure \ref{fig:simu-DRO-adaptive} and Appendix \ref{app:details-simu}). The overall poor performances of the optimistic methods (compared to their $O(1/T)$ average regret guarantees) may reflect their sensibility to the choice of the step sizes.
Additional experiments in Appendix~\ref{app:details-simu} with other real and synthetic EFG and DRO instances show the robustness of the strong performances of \cbap{} across additional problem instances. 
\begin{figure}
\includegraphics[width=1.0\linewidth]{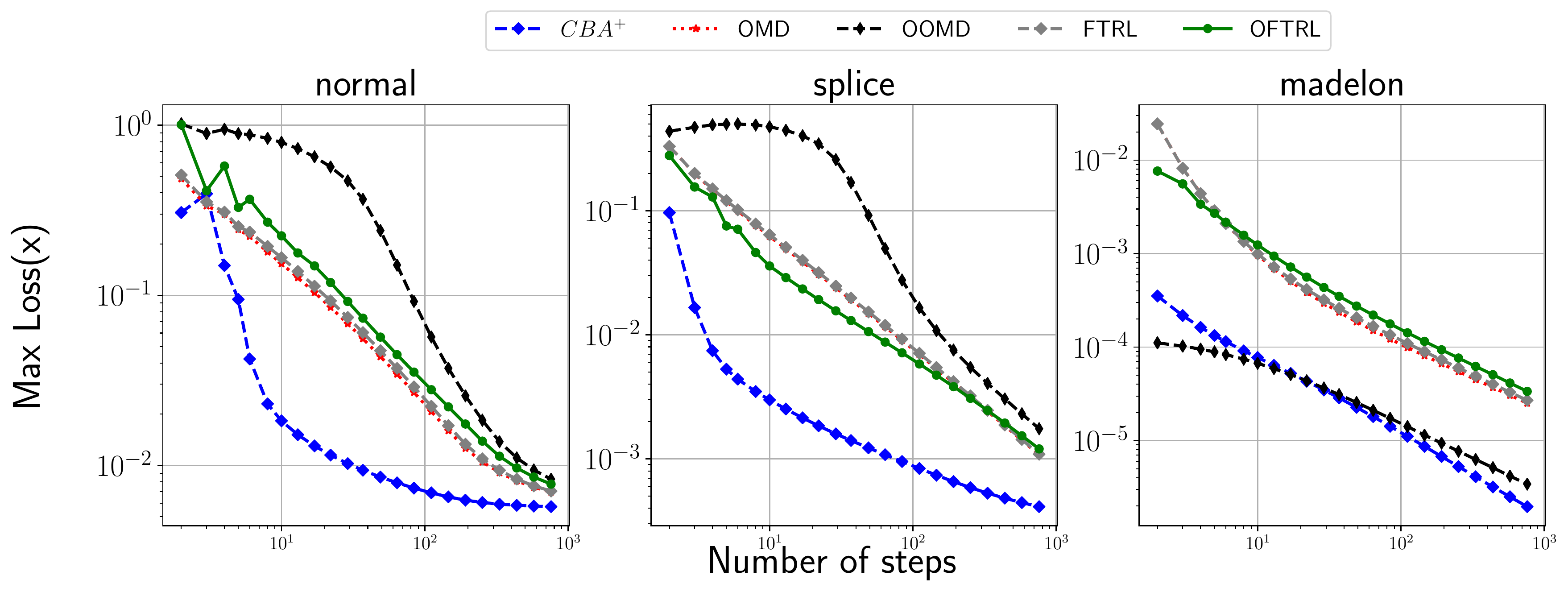}
       \vspace{-2mm}
\caption{Comparisons of the performances of \cbap{}, \ref{alg:OMD},\ref{alg:FTRL},\ref{alg:pred-OMD} and \ref{alg:pred-FTRL} with adaptive step sizes, on synthetic (with \textit{normal} distribution) and real data sets (\textit{splice} and \textit{madelon}).}\label{fig:simu-DRO-adaptive}
       \vspace{-4mm}
\end{figure}

\tb{
\paragraph{Running times compared to \cbap} We would like to emphasize that all of our figures show the {\it number of steps} on the $x$-axis, and not the actual running times of the algorithms. Overall, \cbap{} converges to an optimal solution to the DRO instance \eqref{eq:DRO-empirical} vastly faster than the other algorithms. In particular, empirically, \cbap{} is 2x-2.5x faster than \ref{alg:OMD}, \ref{alg:FTRL} and \ref{alg:pred-FTRL}, and 3x-4x faster than \ref{alg:pred-OMD}.  This is because
 \ref{alg:OMD}, \ref{alg:FTRL},  \ref{alg:pred-OMD}, and \ref{alg:pred-FTRL}  require binary searches at each step, see Appendix \ref{app:OMD-etc}.  The functions used in the binary searches themselves require solving an optimization program (an orthogonal projection onto the simplex, see \eqref{eq:q-mu}) at each evaluation.  Even though computing the orthogonal projection of a vector onto the simplex can be done in $O(n \log(n))$, this results in slower overall running time, compared to \cbap{} with (quasi) closed-form updates at each step.  The situation is even worse for \ref{alg:pred-OMD}, which requires two proximal updates at each iteration. 
We acknowledge that the same holds for \cbap{} compared to \rmp. In particular, \cbap{} is slightly slower than \rmp, because of the computation of $\pi_{\C}(\bm{u})$ in $O(n\log(n))$ operations at every iteration.} 

%% file: text/appendix.tex
\section{Comparison to \citet{shimkin2016online} and \citet{farina2021faster}}
The focus of our paper is on developing new algorithms for convex-concave saddle-point solving via Blackwell approachability algorithms for sets beyond the simplex domain. This is what motivated the development of \cbap, which attempts to generalize the ideas from \rmp{} and \cfrp{} beyond simplex settings.
A complementary question is to consider the second construction of \citet{abernethy2011blackwell}, which is a way to convert a no-regret algorithm into an approachability algorithm. At a very high level, that construction ends up using the no-regret algorithm to select which hyperplane to force, and performing regret minimization on the choice of hyperplane.
An observation made by both \citet{shimkin2016online} and \citet{farina2021faster} is that if one uses FTRL with the Euclidean regularizer, then it corresponds to Blackwell's original approachability algorithm, and thus to regret matching on the simplex.
This begs the question of what happens if one uses a different algorithm from FTRL. The most natural substitute would be OMD with a Bregman divergence derived from the Euclidean distance, which is also known as online gradient descent (OGD).

\citet{shimkin2016online} considers OGD on a variation of the original construction of \citet{abernethy2011blackwell} with the simplex as the decision set. This setup (and the original setup from \citet{abernethy2011blackwell}) uses (a subset of) the unit ball as the feasible set for the regret minimizer, and thus OGD ends up projecting the cumulated payoff vector onto the unit ball after every iteration. Therefore, the OGD setup by \citet{shimkin2016online} yields an algorithm that is reminiscent of \rmp, but where we repeatedly renormalize the cumulated payoff vector. One consequence of this is that the stepsize used to add new payoff vectors becomes important.
\citet{farina2021faster} show that the \citet{abernethy2011blackwell} construction can be extended to allow the regret minimizer to use any decision set $\mathcal D$ such that $\mathcal K \subseteq \mathcal D \subseteq \mathcal S^\circ$, where $\mathcal{S},\mathcal{S}^\circ$ is the target set and its polar cone, and $\mathcal K$ is $\mathcal S^\circ$ intersected with the unit ball. Then, \citet{farina2021faster} consider the simplex regret-minimization setting, and show that OGD instantiated on $\mathbb R_+^n = \mathcal S^\circ$ is equivalent to \rmp.

\cbap{} does not generalize either of the two simplex approaches above. For \citet{shimkin2016online}, this can be seen because of the projection and subsequent dependence on stepsize required by \citet{shimkin2016online}'s construction.
For \citet{farina2021faster}, this can be seen by the fact that they obtain \rmp{} when applying their approach to the simplex setting, and \cbap{} differs from \rmp{} (in fact, that paper does not attempt to derive new unaccelerated regret minimizers; the goal is to design accelerated, or predictive/optimistic, variants of RM and \rmp).
However, an alternative derivation of \cbap{} can be accomplished by using the generalization from \citet{farina2021faster} where OGD is run with the decision set $\mathcal{S}^\circ$. Instead of applying OGD on the no-regret formulation of the Blackwell formulation of regret-minimization on the simplex with target set $\mathcal{S}=\mathbb R_-^n$ as in \citet{farina2021faster}, we can apply the OGD-setup of \citet{farina2021faster} to the Blackwell formulation of a general no-regret problem with decision set $\mathcal X$ and target set $\mathcal C^\circ$ where $\mathcal C = \textrm{cone}(\{\kappa\}\times \mathcal X)$.
Then, we would get an indirect proof of correctness of the \cbap{} algorithm through the correctness of OGD and the two layers of reduction from regret minimization to Blackwell approachability to regret minimization. However, we believe that this approach is significantly less intuitive than understanding \cbap{} directly in terms of its properties as a Blackwell approachability algorithm.

\section{Proofs of Theorem \ref{th:cbap-linear-averaging-only-policy} }\label{app:proof-lin-avg}
\paragraph{Notations and classical results in conic optimization} We make use of the following facts. We provide a proof here for completeness.
\begin{lemma}\label{lem:conic-opt}
Let $\C \subset \R^{n+1} $ a closed convex cone and $\C^{\circ}$ its polar.
\begin{enumerate}
\item If $\bm{u} \in \R^{n+1}$, then $\bm{u} - \pi_{\C^{\circ}}(\bm{u}) = \pi_{\C}(\bm{u}) \in \C$,$\langle \bm{u} - \pi_{\C^{\circ}}(\bm{u}),\pi_{\C^{\circ}}(\bm{u}) \rangle = 0,$ and $\| \bm{u} - \pi_{\C^{\circ}}(\bm{u}) \|_{2} \leq \| \bm{u} \|_{2}$. \label{lem:statement:u-minus-proj-in-C-polar}
\item If $\bm{u} \in \R^{n+1}$ then 
\[ d(\bm{u},\C)   = \max_{\bm{w} \in \C^{\circ} \bigcap B_{2}(1) } \langle\bm{u},\bm{w} \rangle,\]
\label{lem:statement:dist-to-cone}
where $B_{2}(1) = \{ \bm{w} \in \R^{n+1} \; | \; \| \bm{w} \|_{2} \leq 1\}$.
\item If $\bm{u} \in \C$, then $d(\bm{u},\C^{\circ}) = \| \bm{u} \|_{2}$. \label{lem:statement:norm2-C}
\item Assume that $\C = \textrm{cone}(\{ \kappa \} \times \XX)$ with $\XX \subset \R^{n}$ convex compact and $\kappa = \max_{\bm{x} \in \XX} \| \bm{x} \|_{2}$. Then $\C^{\circ}$ is a closed convex cone.  Additionally, if $\bm{u} \in \C$ we have $-\bm{u} \in \C^{\circ}$. \label{lem:statement:negative-cone-in-polar}
\item Let us write $\leq_{\C^{\circ}}$ the order induced by $\C^{\circ}:\bm{x} \leq_{\C^{\circ}} \bm{y} \iff \bm{y} - \bm{x} \in \C^{\circ}$.
Then 
\begin{align}
\bm{x} \leq_{\C^{\circ}} \bm{y}, \bm{x}' \leq_{\C^{\circ}} \bm{y}' & \Rightarrow \bm{x} + \bm{x}' \leq_{\C^{\circ}} \bm{y} + \bm{y}',  \forall \; \bm{x},\bm{x}',\bm{y},\bm{y}' \in \R^{n+1}, \label{eq:partial-order-additivity}\\
\bm{x} + \bm{x}'  \leq_{\C^{\circ}} \bm{y} & \Rightarrow \bm{x} \leq_{\C^{\circ}} \bm{y}, \forall \; \bm{x},\bm{y} \in \R^{n+1}, \forall \; \bm{x}' \in \C^{\circ}, \label{eq:partial-order-minus-x-prime}
\end{align}
\item \label{lem:statement:x-y-dist} Assume that $\bm{x} \leq_{\C^{\circ}} \bm{y}$ for $\bm{x},\bm{y} \in \R^{n+1}$. Then $d(\bm{y},\C^{\circ}) \leq \| \bm{x} \|_{2}$.
\end{enumerate}
\end{lemma}
\begin{proof}
\begin{enumerate}
\item The fact that $\bm{u} - \pi_{\C^{\circ}}(\bm{u}) = \pi_{\C}(\bm{u}) \in \C$,$\langle \bm{u} - \pi_{\C^{\circ}}(\bm{u}),\pi_{\C^{\circ}}(\bm{u}) \rangle = 0$ follows from Moreau's Decomposition Theorem \citep{combettes2013moreau}. The fact that $\| \bm{u} - \pi_{\C^{\circ}}(\bm{u}) \|_{2} \leq \| \bm{u} \|_{2}$ is a straightforward consequence of $\langle \bm{u} - \pi_{\C^{\circ}}(\bm{u}),\pi_{\C^{\circ}}(\bm{u}) \rangle = 0$.
\item For any $\bm{w} \in \C^{\circ} \bigcap B_{2}(1)$ we have
\[ \langle \bm{u},\bm{w} \rangle \leq  \langle \bm{u} - \pi_{C}(\bm{u}),\bm{w} \rangle \leq \| \bm{w} \|_{2} \|  \bm{u} - \pi_{C}(\bm{u}) \|_{2} \leq \|  \bm{u} - \pi_{C}(\bm{u})\|_{2}.\]
Conversely,  since $\left( \bm{u} - \pi_{\C}(\bm{u}) \right)/\| \bm{u} - \pi_{\C}(\bm{u}) \|_{2} \in \C^{\circ}$, we have 
\[  \max_{\bm{w} \in \C^{\circ} \bigcap B_{2}^{d}(1) } \langle\bm{u},\bm{w} \rangle \geq \| \bm{u} - \pi_{\C}(\bm{u})\|_{2}.\]
This shows that 
\[ \max_{\bm{w} \in \C^{\circ} \bigcap B_{2}^{d}(1) } \langle\bm{u},\bm{w} \rangle = \| \bm{u} - \pi_{\C}(\bm{u})\|_{2} = d(\bm{u},\C).\]
\item For any $\bm{u} \in \R^{n+1}$,  by definition we have  $d(\bm{u},\C^{\circ}) = \| \bm{u} - \pi_{\C^{\circ}}(\bm{u})\|_{2}$. Now if $\bm{u} \in \C$ we have $ \pi_{\C^{\circ}}(\bm{u}) = 0$ so $d(\bm{u},\C^{\circ}) = \| \bm{u} \|_{2}$.
\item Let $\bm{u} \in \C$. Then $\bm{u} = \alpha (\kappa,\bm{x})$ for $\alpha \geq 0, \bm{x} \in \XX$. We will show that $-\bm{u} \in \C^{\circ}$. We have
\begin{align*}
- \bm{u} \in \C^{\circ} &  \iff \langle - \bm{u} , \bm{u}' \rangle \leq 0, \forall \; \bm{u}' \in \C \\
& \iff \langle - \alpha (\kappa,\bm{x}), \alpha' (\kappa,\bm{x}') \rangle \leq 0, \forall \; \alpha' \geq 0, \forall \; \bm{x}' \in \XX \\
& \iff   \kappa^{2} + \langle \bm{x},\bm{x}' \rangle \geq 0 \\
& \iff -  \langle \bm{x},\bm{x}' \rangle \leq \kappa^{2},
\end{align*}
and $- \langle \bm{x},\bm{x}' \rangle \leq \kappa^{2}$ is true by Cauchy-Schwartz and the definition of $\kappa = \max_{\bm{x} \in \XX} \| \bm{x} \|_{2}$.
\item We start by proving \eqref{eq:partial-order-additivity}.  Let $ \bm{x},\bm{x}',\bm{y},\bm{y}' \in \R^{n+1}, $ and assume that $ \bm{x} \leq_{\C^{\circ}} \bm{y}, \bm{x}' \leq_{\C^{\circ}} \bm{y}'$. Then $\bm{y} - \bm{x} \in \C^{\circ},\bm{y}'- \bm{x}' \in \C^{\circ}$. Because $\C^{\circ}$ is a convex set, and a cone, we have $2 \cdot \left( \dfrac{\bm{y} - \bm{x}}{2} +  \dfrac{\bm{y}' - \bm{x}'}{2} \right) \in \C^{\circ}$. Therefore, $\bm{y} + \bm{y}' - \bm{x} - \bm{x}' \in \C^{\circ}$, i.e.,  $\bm{x} + \bm{x}' \leq_{\C^{\circ}} \bm{y} + \bm{y}'$.

We now prove \eqref{eq:partial-order-minus-x-prime}.  Let $\; \bm{x},\bm{y} \in \R^{n+1}, \bm{x}' \in \C^{\circ}$ and assume that $\bm{x} + \bm{x}'  \leq_{\C^{\circ}} \bm{y}$. Then by definition $\bm{y}  - \bm{x} -  \bm{x}'  \in \C^{\circ}$. Additionally, $\bm{x}' \in \C^{\circ}$ by assumption.  Since $\C^{\circ}$ is convex, and is a cone, $2 \cdot \left( \dfrac{\bm{y}  - \bm{x} -  \bm{x}'}{2}+ \dfrac{\bm{x}'}{2} \right) \in \C^{\circ}$, i.e., $\bm{y} - \bm{x} \in \C^{\circ}$. Therefore, $\bm{x} \leq_{\C^{\circ}} \bm{y}.$
\item Let $\bm{x},\bm{y} \in \R^{n+1}$ such that $\bm{x} \leq_{\C^{\circ}} \bm{y}$. Then $\bm{y} - \bm{x} \in \C^{\circ}$. We have
\[ d(\bm{y},\C^{\circ}) = \min_{\bm{z} \in \C^{\circ}} \| \bm{y} - \bm{z} \|_{2} \leq \| \bm{y} - ( \bm{y} - \bm{x}) \|_{2} = \| \bm{x} \|_{2}.\]
\end{enumerate}
\end{proof}
Based on Moreau's Decomposition Theorem, we will use $\pi_{\C}(\bm{u})$ and $\bm{u}-\pi_{\C^{\circ}}(\bm{u})$ interchangeably.
\paragraph{Results for various linear averaging schemes} We now present our convergence results for various linear averaging schemes.  As a warm-up, we start with two theorems, Theorem \ref{th:cba-linear-averaging-both} and Theorem \ref{th:cbap-linear-averaging-both}, which show that \cba{} and \cbap{} are compatible with weighted average schemes, when \textit{both} the decisions and the payoffs are weighted. The proofs for these theorems will be used in the proof of our main theorem, Theorem \ref{th:cbap-linear-averaging-only-policy}. For the sake of consiness,  in all the proofs of this section we will always write $L = \max \{ \| \bm{f}_{t} \|_{2} \; |\;  t \geq 1\}, \kappa = \max \{ \| \bm{x} \|_{2} \; | \; t \geq 1 \},\bm{v}_{t} = \left(\langle \bm{f}_{t},\bm{x}_{t} \rangle / \kappa, - \bm{f}_{t} \rangle \right).$. 
We start with the following theorem.
\begin{theorem}\label{th:cba-linear-averaging-both}
Let $\left( \bm{x}_{t} \right)_{t \geq 0}$ the sequence of decisions generated by \cba{} with weights $\left( \omega_{t} \right)_{t \geq 0}$ and let $S_{t} = \sum_{\tau=1}^{t} \omega_{\tau}$ for any $t \geq 1$. 
Then
\[ \dfrac{\sum_{t=1}^{T} \omega_{t} \langle \bm{f}_{t},\bm{x}_{t} \rangle -  \min_{\bm{x} \in \XX} \sum_{t=1}^{T} \omega_{t} \langle \bm{f}_{t},\bm{x} \rangle}{S_{T}} = O \left( \kappa \cdot  d(\bm{u}_{T},\C^{\circ}) \right).\]
Additionally,
\[ d(\bm{u}_{T},\C^{\circ})^{2} = O \left(L^{2} \cdot  \frac{\sum_{t=1}^{T} \omega_{t}^{2}}{\left(\sum_{t=1}^{T} \omega_{t} \right)^{2}} \right).\]
Overall, 
\[\dfrac{\sum_{t=1}^{T} \omega_{t} \langle \bm{f}_{t},\bm{x}_{t} \rangle -  \min_{\bm{x} \in \XX} \sum_{t=1}^{T} \omega_{t} \langle \bm{f}_{t},\bm{x} \rangle}{S_{T}} = O \left( \kappa L \sqrt{ \frac{\sum_{t=1}^{T} \omega_{t}^{2}}{\left(\sum_{t=1}^{T} \omega_{t} \right)^{2}} } \right).\]
\end{theorem}
\begin{proof} The proof proceeds in two steps.
We start by proving 
\[ \dfrac{\sum_{t=1}^{T} \omega_{t} \langle \bm{f}_{t},\bm{x}_{t} \rangle -  \min_{\bm{x} \in \XX} \sum_{t=1}^{T} \omega_{t} \langle \bm{f}_{t},\bm{x} \rangle}{S_{T}} = O \left(\kappa \cdot  d(\bm{u}_{T},S) \right).\]
We have
\begin{align}
d(\bm{u}_{T},\C^{\circ}) &  = \max_{\bm{w} \in \textrm{cone}(\{ \kappa \} \times \XX) \bigcap B_{2}(1) } \langle \dfrac{1}{S_{T}} \sum_{t=1}^{T} \omega_{t} \bm{v}_{t},\bm{w} \rangle \label{eq:proof-step-0} \\
 & \geq  \max_{\bm{x} \in \XX} \langle \dfrac{1}{S_{T}} \sum_{t=1}^{T}  \omega_{t} \bm{v}_{t},\dfrac{(\kappa,\bm{x})}{\| (\kappa,\bm{x}) \|_{2}} \rangle \nonumber \\
 & \geq \dfrac{1}{S_{T}} \max_{\bm{x} \in \XX} \dfrac{\sum_{t=1}^{T} \omega_{t} \langle \bm{f}_{t},\bm{x}_{t} \rangle - \sum_{t=1}^{T} \omega_{t} \langle \bm{f}_{t},\bm{x} \rangle}{ \| (\kappa,\bm{x})\|_{2}} \label{eq:proof-lin-avg-step-0},
\end{align}
where \eqref{eq:proof-step-0} follows from Statement 1 in Lemma \ref{lem:conic-opt}, and \eqref{eq:proof-lin-avg-step-0} follows from \cba{} maintaining
\[ \bm{u}_{t} = \left( \frac{1}{S_{t}} \sum_{\tau=1}^{t} \omega_{\tau} \frac{ \langle \bm{f}_{\tau},  \bm{x}_{\tau}\rangle}{\kappa} ,  -  \frac{1}{S_{t}}  \sum_{\tau=1}^{t} \omega_{\tau} \bm{f}_{\tau} \right), \forall \; t \geq 1.\]
We can conclude that
\[ 2 \kappa d(\bm{u}_{T},\C^{\circ}) \geq \dfrac{\sum_{t=1}^{T} \omega_{t} \langle \bm{f}_{t},\bm{x}_{t} \rangle -  \min_{\bm{x} \in \XX} \sum_{t=1}^{T} \omega_{t} \langle \bm{f}_{t},\bm{x} \rangle}{S_{T}}.\]
We now prove that \[ d(\bm{u}_{T},\C^{\circ})^{2} = O \left( L^2 \frac{\sum_{\tau=1}^{T} \omega_{\tau}^{2}}{\left(\sum_{\tau=1}^{T} \omega_{\tau} \right)^{2}} \right).\]
We have
\begin{align}
d(\bm{u}_{t+1},\C^{\circ})^{2} & = \min_{\bm{z} \in \C^{\circ}} \| \bm{u}_{t+1} - \bm{z} \|_{2}^{2} \\
& \leq \| \bm{u}_{t+1} - \pi_{\C^{\circ}}(\bm{u}_{t}) \|_{2}^{2} \nonumber \\
& \leq \|  \dfrac{S_{t}}{S_{t}+ \omega_{t+1}} \bm{u}_{t} + \dfrac{\omega_{t+1}}{S_{t}+ \omega_{t+1}} \bm{v}_{t+1} - \pi_{\C^{\circ}}(\bm{u}_{t})\|_{2}^{2} \nonumber \\
& \leq \|  \dfrac{S_{t}}{S_{t}+ \omega_{t+1}}  (\bm{u}_{t} -\pi_{\C^{\circ}}(\bm{u}_{t}))  + \dfrac{\omega_{t+1}}{S_{t}+ \omega_{t+1}} \left( \bm{v}_{t+1} - \pi_{\C^{\circ}}(\bm{u}_{t}) \right) \|_{2}^{2} \nonumber \\
& \leq \dfrac{1}{S_{t+1}^{2}} (S_{t}^{2} \| \bm{u}_{t} - \pi_{\C^{\circ}}(\bm{u}_{t}) \|_{2}^{2} + \omega_{t+1}^{2} \| \bm{v}_{t+1} - \pi_{\C^{\circ}}(\bm{u}_{t}) \|_{2}^{2} \nonumber \\
& + 2 S_{t}\omega_{t+1} \langle \bm{u}_{t} - \pi_{\C^{\circ}}(\bm{u}_{t}), \bm{v}_{t+1} - \pi_{\C^{\circ}}(\bm{u}_{t}) \rangle ) \nonumber \\
& \leq \dfrac{1}{S_{t+1}^{2}} \left( S_{t}^{2} \| \bm{u}_{t} - \pi_{\C^{\circ}}(\bm{u}_{t}) \|_{2}^{2} + \omega_{t+1}^{2} \| \bm{v}_{t+1} - \pi_{\C^{\circ}}(\bm{u}_{t}) \|_{2}^{2} \right), \label{eq:proof-lin-avg-step-1}
\end{align}
\tb{
where \eqref{eq:proof-lin-avg-step-1} follows from
\begin{equation}\label{eq:blackwell-forcing}
\langle \bm{u}_{t} - \pi_{\C^{\circ}}(\bm{u}_{t}), \bm{v}_{t+1} - \pi_{\C^{\circ}}(\bm{u}_{t}) \rangle = 0.
\end{equation} This is because:
\begin{itemize}
\item $ \langle \bm{u}_{t} - \pi_{\C^{\circ}}(\bm{u}_{t}), \bm{v}_{t+1} \rangle = 0$. This is one of the crucial component of Blackwell's approachability framework: the current decision is chosen to force a hyperplane on the aggregate payoff. To see this, first note that $\bm{u}_{t} - \pi_{\C^{\circ}}(\bm{u}) = \pi_{\C}(\bm{u}_{t})$. Let us write $\bm{\pi} = \left(\tilde{\pi},\hat{\bm{\pi}} \right) =  \pi_{\C}(\bm{u}_{t})$. Note that by definition, $\bm{x}_{t+1} =  (\kappa/\tilde{\pi})\hat{\bm{\pi}}$, and $\bm{v}_{t+1} = \left(\langle \bm{f}_{t+1},\bm{x}_{t+1} \rangle / \kappa, - \bm{f}_{t+1} \right)$.
Therefore,
\begin{align*}
\langle \bm{u}_{t} - \pi_{\C^{\circ}}(\bm{u}_{t}), \bm{v}_{t+1} \rangle & = \langle \bm{\pi}, \bm{v}_{t+1} \rangle \\
& = \langle \left(\tilde{\pi},\hat{\bm{\pi}} \right), \left(\langle \bm{f}_{t+1},\bm{x}_{t+1} \rangle / \kappa, - \bm{f}_{t+1}  \right) \rangle \\
&  = \langle \left(\tilde{\pi},\hat{\bm{\pi}} \right), \left(\langle \bm{f}_{t+1},(\kappa/\tilde{\pi})\hat{\bm{\pi}} \rangle / \kappa, - \bm{f}_{t+1} \right) \rangle \\
& = \langle \hat{\bm{\pi}}, \bm{f}_{t+1} \rangle - \langle \hat{\bm{\pi}}, \bm{f}_{t+1} \rangle \\
& = 0.
\end{align*}
\item $ \langle \bm{u}_{t} - \pi_{\C^{\circ}}(\bm{u}_{t}),  \pi_{\C^{\circ}}(\bm{u}_{t}) \rangle = 0$ from Statement 3 of Lemma \ref{lem:conic-opt} and $ \bm{u}_{t} - \pi_{\C^{\circ}}(\bm{u}_{t}) = \pi_{\C}(\bm{u}_{t}) \in \C$.
\end{itemize}
} 
We therefore have
\[ d(\bm{u}_{t+1},\C^{\circ})^{2} \leq \dfrac{1}{S_{t+1}^{2}} \left( S_{t}^{2} \| \bm{u}_{t} - \pi_{\C^{\circ}}(\bm{u}_{t}) \|_{2}^{2} + \omega_{t+1}^{2} \| \bm{v}_{t+1} - \pi_{\C^{\circ}}(\bm{u}_{t}) \|_{2}^{2} \right).\]
This recursion directly gives
\[ d(\bm{u}_{t+1},\C^{\circ})^{2} \leq \dfrac{1}{S_{t+1}^{2}} \sum_{\tau=1}^{t}  \omega_{\tau+1}^{2} \| \bm{v}_{\tau+1} - \pi_{\C^{\circ}}(\bm{u}_{\tau}) \|_{2}^{2} \leq O \left( L^2 \cdot \dfrac{\sum_{\tau=1}^{t+1}  \omega_{\tau}^{2}}{S_{t+1}^{2}} \right),\]
where the last inequality follows from the definition of $\bm{v}_{t}$ and $L$.
\end{proof}
 \begin{theorem}\label{th:cbap-linear-averaging-both}
Let $\left( \bm{x}_{t} \right)_{t \geq 0}$ the sequence of decisions generated by \cbap{} with weights $\left( \omega_{t} \right)_{t \geq 0}$ and let $S_{t} = \sum_{\tau=1}^{t} \omega_{\tau}$ for any $t \geq 1$.
Then
\[ \dfrac{\sum_{t=1}^{T} \omega_{t} \langle \bm{f}_{t},\bm{x}_{t} \rangle -  \min_{\bm{x} \in \XX} \sum_{t=1}^{T} \omega_{t} \langle \bm{f}_{t},\bm{x} \rangle}{S_{T}} = O \left( \kappa \cdot d(\bm{u}_{T},\C^{\circ}) \right).\]
Additionally,
\[ d(\bm{u}_{T},\C^{\circ})^{2} = O \left(L^{2} \cdot  \frac{\sum_{t=1}^{T} \omega_{t}^{2}}{\left(\sum_{t=1}^{T} \omega_{t} \right)^{2}} \right).\]
Overall, 
\[\dfrac{\sum_{t=1}^{T} \omega_{t} \langle \bm{f}_{t},\bm{x}_{t} \rangle -  \min_{\bm{x} \in \XX} \sum_{t=1}^{T} \omega_{t} \langle \bm{f}_{t},\bm{x} \rangle}{S_{T}} = O \left( \kappa L \sqrt{ \frac{\sum_{t=1}^{T} \omega_{t}^{2}}{\left(\sum_{t=1}^{T} \omega_{t} \right)^{2}}}  \right).\]
\end{theorem}
\begin{proof}[Proof of Theorem \ref{th:cbap-linear-averaging-both}]
The proof proceeds in two steps.
We start by proving 
\[ \dfrac{\sum_{t=1}^{T} \omega_{t} \langle \bm{f}_{t},\bm{x}_{t} \rangle -  \min_{\bm{x} \in \XX} \sum_{t=1}^{T} \omega_{t} \langle \bm{f}_{t},\bm{x} \rangle}{S_{T}} = O \left( \kappa \cdot  d(\bm{u}_{T},S) \right).\]
Recall that $\bm{v}_{t} = \left(\langle \bm{f}_{t},\bm{x}_{t} \rangle / \kappa, - \bm{f}_{t} \rangle \right)$, and let us consider $\bm{R}_{t} = \dfrac{1}{S_{t}} \sum_{\tau=1}^{t} \omega_{\tau} \bm{v}_{\tau}$. By definition of $\bm{R}_{t}$,  similarly as in the proof of Theorem \ref{th:cba-linear-averaging-both}, we have
\[ \dfrac{\sum_{t=1}^{T} \omega_{t} \langle \bm{f}_{t},\bm{x}_{t} \rangle -  \min_{\bm{x} \in \XX} \sum_{t=1}^{T} \omega_{t} \langle \bm{f}_{t},\bm{x} \rangle}{S_{T}} = O \left(\kappa \cdot d(\bm{R}_{T},\C^{\circ}) \right).\]

Note that at any period $t$, we have
\begin{equation}\label{eq:delta-R-delta-u}
S_{t+1}\bm{u}_{t+1} -S_{t} \bm{u}_{t} \leq_{\C^{\circ}} S_{t+1} \bm{R}_{t+1} - S_{t} \bm{R}_{t}.
\end{equation}
This is simply because $\bm{u}_{t+1} = \pi_{\C}(\bm{u}_{t+1/2}) = \bm{u}_{t+1/2} - \pi_{\C^{o}}(\bm{u}_{t+1/2})$ with \[\bm{u}_{t+1/2} = \upp_{\sf CBA}(\bm{u}_{t}) =  \frac{S_{t}}{S_{t}+\omega_{t+1}}\bm{u}_{t} + \frac{\omega_{t+1}}{S_{t}+\omega_{t+1}} \bm{v}_{t+1}.\]
 Now we have
\begin{align*}
S_{t+1}\bm{R}_{t+1} - S_{t}\bm{R}_{t} - \left(  S_{t+1}\bm{u}_{t+1} -S_{t} \bm{u}_{t} \right) & = \omega_{t+1} \bm{v}_{t+1} +  S_{t} \bm{u}_{t} -S_{t+1} \bm{u}_{t+1/2} +S_{t+1} \pi_{\C^{\circ}}(\bm{u}_{t+1/2}) \\
& = S_{t+1} \bm{u}_{t+1/2} -S_{t+1} \bm{u}_{t+1/2} + S_{t+1} \pi_{\C^{\circ}}(\bm{u}_{t+1/2}) \\
& =  S_{t+1} \pi_{\C^{\circ}}(\bm{u}_{t+1/2}) \in \C^{\circ}.
\end{align*}
From \eqref{eq:partial-order-additivity} in Lemma \ref{lem:conic-opt}, we can sum the inequalities \eqref{eq:delta-R-delta-u}.  Noticing that $\bm{u}_{1} = \bm{R}_{1}$,  we can conclude that 
\[ \bm{u}_{t} \leq_{\C^{\circ}} \bm{R}_{t}.\]
From $\bm{u}_{t} \in \C$ and Statement \ref{lem:statement:x-y-dist} in Lemma \ref{lem:conic-opt}, we have $d(\bm{R}_{t},\C^{\circ}) \leq \| \bm{u}_{t} \|_{2}$. This implies
\[  \dfrac{\sum_{t=1}^{T} \omega_{t} \langle \bm{f}_{t},\bm{x}_{t} \rangle -  \min_{\bm{x} \in \XX} \sum_{t=1}^{T} \omega_{t} \langle \bm{f}_{t},\bm{x} \rangle}{S_{T}} = O \left(\kappa  \| \bm{u}_{T} \|_{2} \right).\]
We now turn to proving
\[ \| \bm{u}_{T} \|_{2}^{2} = O \left( L^2 \frac{\sum_{t=1}^{T} \omega_{t}^{2}}{\left(\sum_{t=1}^{T} \omega_{t} \right)^{2}} \right).\]
We have
\begin{align}
 \| \bm{u}_{t+1} \|_{2}^{2}
& = \| \bm{u}_{t+1/2} - \pi_{\C^{\circ}}(\bm{u}_{t+1/2}) \|_{2}^{2} \\
& \leq \| \bm{u}_{t+1/2} \|_{2}^{2} \label{eq:norm-proj-smaller-than-norm} \\
& \leq \|  \frac{S_{t}}{S_{t}+\omega_{t+1}}\bm{u}_{t} + \frac{\omega_{t+1}}{S_{t}+\omega_{t+1}} \bm{v}_{t+1} \|_{2}^{2}, 
\end{align}
where \eqref{eq:norm-proj-smaller-than-norm} follows from Statement \ref{lem:statement:u-minus-proj-in-C-polar} in Lemma \ref{lem:conic-opt}.
Therefore,
\begin{align*}
\| \bm{u}_{t+1}\|_{2}^{2} & \leq \dfrac{1}{\left(S_{t}+\omega_{t+1}\right)^{2}} \left( S_{t}^{2} \| \bm{u}_{t} \|^{2}_{2} + \omega_{t+1}^{2} \| \bm{v}_{t+1} \|^{2}_{2} + 2 S_{t}\omega_{t+1} \langle \bm{u}_{t},\bm{v}_{t+1} \rangle \right).
\end{align*}
By construction and for the same reason as for \eqref{eq:blackwell-forcing}, $\langle \bm{u}_{t},\bm{v}_{t+1} \rangle=0$. 
Therefore, we have the recursion
\[ \| \bm{u}_{t+1}\|_{2}^{2} \leq \dfrac{1}{S_{t+1}^{2}} \left( S_{t}^{2} \| \bm{u}_{t} \|^{2}_{2} + \omega_{t+1}^{2} \| \bm{v}_{t+1} \|^{2}_{2}\right).\]
By telescoping the inequality above we obtain
\[ d(\bm{u}_{t+1},\C^{\circ})^{2} \leq \dfrac{1}{S_{t+1}^{2}} \left( \sum_{\tau=1}^{t+1} \omega_{\tau}^{2} \| \bm{v}_{\tau} \|^{2}_{2} \right).\]
By definition of $L$,
\[ \| \bm{u}_{t+1} \|_{2}^{2} = O \left(L^{2} \cdot \dfrac{\sum_{\tau=1}^{t+1} \omega_{\tau}^{2}}{S_{t+1}^{2}} \right).\]
\end{proof}
\paragraph{Linear averaging only on decisions}
We are now ready to prove our main convergence result, Theorem \ref{th:cbap-linear-averaging-only-policy}.  Our proof heavily relies on the sequence of payoffs belonging to the cone $\C$ at every iteration ($\bm{u}_{t} \in \C, \forall \; t \geq 1$), and for this reason it does not extend to \cba. We also note that the use of conic optimization somewhat simplifies the argument compared to the proof that \rmp is compatible with linear averaging \citep{tammelin2015solving}.
\begin{proof}[Proof of Theorem \ref{th:cbap-linear-averaging-only-policy}]
Recall that $\bm{v}_{t} = \left(\langle \bm{f}_{t},\bm{x}_{t} \rangle / \kappa, - \bm{f}_{t} \rangle \right)$. By construction and following the same argument as for the proof of Theorem \ref{th:cbap-linear-averaging-both}, we have
\begin{align}
\sum_{t=1}^{T} t \langle \bm{f}_{t},\bm{x}_{t} \rangle -  \min_{\bm{x} \in \XX} \sum_{t=1}^{T} t \langle \bm{f}_{t},\bm{x} \rangle = O \left( \kappa \cdot d \left(\sum_{t=1}^{T} t \bm{v}_{t},\C^{\circ} \right) \right). 
\end{align}
Additionally,  Equation \eqref{eq:delta-R-delta-u} for uniform weights ($\omega_{\tau}=1, S_{\tau}=\tau$) yields
\[ \bm{v}_{t+1} \geq_{\C^{\circ}} (t+1)\bm{u}_{t+1} - t \bm{u}_{t}.\]
Therefore,
\[ (t+1) \bm{v}_{t+1} \geq_{\C^{\circ}} (t+1)^{2}\bm{u}_{t+1} - t^{2} \bm{u}_{t} -t \bm{u}_{t}.\]
Summing up the previous inequalities from $t=1$ to $t=T-1$ and using $\bm{u}_{1}=\bm{v}_{1}$ we obtain
\[ \sum_{t=1}^{T} t \bm{v}_{t} \geq_{\C^{\circ}} T^{2}\bm{u}_{T} - \sum_{t=1}^{T-1} t \bm{u}_{t}.\]

Note that since $\sum_{t=1}^{T-1} t \bm{u}_{t} \in \C$, Statement \ref{lem:statement:negative-cone-in-polar} in Lemma \ref{lem:conic-opt} shows that $- \sum_{t=1}^{T-1} t \bm{u}_{t} \in \C^{\circ}$.
Now, by applying \eqref{eq:partial-order-minus-x-prime} in Lemma \ref{lem:conic-opt},  we have
\[ \sum_{t=1}^{T} t \bm{v}_{t} \geq_{\C^{\circ}} T^{2}\bm{u}_{T} - \sum_{t=1}^{T-1} t \bm{u}_{t} \Rightarrow \sum_{t=1}^{T} t \bm{v}_{t} \geq_{\C^{\circ}} T^{2}\bm{u}_{T}. \]
Since $T^{2}\bm{u}_{T} \in \C$, Statement \ref{lem:statement:x-y-dist} shows that
\[ d \left(\sum_{t=1}^{T} t \bm{v}_{t} ,\C^{\circ} \right) \leq \| T^{2} \bm{u}_{T} \|_{2}.\]
By construction $\bm{u}_{T}$ is the output of \cbap{} with uniform weight, so that $d(\bm{u}_{T},\C^{\circ}) = \| \bm{u}_{T} \|_{2} =  O(L/\sqrt{T})$ (see Theorem \ref{th:cbap-linear-averaging-both}).
Therefore,  $d(\sum_{t=1}^{T} t \bm{v}_{t} ,\C^{\circ}) = O \left( L \cdot T^{3/2} \right).$
This shows that 
\[ \dfrac{\sum_{t=1}^{T} t \langle \bm{f}_{t},\bm{x}_{t} \rangle -  \min_{\bm{x} \in \XX} \sum_{t=1}^{T} t \langle \bm{f}_{t},\bm{x} \rangle}{T(T+1)} = O \left( \kappa \dfrac{d \left( \sum_{t=1}^{T} t \bm{v}_{t} ,\C^{\circ} \right)}{T(T+1)} \right) = O \left(\kappa L /\sqrt{T} \right).\]
\end{proof}
\paragraph{Comparisons of different weighted average schemes} We conclude this section with an empirical comparisons of the different weighted average schemes (Theorem \ref{th:cba-linear-averaging-both}, Theorem \ref{th:cbap-linear-averaging-both}, and Theorem \ref{th:cbap-linear-averaging-only-policy}). We also compare these algorithms with \rmp. We present our numerical experiments on sets of random matrix game instances in Figure \ref{fig:comparison-averaging-schemes}.  The setting is the same as in our simulation section, Section \ref{sec:simu}. We note that \cbap{} with linear averaging only on decisions outperforms both \cbap{} and \cba{} with linear averaging on both decisions and payoffs, as well as \rmp{} with linear averaging on decisions.
\begin{figure}[h]
 \begin{subfigure}{0.45\textwidth}
\centering
         \includegraphics[width=1.0\linewidth]{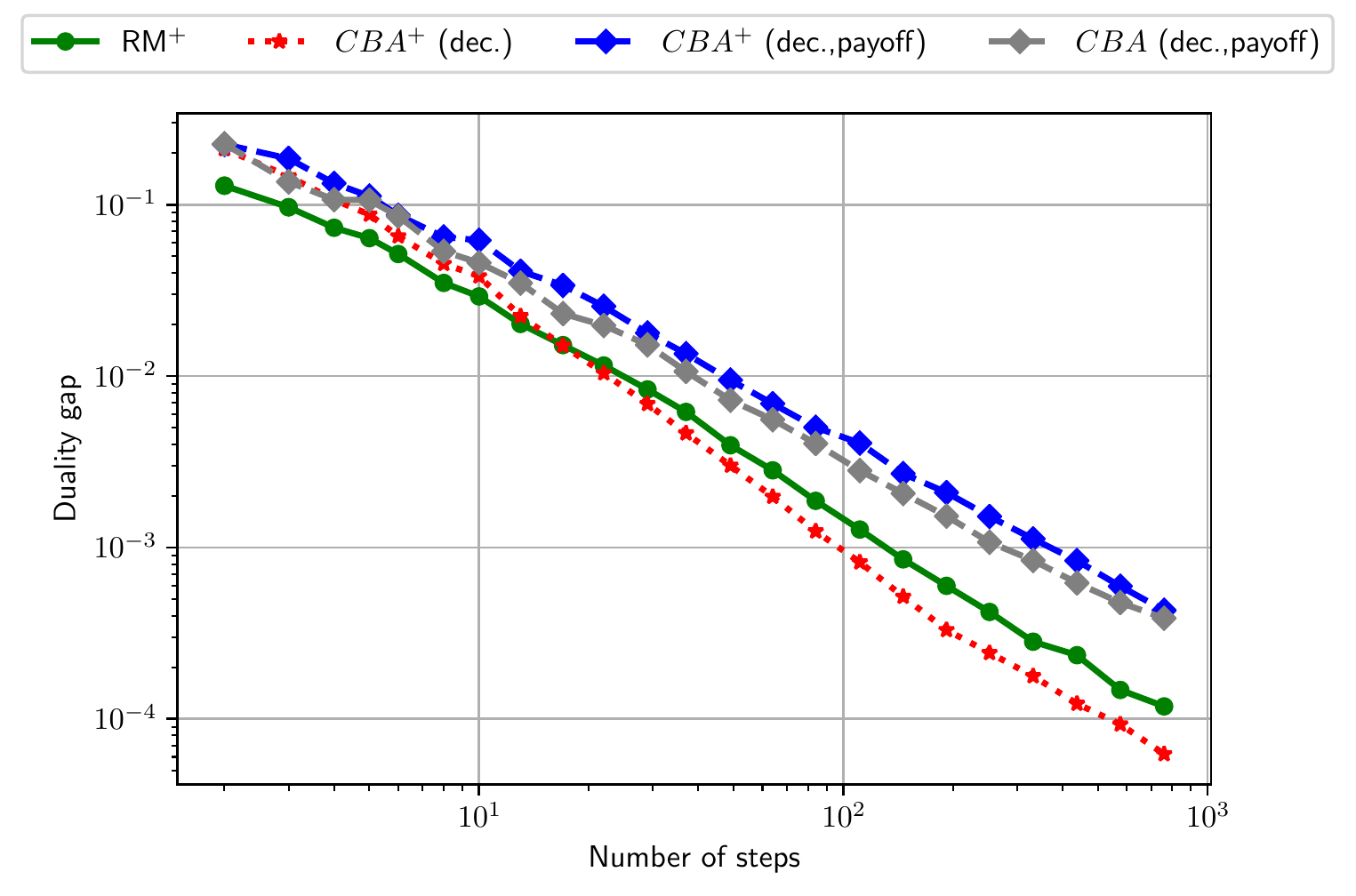}
         \caption{Uniform distribution.}
  \end{subfigure}
   \begin{subfigure}{0.45\textwidth}
\centering
         \includegraphics[width=1.0\linewidth]{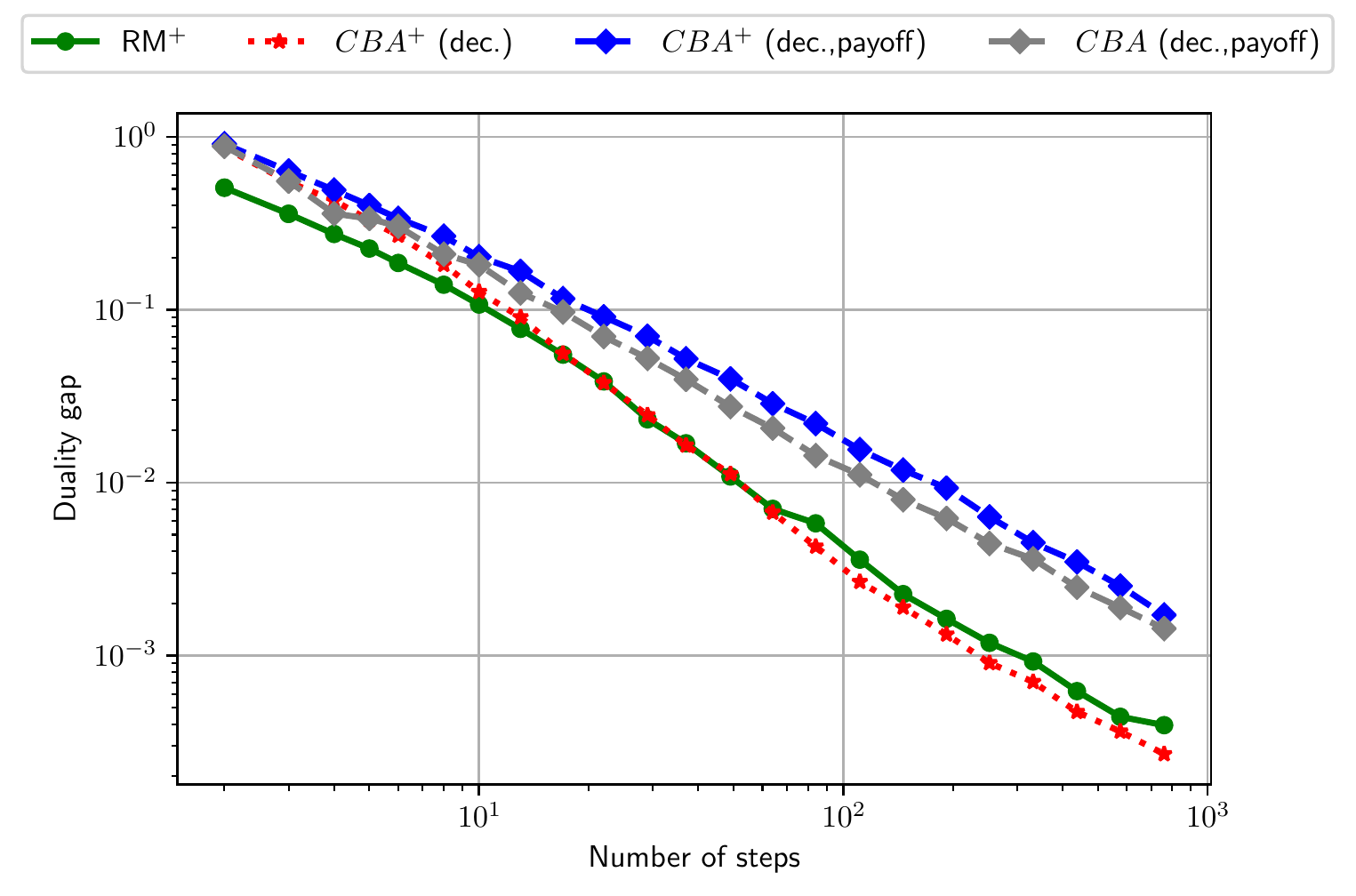}
         \caption{Normal distribution.}
  \end{subfigure}
  \caption{Comparison of \rmp{} vs.  \cbap{} and \cba{} with different linear averaging schemes: only on decisions (\cbap{} (dec.)), or on both the decisions and the payoffs $\bm{u}$ (\cbap{} (dec.,payoff),\cba{} (dec.,payoff)).}
  \label{fig:comparison-averaging-schemes}
\end{figure}
\tb{
\subsection{Geometric intuition on the projection step of \cbap}\label{app:figure-projection-cone}
Figure \ref{fig:projection-cone} illustrates the projection step $\pi_{\C}(\cdot)$ of $\cbap$. At a high level, from $\bm{u}_{t}$ to $\bm{u}_{t+1}$, an instantaneous payoff $\bm{v}_{t}$ is first added to $\bm{u}_{t}$ (where $\bm{v}_{t} = \left(\langle \bm{f}_{t},\bm{x}_{t} \rangle / \kappa, - \bm{f}_{t} \rangle \right)$), and then the resulting vector $\bm{u}^{+}_{t}=\bm{u}_{t}+\bm{v}_{t}$ is projected onto $\C$. The projection $\pi_{\C}(\cdot)$ moves the vector $\bm{u}^{+}_{t}$ along the edges of the cone $\C^{\circ}$, preserving the (orthogonal) distance $d$ to $\C^{\circ}$.
\begin{figure}[hbt]
\centering
   \begin{subfigure}{0.4\textwidth}

\centering
         \includegraphics[width=1.0\linewidth]{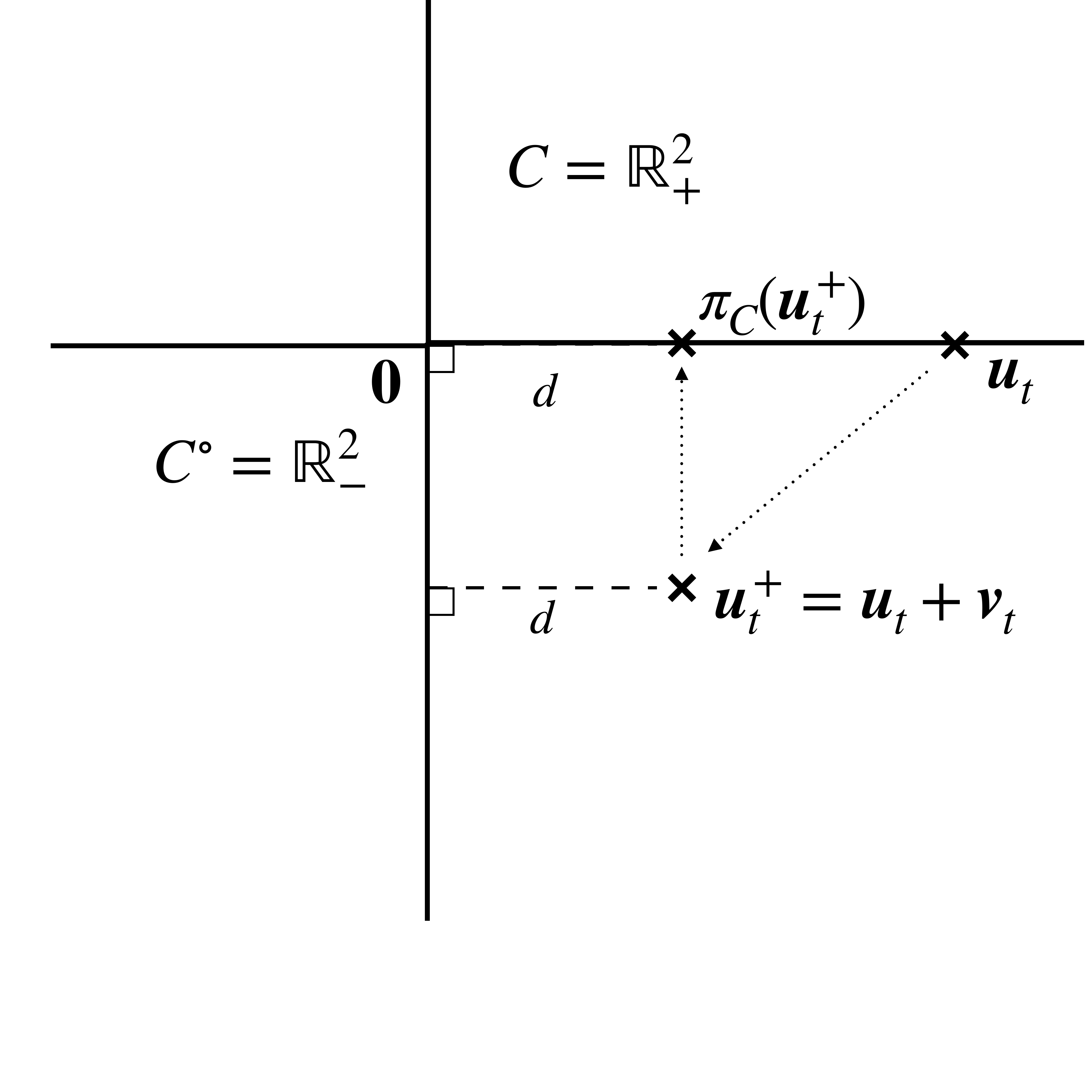}
  \end{subfigure}
     \begin{subfigure}{0.4\textwidth}
\centering
         \includegraphics[width=1.0\linewidth]{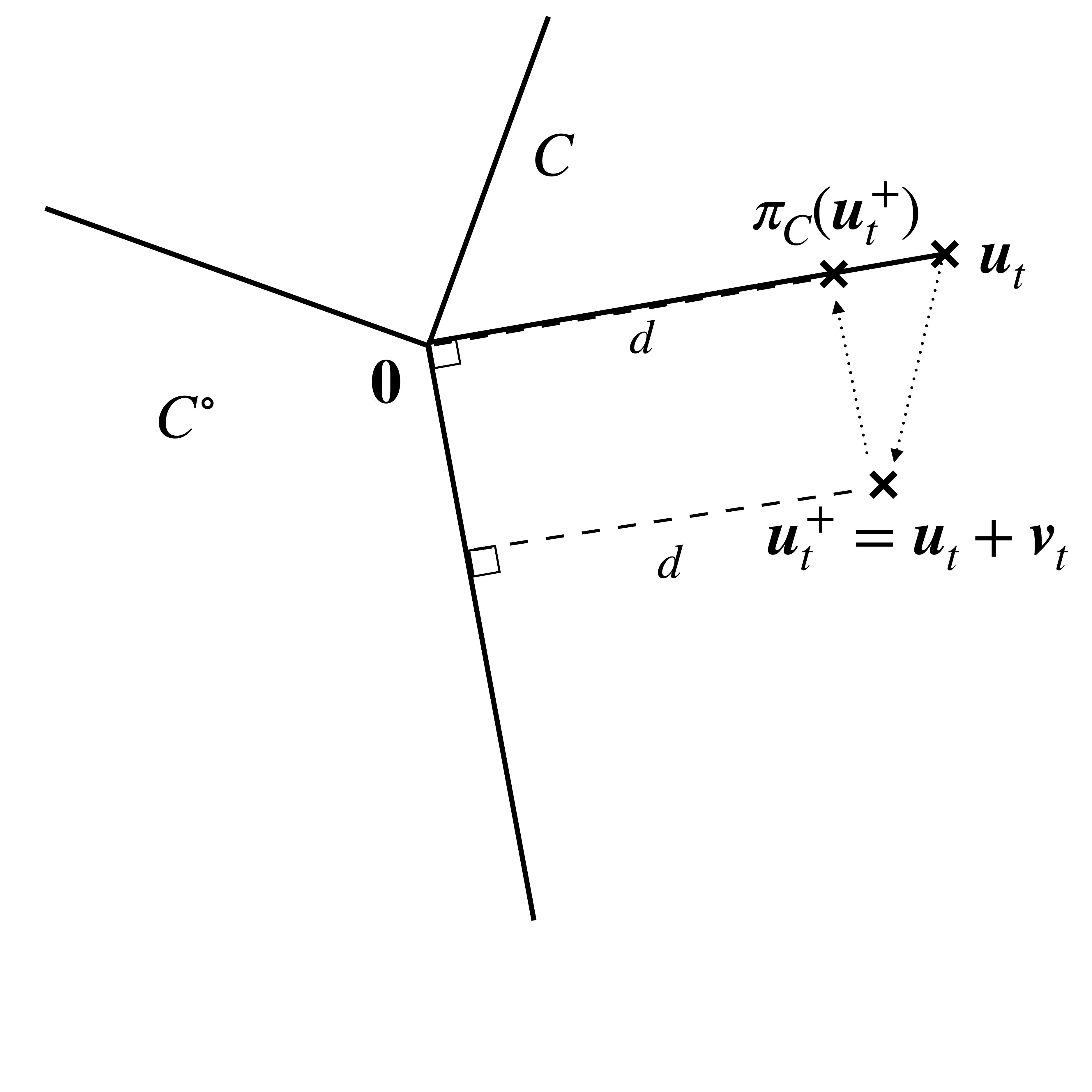}
  \end{subfigure}
  \caption{Illustration of $\pi_{\C}(\cdot)$ for $\C=\R_{+}^{2}$ (left-hand side) and $\C$ any cone in $\R^{2}$ (right-hand side).}
  \label{fig:projection-cone}
\end{figure}}
\section{Proof of Theorem \ref{th:folk-final}}\label{app:proof-folk-final}
Let $\omega_{t} =t, S_{T} = \sum_{t=1}^{T} \omega_{t} = T(T+1)/2$, and
\[  \bar{\bm{x}}_{T} = \frac{1}{S_{T}} \sum_{t=1}^{T} \omega_{t}\bm{x}_{t}, \bar{\bm{y}}_{T} = \frac{1}{S_{T}} \sum_{t=1}^{T} \omega_{t}\bm{y}_{t}.\] 
Since $F$ is convex-concave,  we first have
\[ \max_{\bm{y} \in \YY}F(\bar{\bm{x}}_{T},\bm{y}) - \min_{\bm{x} \in \XX} F(\bm{x},\bar{\bm{y}}_{T}) \leq \frac{2}{T(T+1)} \left(  \max_{\bm{y} \in \YY}  \sum_{t=1}^{T} \omega_{t}F(\bm{x}_{t},\bm{y}) - \min_{\bm{x} \in \XX} \sum_{t=1}^{T} \omega_{t}F(\bm{x},\bm{y}_{t}) \right).\]
Now,
\begin{align*}
 \max_{\bm{y} \in \YY}  \sum_{t=1}^{T} \omega_{t}F(\bm{x}_{t},\bm{y}) - \min_{\bm{x} \in \XX} \sum_{t=1}^{T} \omega_{t}F(\bm{x},\bm{y}_{t})  & =  \left( \max_{\bm{y} \in \YY}  \sum_{t=1}^{T} \omega_{t}F(\bm{x}_{t},\bm{y}) -  \sum_{t=1}^{T} \omega_{t} F(\bm{x}_{t},\bm{y}_{t}) \right) \\
 & + \left( \sum_{t=1}^{T} \omega_{t} F(\bm{x}_{t},\bm{y}_{t})  - \min_{\bm{x} \in \XX} \sum_{t=1}^{T} \omega_{t}F(\bm{x},\bm{y}_{t}) \right).
\end{align*}
Now since $F$ is convex-concave, we can use the following upper bound:
\begin{align*}
\max_{\bm{y} \in \YY}  \sum_{t=1}^{T} \omega_{t}F(\bm{x}_{t},\bm{y}) -  \sum_{t=1}^{T} \omega_{t} F(\bm{x}_{t},\bm{y}_{t}) &  \leq \max_{\bm{y} \in \YY}  \omega_{t} \sum_{t=1}^{T} \langle \bm g_t ,\bm{y} \rangle - \sum_{t=1}^{T}  \omega_{t} \langle \bm g_t,\bm{y}_{t} \rangle, \\
 \sum_{t=1}^{T} \omega_{t} F(\bm{x}_{t},\bm{y}_{t})  - \min_{\bm{x} \in \XX} \sum_{t=1}^{T} \omega_{t}F(\bm{x},\bm{y}_{t}) & \leq  \sum_{t=1}^{T}  \omega_{t} \langle \bm f_t, \bm x_t \rangle - \min_{\bm{x} \in \XX} \sum_{t=1}^{T} \omega_{t} \langle \bm f_t, \bm x \rangle,
\end{align*}
where $\bm f_t \in \partial_{\bm x} F(\bm x_t,\bm y_t),\bm g_t \in \partial_{\bm y} F(\bm x_t,\bm y_t)$ (recall the repeated game framework presented  at the beginning of Section \ref{sec:setup}).

Now we have proved in Theorem \ref{th:cbap-linear-averaging-only-policy} that 
\begin{align*}
\frac{2}{T(T+1)} \max_{\bm{y} \in \YY} \sum_{t=1}^{T} \omega_{t} \langle \bm g_t ,\bm{y} \rangle - \sum_{t=1}^{T} \langle  \omega_{t} \bm g_t,\bm{y}_{t} \rangle & = O \left( \kappa L / \sqrt{T} \right), \\
\frac{2}{T(T+1)} \sum_{t=1}^{T}  \omega_{t} \langle \bm f_t, \bm x_t \rangle - \min_{\bm{x} \in \XX} \sum_{t=1}^{T}  \omega_{t} \langle \bm f_t, \bm x \rangle & = O \left( \kappa L / \sqrt{T} \right).
\end{align*}
Therefore, we can conclude that 
\[\max_{\bm{y} \in \YY}F(\bar{\bm{x}}_{T},\bm{y}) - \min_{\bm{x} \in \XX} F(\bm{x},\bar{\bm{y}}_{T}) = O \left( \kappa L / \sqrt{T} \right).\]
\begin{remark}
Note that we essentially reprove the folk theorem, except that we consider \textit{weighted} average for the decisions of both players.  This is because Theorem \ref{th:folk-final} uses linear averaging on decisions, whereas Theorem \ref{thm:folk theorem} is written with uniform averaging on decisions.
\end{remark}
\section{Proofs of the projections of Section \ref{sec:efficient-projection}}\label{app:proofs-of-projections}
We will extensively use \textit{Moreau's Decomposition Theorem}~\citep{combettes2013moreau}: for any convex cone $\C \subset \R^{n+1}$ and $\bm{u} \in \R^{n+1}$, we can decompose $\bm{u} = \pi_{\C}(\bm{u}) + \pi_{\C^{\circ}}(\bm{u})$, where $\C^{\circ}$ is the \textit{polar cone} of $\C$. Therefore, to compute $\pi_{\C}(\bm{u})$, it is sufficient to compute $\pi_{\C^{\circ}}(\bm{u})$, the orthogonal projection of $\bm{u}$ onto $\C^{\circ}$. We will see that in some cases, it is simpler to compute $\pi_{\C^{\circ}}(\bm{u})$ and then use $\pi_{\C}(\bm{u}) = \bm{u} - \pi_{C^{\circ}}(\bm{u})$ than directly computing $\pi_{C}(\bm{u})$  via solving \eqref{eq:projection-onto-C}.
\subsection{The case of the simplex}\label{app:proof-proj-simplex}
We consider $\XX = \Delta(n)$.
Note that in this case, $\kappa=\max_{ \bm{x} \in \Delta(n)} \| \bm{x} \|_{2} = 1$.   
The next lemma gives a closed-form expression of $\C^{\circ}$. 
\begin{lemma}\label{lem:simplex-C-polar} 
Let $\C = \textrm{cone}(\{1\} \times \Delta(n))$. Then $ \C^{\circ} = \{ (\tilde{y},\hat{\bm{y}}) \in \R^{n+1} \; | \; \max_{i=1,...,n} \hat{y}_{i} \leq - \tilde{y} \}.$
\end{lemma}
\begin{proof}[Proof of Lemma \ref{lem:simplex-C-polar}]
 Note that for $\bm{y}=(\tilde{y},\hat{\bm{y}}) \in \R^{n+1}$ we have
\begin{align*}
\bm{y} \in \C^{\circ} &  \iff \langle \bm{y},\bm{z} \rangle \leq 0, \forall \; \bm{z} \in \C \\
& \iff \langle (\tilde{y},\hat{\bm{y}}), \alpha(1,\bm{x}) \rangle \leq 0, \forall \; \bm{x} \in \Delta(n), \forall \; \alpha \geq 0 \\
& \iff \tilde{y} + \langle \hat{\bm{y}},\bm{x} \rangle \leq 0, \forall \; \bm{x} \in \Delta(n) \\
& \iff \max_{\bm{x} \in \Delta(n)} \langle \hat{\bm{y}},\bm{x} \rangle \leq - \tilde{y} \\
& \iff \max_{i=1,...,n} \hat{y}_{i} \leq - \tilde{y}.
\end{align*}
\end{proof}
For a given $\bm{u} = \left(\tilde{u},\hat{\bm{u}}\right)$,  computing $\pi_{\C^{\circ}}(\bm{u})$ is now equivalent to solving
\begin{equation}\label{eq:projection-onto-C-polar-simplex}
\min \{ ( \tilde{y}-\tilde{u} )^{2} +  \| \hat{\bm{y}}-\hat{\bm{u}} \|_{2}^{2}  \; | \; (\tilde{y},\hat{\bm{y}}) \in \R^{n+1}, \max_{i=1,...,n} \hat{y}_{i} \leq - \tilde{y} \}. 
\end{equation}
Using the reformulation \eqref{eq:projection-onto-C-polar-simplex},  we show that for a fixed $\tilde{y}$, the optimal $\hat{\bm{y}}(\tilde{y})$ can be computed in closed-form.  It is then possible to avoid a binary search over $\tilde{y}$ and to simply use a sorting algorithm to obtain the optimal $\tilde{y}$.  The next proposition summarizes our complexity result for $\XX = \Delta(n)$.
\begin{proposition}
An optimal solution $\pi_{\C^{\circ}}(\bm{u})$ to \eqref{eq:projection-onto-C-polar-simplex} can be computed in $O(n \log(n))$ time.
\end{proposition}
\begin{proof}
 Computing $\pi_{\C^{\circ}}(\bm{u})$ is equivalent to computing
\begin{align*}
\min \;  & ( \tilde{y}-\tilde{u} )^{2} +  \| \hat{\bm{y}}-\hat{\bm{u}} \|_{2}^{2} \\
& \tilde{y} \in \R,\hat{\bm{y}} \in \R^{n}, \\
& \max_{i=1,...,n} \hat{y}_{i} \leq - \tilde{y}.
\end{align*}
Let us fix $\tilde{y} \in \R$ and let us first solve
\begin{equation}\label{eq:y-hat-of-y-tilde}
\begin{aligned}
\min \;  &  \| \hat{\bm{y}}-\hat{\bm{u}} \|_{2}^{2} \\
&\hat{\bm{y}} \in \R^{n}, \\
& \max_{i=1,...,n} \hat{y}_{i} \leq - \tilde{y}.
\end{aligned}
\end{equation}

This is essentially the projection of $\hat{\bm{u}}$ on $(-\infty,-\tilde{y}]^{n}$. So a solution to \eqref{eq:y-hat-of-y-tilde} is $ \hat{y}_{i}(\tilde{y}) = \min \{-\tilde{y}, \hat{u}_{i} \}, \forall \; i=1,...,n.$
Note that in this case we have
$ \hat{\bm{u}} - \hat{\bm{y}}(\tilde{y}) = \left( \hat{\bm{u}} + \tilde{y}\bm{e} \right)^{+}.$
So overall the projection brings down to the optimization of $\phi: \R \mapsto \R_{+}$ such that
\begin{equation}\label{eq:function-F}
\phi: \tilde{y} \mapsto  (\tilde{y}-\tilde{u} )^{2} + \| \left( \hat{\bm{u}} + \tilde{y}\bm{e} \right)^{+} \|_{2}^{2}.
\end{equation}
In principle, we could use binary search with a doubling trick to compute a $\epsilon$-minimizer of the convex function $\phi$ in $O\left( \log(\epsilon^{-1}) \right)$ calls to $\phi$. However, it is possible to a minimizer $\tilde{y}^{*}$ of $\phi$ using the following remark.
By construction,  we know that $\bm{u} - \pi_{\C^{\circ}}(\bm{u}) \in \C$. Here,  $\C = \textrm{cone}\left( \{1\} \times \Delta(n) \right)$,  and $\bm{u} - \pi_{\C^{\circ}}(\bm{u}) = \left( \tilde{u}-\tilde{y}^{*},\left( \hat{\bm{u}} + \tilde{y}^{*}\bm{e} \right)^{+} \right).$ 
This proves that
\[ \dfrac{\left( \hat{\bm{u}} + \tilde{y}^{*}\bm{e} \right)^{+} }{\tilde{u}-\tilde{y}^{*}} \in \Delta(n),\]
which in turns imply that 
\begin{equation}\label{eq:simple-eq-y-tilde}
 \tilde{y}^{*}+ \sum_{i=1}^{n} \max \{ \hat{u}_{i} + \tilde{y}^{*},0 \}=\tilde{u}.
\end{equation}
We can use \eqref{eq:simple-eq-y-tilde} to efficiently compute $\tilde{y}^{*}$ without using any binary search. In particular, we can sort the coefficients of $\hat{\bm{u}}$ in $O\left( n \log(n) \right)$ operations, and use \eqref{eq:simple-eq-y-tilde} to find $\tilde{y}^{*}$.
\end{proof}
Having obtained $\pi_{\C^{\circ}}(\bm{u})$, we can obtain $\pi_{\C}(\bm{u})$ by using the identity $\pi_{\C}(\bm{u}) = \bm{u} - \pi_{\C^{\circ}}(\bm{u})$.  Note that \rmm{} and \rmp{} are obtained by choosing the closed-form feasible point corresponding to $\tilde{y} = \tilde{u}$ in \eqref{eq:projection-onto-C-polar-simplex}.  
\subsection{The case of an $\ell_{p}$ ball}\label{app:proof-ell-p-ball}
In this section we assume that $\XX = \{ \bm{x} \in \R^{n} \; | \ ; \| \bm{x} \|_{p} \leq 1\}$ with $p  \geq 1$ or $p=\infty$. 
The next lemma provides a closed-form reformulation of the polar cone $\C^{\circ}$.
\begin{lemma}\label{lem:ball-p-C-polar}
Let $\XX = \{ \bm{x} \in \R^{n} \; | \ ; \| \bm{x} \|_{p} \leq 1\},$ with $p  \geq 1$ or $p=\infty$.
 Then $\C^{\circ} = \{ (\tilde{y},\bm{y}) \in \R \times \R^{n} \; | \; \| \bm{y} \|_{q} \leq - \tilde{y} \}$,  with $q$ such that $1/p + 1/q =1$.
\end{lemma}
\begin{proof}[Proof of Lemma \ref{lem:ball-p-C-polar}]
Let us write $B_{p}(1) = \{ \bm{z} \in \R^{n} \; | \; \| \bm{z} \|_{p} \leq 1\}$.
 Note that for $\bm{y}=(\tilde{y},\hat{\bm{y}}) \in \R^{n+1}$ we have
\begin{align*}
\bm{y} \in \C^{\circ} &  \iff \langle \bm{y},\bm{z} \rangle \leq 0, \forall \; \bm{z} \in \C \\
& \iff \langle (\tilde{y},\hat{\bm{y}}), \alpha(1,\bm{x}) \rangle \leq 0, \forall \; \bm{x} \in B_{p}(1), \forall \; \alpha \geq 0 \\
& \iff \tilde{y} + \langle \hat{\bm{y}},\bm{x} \rangle \leq 0, \forall \; \bm{x} \in B_{p}(1), \\
& \iff \max_{\bm{x} \in B_{p}(1),} \langle \hat{\bm{y}},\bm{x} \rangle \leq - \tilde{y} \\
& \iff \| \bm{x} \|_{q} \leq - \tilde{y},
\end{align*}
since $\| \cdot \|_{q}$ is the dual norm of $\| \cdot \|_{p}$.
\end{proof}
The orthogonal projection problem onto $\C^{\circ}$ becomes
\begin{equation}\label{eq:projection-onto-C-polar-ell-p}
\min \{ ( \tilde{y}-\tilde{u} )^{2} +  \| \hat{\bm{y}}-\hat{\bm{u}} \|_{2}^{2}  \; | \; (\tilde{y},\hat{\bm{y}}) \in \R^{n+1}, \| \hat{\bm{y}} \|_{q} \leq - \tilde{y} \}. 
\end{equation}
For $p=2$, \eqref{eq:projection-onto-C-polar-ell-p} has a closed-form solution. 
For $p =1$, a quasi-closed-form solution to \eqref{eq:projection-onto-C-polar-ell-p} can be obtained  efficiently using sorting.  For $p = \infty$, it is more efficient to directly compute $\pi_{\C}(\bm{u})$. This is because the dual norm of $\| \cdot \|_{\infty}$ is $\| \cdot \|_{1}$. 
\begin{proposition}\label{prop:complexity-computing-ortho-proj-norm-p}
\begin{itemize}
\item  For $p=1$, $\pi_{\C^{\circ}}(\bm{u})$ can be computed in $O(n \log(n))$ arithmetic operations.  
\item For $p=\infty$, $\pi_{\C}(\bm{u})$ can be computed in $O(n \log(n))$ arithmetic operations.  
\item For $p = 2$,  $\pi_{\C}(\bm{u})$ can be computed in closed-form.
\end{itemize}
\end{proposition}
\begin{proof}
\textbf{The case $p=1$.}
Assume that $p=1$.  Then $\| \cdot \|_{q} = \| \cdot \|_{\infty}$. 
We want to compute the projection of $\left(\tilde{u},\hat{\bm{u}} \right)$ on $\C^{\circ}$:
\begin{equation}\label{eq:ortho-proj-onto-polar-norm-1}
\begin{aligned}
\min_{\bm{y} \in \C^{\circ}} \| \bm{y} - \bm{u} \|_{2}^{2} = \min & \; (\tilde{y}-\tilde{u})^{2} + \| \hat{\bm{y}}-\hat{\bm{u}} \|_{2}^{2} \\
& \; \tilde{y} \in \R, \hat{\bm{y}} \in \R^{n}, \\
& \| \hat{\bm{y}} \|_{\infty} \leq -\tilde{y}.
\end{aligned}
\end{equation}
For a fixed $\tilde{y}$, we want to compute
\begin{equation}\label{eq:proj-norm-q-y-hat-infty}
\begin{aligned}
 \min & \; \| \hat{\bm{y}}-\hat{\bm{u}} \|_{2}^{2} \\
& \hat{\bm{y}} \in \R^{n}, \\
& \| \hat{\bm{y}} \|_{\infty} \leq -\tilde{y}.
\end{aligned}
\end{equation}
 The projection \eqref{eq:proj-norm-q-y-hat-infty} can be computed in closed-form as
 \begin{equation}\label{eq:closed-form-tilde-y-for-p-1} 
 \hat{\bm{y}}^{*}(\tilde{y})  =  \min \{ - \tilde{y}, \max \{ \tilde{y}, \hat{\bm{u}} \} \}
 \end{equation}
since this is simply the orthogonal projection of $\hat{\bm{u}}$ onto the $\ell_{\infty}$ ball of radius $-\tilde{y}$.  Let us call $\phi: \R \mapsto \R$ such that
\[ \phi(\tilde{y}) = \left( \tilde{y} - \tilde{u} \right)^{2} + \| \hat{\bm{y}}^{*}(\tilde{y}) - \hat{\bm{u}} \|_{2}^{2}.\]
Because of the closed-form expression for $ \hat{\bm{y}}^{*}(\tilde{y}) $ as in \eqref{eq:closed-form-tilde-y-for-p-1}, we have
 \[ \phi: \tilde{y} \mapsto  \left( \tilde{y} - \tilde{u} \right)^{2}  + \| \left( \hat{\bm{u}} + \tilde{y}\bm{e} \right)^{+} \|_{2}^{2}.\]
Finding a minimizer of $\phi$ can be done in $O(n\log(n))$,  with the same methods as in the proof in the previous section (Appendix \ref{app:proof-proj-simplex}).

\vspace{2mm}
\noindent
\textbf{The case $p=\infty$.}
Let $p=\infty$.  The problem of computing $\pi_{\C}(\bm{u})$, the orthogonal projection onto the cone $\C$,  is equivalent to 
\begin{equation}\label{eq:ortho-proj-onto-cone-norm-infinity}
\begin{aligned}
\min_{\bm{y} \in \C^{\circ}} \| \bm{y} - \bm{u} \|_{2}^{2} = \min & \; (\tilde{y}-\tilde{u})^{2} + \| \hat{\bm{y}}-\hat{\bm{u}} \|_{2}^{2} \\
& \; \tilde{y} \in \R, \hat{\bm{y}} \in \R^{n}, \\
& \| \hat{\bm{y}} \|_{\infty} \leq \tilde{y}.
\end{aligned}
\end{equation}
Note the similarity between \eqref{eq:ortho-proj-onto-cone-norm-infinity} (computing the orthogonal projection onto $\C$ when $p=\infty$), and \eqref{eq:ortho-proj-onto-polar-norm-1} (computing the orthogonal projection onto $\C^{\circ}$ when $p=1$). From Lemma \ref{lem:ball-p-C-polar}, we know that this is the case because $\| \cdot \|_{1}$ and $\| \cdot \|_{\infty}$ are dual norms to each other.

Therefore, the methods described for computing $\pi_{\C^{\circ}}(\bm{u})$ for $p=1$  can be applied to the case $p=\infty$  for directly computing $\pi_{\C}(\bm{u})$. This gives the complexity results as stated in Proposition \ref{prop:complexity-computing-ortho-proj-norm-p}: $\pi_{\C}(\bm{u})$ can be computed in $O(n\log(n))$ operations.

\vspace{2mm}
\noindent
\textbf{The case $p=2$.}
Let $\| \cdot \|_{p} = \| \cdot \|_{2}$, then $\| \cdot \|_{q} = \| \cdot \|_{2}$.  
Let us fix $\tilde{y}$ and consider solving
\begin{equation}\label{eq:proj-norm-q-y-hat-2}
\begin{aligned}
 \min & \; \| \hat{\bm{y}}-\hat{\bm{u}} \|_{2}^{2} \\
& \hat{\bm{y}} \in \R^{n}, \\
& \| \hat{\bm{y}} \|_{2} \leq -\tilde{y}.
\end{aligned}
\end{equation}
The projection \eqref{eq:proj-norm-q-y-hat-2} can be computed in closed-form as
\[ \hat{\bm{y}}^{*}(\tilde{y}) = \left( - \tilde{y} \right) \dfrac{\hat{\bm{u}}}{\| \hat{\bm{u}} \|_{2}},\]
since this is just the orthogonal projection of the vector $\hat{\bm{u}}$ onto the $\ell_{2}$-ball of radius $-\tilde{y}$.
Let us call $\phi: \R \mapsto \R$ such that
\[ \phi(\tilde{y}) = \left( \tilde{y} - \tilde{u} \right)^{2} + \| \hat{\bm{y}}^{*}(\tilde{y}) - \hat{\bm{u}} \|_{2}^{2}.\]
Note that here, $\tilde{y} \mapsto \hat{\bm{y}}^{*}(\tilde{y})$ is differentiable.  Therefore $\phi:\tilde{y} \mapsto \left( \tilde{y} - \tilde{u} \right)^{2} + \| \hat{\bm{y}}^{*}(\tilde{y}) - \hat{\bm{u}} \|_{2}^{2}$ is also differentiable.  First-order optimality conditions yield a closed-form solution for computing $(\tilde{y}^{*},\hat{\bm{y}}^{*})=\pi_{\C^{\circ}}(\bm{u})$, as 
\begin{equation}
\tilde{y}^{*} = \dfrac{\tilde{u} - \| \hat{\bm{u}} \|_{2}}{2}, \hat{\bm{y}}^{*} = -\dfrac{1}{2}(\tilde{u} - \| \hat{\bm{u}} \|_{2})\frac{\hat{\bm{u}}}{\| \hat{\bm{u}} \|_{2}}.
\end{equation}
\end{proof}
\subsection{The case of an ellipsoidal confidence region in the simplex}
\label{sec:CBA-subregion}
In this section we assume that $\XX$ is
$\XX = \{ \bm{x} \in \Delta(n) \; | \; \| \bm{x} - \bm{x}_{0} \|_{2} \leq \epsilon_{x} \}$. 
We also assume that $\{ \bm{x} \in \R^{n} | \bm{x}^{\top}\bm{e}=1 \} \bigcap \{ \bm{x} \in \R^{n} \; | \; \| \bm{x}- \bm{x}_{0} \|_{2} \leq \epsilon_{x} \} \subseteq \Delta(n)$,  so that we can write $\XX = \bm{x}_{0} + \epsilon \tilde{B},$ where
$ \tilde{B} = \{ \bm{z} \in \R^{n} \; | \; \bm{z}^{\top}\bm{e}=0, \| \bm{z} \|_{2} \leq 1\}.$

Suppose we made a sequence of decisions $\bm{x}_{1}, ...,\bm{x}_{T}$, which can be written as $\bm{x}_{t} = \bm{x}_{0} + \epsilon \bm{z}_{t}$ for $\bm{z}_{t} \in \tilde{B}.$ Then it is clear that for any sequence of payoffs $\bm{f}_{1},...,\bm{f}_{T}$, we have
\begin{align}\label{eq:first-reformulation-regret}
\sum_{t=1}^{T} \omega_{t} \langle \bm{f}_{t},\bm{x}_{t} \rangle - \min_{\bm{x} \in \XX}  \sum_{t=1}^{T} \omega_{t} \langle \bm{f}_{t},\bm{x} \rangle = \epsilon_{x} \left( \sum_{t=1}^{T} \omega_{t} \langle \bm{f}_{t},\bm{z}_{t} \rangle - \min_{\bm{z} \in \tilde{B}}  \sum_{t=1}^{T}  \omega_{t}\langle \bm{f}_{t},\bm{z} \rangle \right).
\end{align}
Therefore,  if we run \cbap{} on the set $\tilde{B}$ to obtain $O \left( \sqrt{T} \right)$ growth of the right-hand side of \eqref{eq:first-reformulation-regret},  we obtain a no-regret algorithm for $\XX$. We now show how to run \cbap{} for the set $\tilde{B}$.
Let $\V = \{ \bm{v} \in \R^{n} \; | \; \bm{v}^{\top}\bm{e}=0\}.$
 We use the following orthonormal basis of $\V$:
 let $\bm{v}_{1}, ..., \bm{v}_{n-1} \in \R^{n}$ be the vectors
$\bm{v}_{i} = \sqrt{i/(i+1)} \left(1/i, ..., 1/i, -1, 0,...,0 \right), \forall \; i=1,...,n-1,$ where the component $1/i$ is repeated $i$ times.
The vectors $\bm{v}_{1},...,\bm{v}_{n-1}$ are orthonormal and constitute a basis of $\V$ \citep{egozcue2003isometric}.  Writing $\bm{V} = \left(\bm{v}_{1},...,\bm{v}_{n-1} \right) \in \R^{n \times (n-1)}$,  and noting that $\bm{V}^{\top}\bm{V} = \bm{I}$, we can write $\tilde{B} = \{ \bm{Vs} \; | \; \bm{s} \in \R^{n-1},\| \bm{s} \|_{2} \leq 1 \}.$
Now, if $\bm{x} = \bm{x}_{0} + \epsilon_{x} \bm{z}_{t}$ with $\bm{z}_{t} \in \V$, we have $\bm{z}_{t} = \bm{Vs}_{t}$,  for $\bm{s}_{t} \in \R^{n-1}$ and $\| \bm{s} \|_{2} \leq 1$.  Finally,
\begin{align}\label{eq:second-reformulation-regret}
\sum_{t=1}^{T} \omega_{t} \langle \bm{f}_{t},\bm{x}_{t} \rangle - \min_{\bm{x} \in \XX}  \sum_{t=1}^{T} \omega_{t} \langle \bm{f}_{t},\bm{x} \rangle =  \epsilon_{x} \left( \sum_{t=1}^{T} \omega_{t} \langle \bm{V}^{\top} \bm{f}_{t},\bm{s}_{t} \rangle - \min_{\bm{s} \in \R^{n-1},\|\bm{s}\|_{2} \leq 1}  \sum_{t=1}^{T} \omega_{t} \langle \bm{V}^{\top} \bm{f}_{t},\bm{s} \rangle \right).
\end{align}
Therefore,  to obtain a regret minimizer for the left-hand side of \eqref{eq:second-reformulation-regret} with observed payoffs $\left( \bm{f} \right)_{t \geq 0}$, we can run \cbap{} on the right-hand side,  where the decision set is an $\ell_{2}$ ball and the sequence of observed payoffs is $\left( \bm{V}^{\top}\bm{f}_{t} \right)_{t \geq 0}$.  
In the previous section we showed how to efficiently instantiate \cbap{} in this setting (see Proposition \ref{prop:complexity-computing-ortho-proj-norm-p}).
\begin{remark}
In this section we have highlighted a sequence of reformulations of the regret, from \eqref{eq:first-reformulation-regret} to \eqref{eq:second-reformulation-regret}.  We essentially showed how to instantiate \cbap{} for settings where the decision set $\XX$ is the intersection of an $\ell_{2}$ ball with a hyperplane for which we have an orthonormal basis.
\end{remark}
\section{Additional details and numerical experiments for matrix games and EFGs}\label{app:simu-matrix-games-and-EFGs}
\subsection{Numerical setup}\label{app:numerical-setup-details}
\paragraph{Numerical setup for matrix games}
For the experiments on matrix games, we sample at random the matrix of payoffs $\bm{A} \in \R^{n \times m}$ and we let $n,m=10$, where $n,m$ represent the number of actions of each player. We average our results over $70$ instances. The decision sets $\XX$ and $\YY$ are given as $\XX = \Delta(n)$ and $\YY = \Delta(m)$. 
\paragraph{Alternation} Alternation is a method which improves the performances of \rmm{} and \rmp~\citep{burch2019revisiting}. We leave proving this for \cba{} and \cbap{} to future works. Using alternation,  the players play in turn, instead of playing at the same time. In particular, the $y$-player may observe the current decision $\bm{x}_{t}$ of the $x$-player at period $t$, before choosing its own decision $\bm{y}_{t}$.  For \cba{} and \cbap, it is implemented as follows. At period $t \geq 2$,
 \begin{enumerate}
 \item The $x$-player chooses $\bm{x}_{t}$ using its payoff $\bm{u}^{x}_{t-1}: \bm{x}_{t} = \chp(\bm{u}^{x}_{t-1}).$
 \item The $y$-player observes $\bm{g} \in \partial_{\bm{y}} F(\bm{x}_{t},\bm{y}_{t-1})$ and updates $\bm{u}^{y}_{t}$:
 \[  \bm{u}^{y}_{t} = \upp_{\cbap}(\bm{u}_{t-1}^{y}, \bm{y}_{t-1},\bm{g},\omega_{t},\sum_{\tau=1}^{t-1} \omega_{\tau}).\]
 \item The $y$-player chooses $\bm{y}_{t}$ using $\bm{u}^{y}_{t-1}: \bm{y}_{t} = \chp(\bm{u}^{y}_{t}).$
 \item The $x$-player observes $\bm{f} \in \partial_{\bm{x}} F(\bm{x}_{t},\bm{y}_{t} )$ and updates $\bm{u}^{x}_{t}$:
\[  \bm{u}^{x}_{t} = \upp_{\cbap}(\bm{u}_{t}^{x}, \bm{x}_{t},\bm{f},\omega_{t},\sum_{\tau=1}^{t-1} \omega_{\tau}).\]
 \end{enumerate}
 \subsection{Comparing \rmm, \rmp, \cba, and \cbap{} on matrix games}\label{app:comparing-rm-cba} In Figure \ref{fig:two-player-simplex-uniform} and Figure \ref{fig:two-player-simplex-normal}, we show the performances of \rmm, \rmp, \cba{} and \cbap{} with and without alternation, and with and without linear averaging.  On the $y$-axis we show the duality gap of the current averaged decisions $(\bar{\bm{x}}_{T},\bar{\bm{y}}_{T})$. On the $x$-axis we show the number of iterations. 
 \tb{
\begin{itemize}
\item In Figure \ref{fig:subfig:comp-simplex-without-alt-uniform} and Figure \ref{fig:subfig:comp-simplex-without-alt-normal}, the four algorithms do not use alternation nor linear averaging, i.e., the four algorithms use uniform weights on the sequence of decisions. We note that \rmp{} is the best algorithm in this setting, outperforming \cbap, which performs similarly as \rmm.
\item In Figure \ref{fig:subfig:comp-simplex-without-alt-with-linavg-uniform} and Figure \ref{fig:subfig:comp-simplex-without-alt-with-linavg-normal}, we note that linear averaging (only on the decisions) slightly improves the performances of \rmp{} and \cbap{}. Note that \rmm{} and \cba{} are not known to be compatible with linear averaging (on decisions only), so we use uniform weights for \rmm{} and \cba{} here.
\item In Figures \ref{fig:subfig:comp-simplex-with-alt-uniform} and Figure \ref{fig:subfig:comp-simplex-with-alt-normal}, the four algorithms use alternation, but not linear averaging. The performances of the four algorithms are very similar.
\item Finally, in Figure \ref{fig:subfig:comp-simplex-with-lin-avg-uniform} and Figure \ref{fig:subfig:comp-simplex-with-lin-avg-normal}, \rmp{} and \cbap{} use linear averaging on decisions \textit{and} alternation.  We see that the strongest performances are achieved by \cbap. Recall that \rmm{} and \cba{} are not known to be compatible with linear averaging (on decisions only), so we do not show their performances here.
\item \textbf{Conclusion of our experiments.} We see that it is both alternation {\it and} linear averaging that enable the strong empirical performances of \cbap, rendering it capable to outperform \rmp.  Crucially, it is the ``$+$ operation'' that enables \cbap{} (and \rmp{}) to be compatible with linear averaging on the decisions only and to outperform \cba{} and \rmm.
\end{itemize}}

\begin{figure}[H]
 \begin{subfigure}{0.24\textwidth}
\centering
         \includegraphics[width=1.0\linewidth]{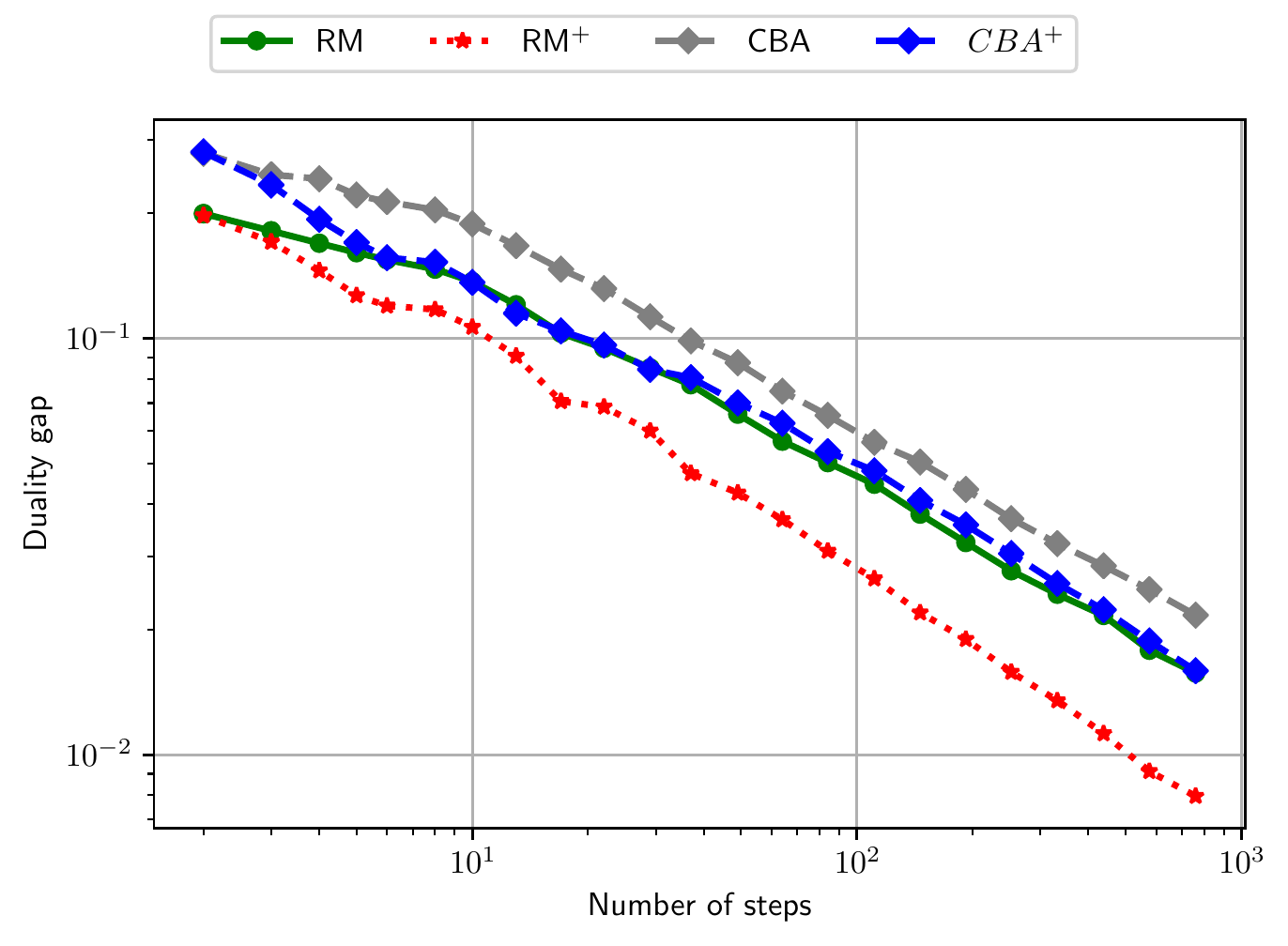}
         \caption{Without alternation nor linear averaging.}
         \label{fig:subfig:comp-simplex-without-alt-uniform}
  \end{subfigure}
     \begin{subfigure}{0.24\textwidth}
\centering
         \includegraphics[width=1.0\linewidth]{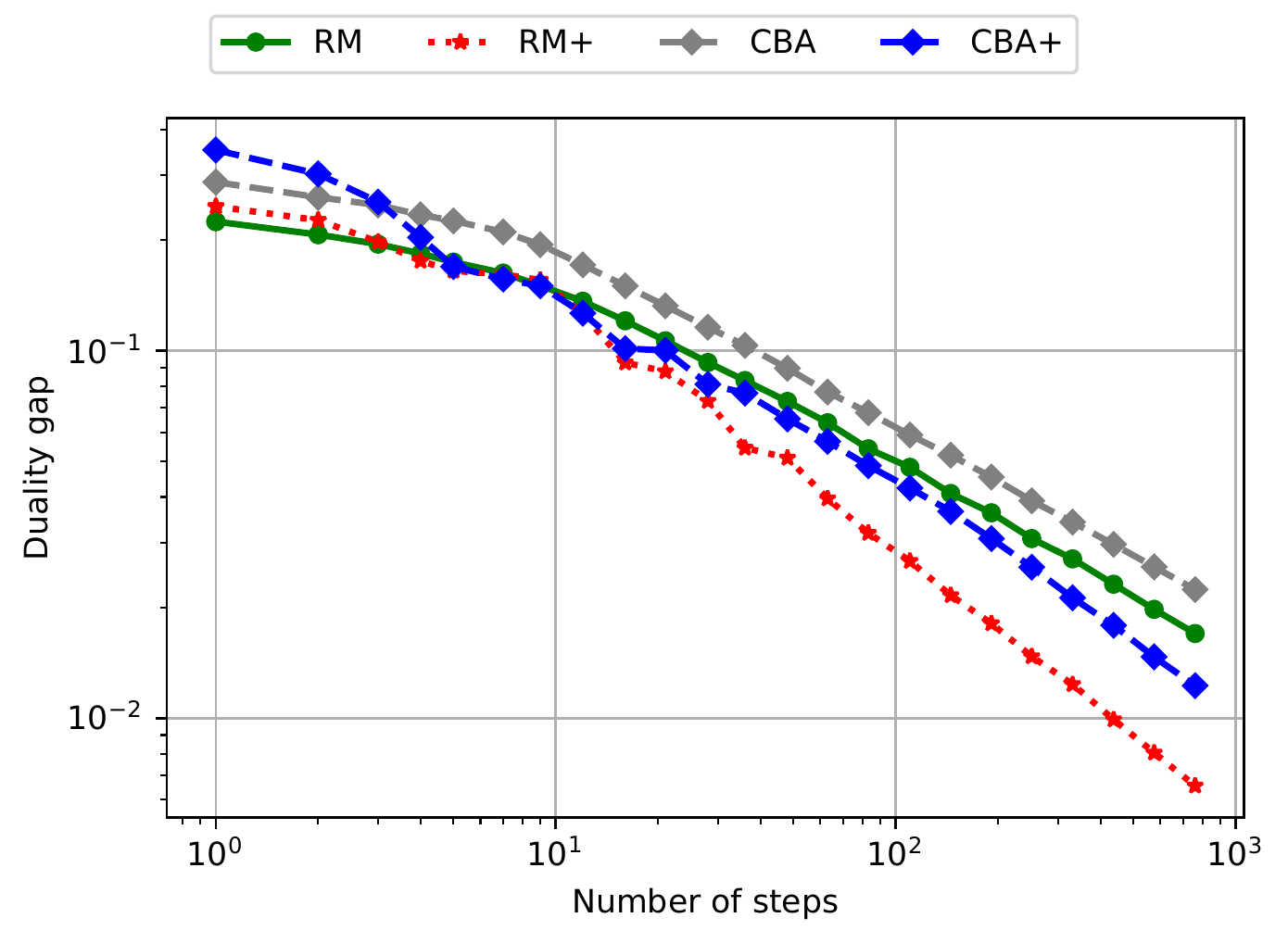}
         \caption{Without alternation but with linear averaging (only for \rmp{} and \cbap).}
          \label{fig:subfig:comp-simplex-without-alt-with-linavg-uniform}
  \end{subfigure}
     \begin{subfigure}{0.24\textwidth}
\centering
         \includegraphics[width=1.0\linewidth]{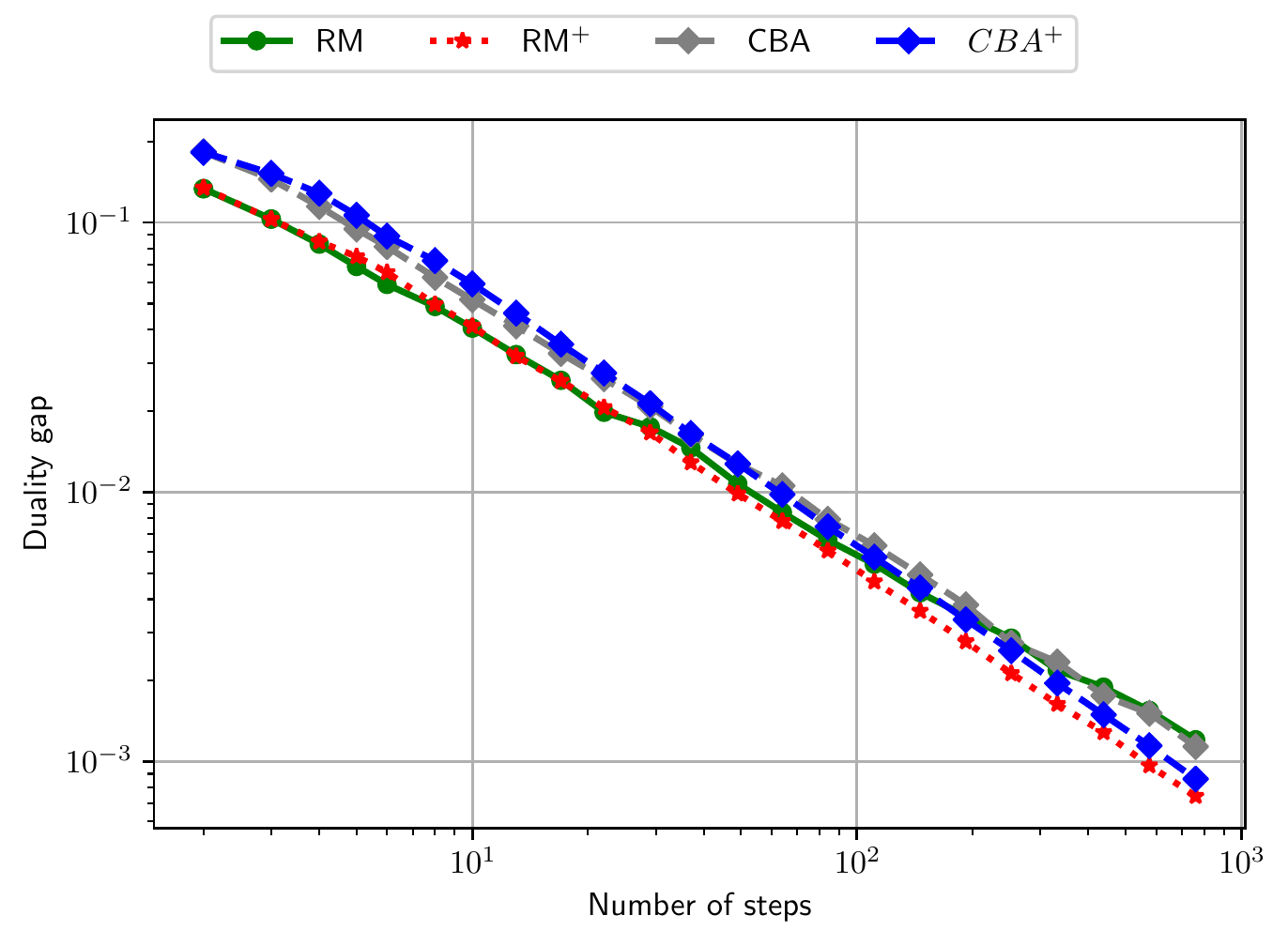}
         \caption{With alternation but no linear averaging.}
                  \label{fig:subfig:comp-simplex-with-alt-uniform}
  \end{subfigure}
     \begin{subfigure}{0.24\textwidth}
\centering
         \includegraphics[width=1.0\linewidth]{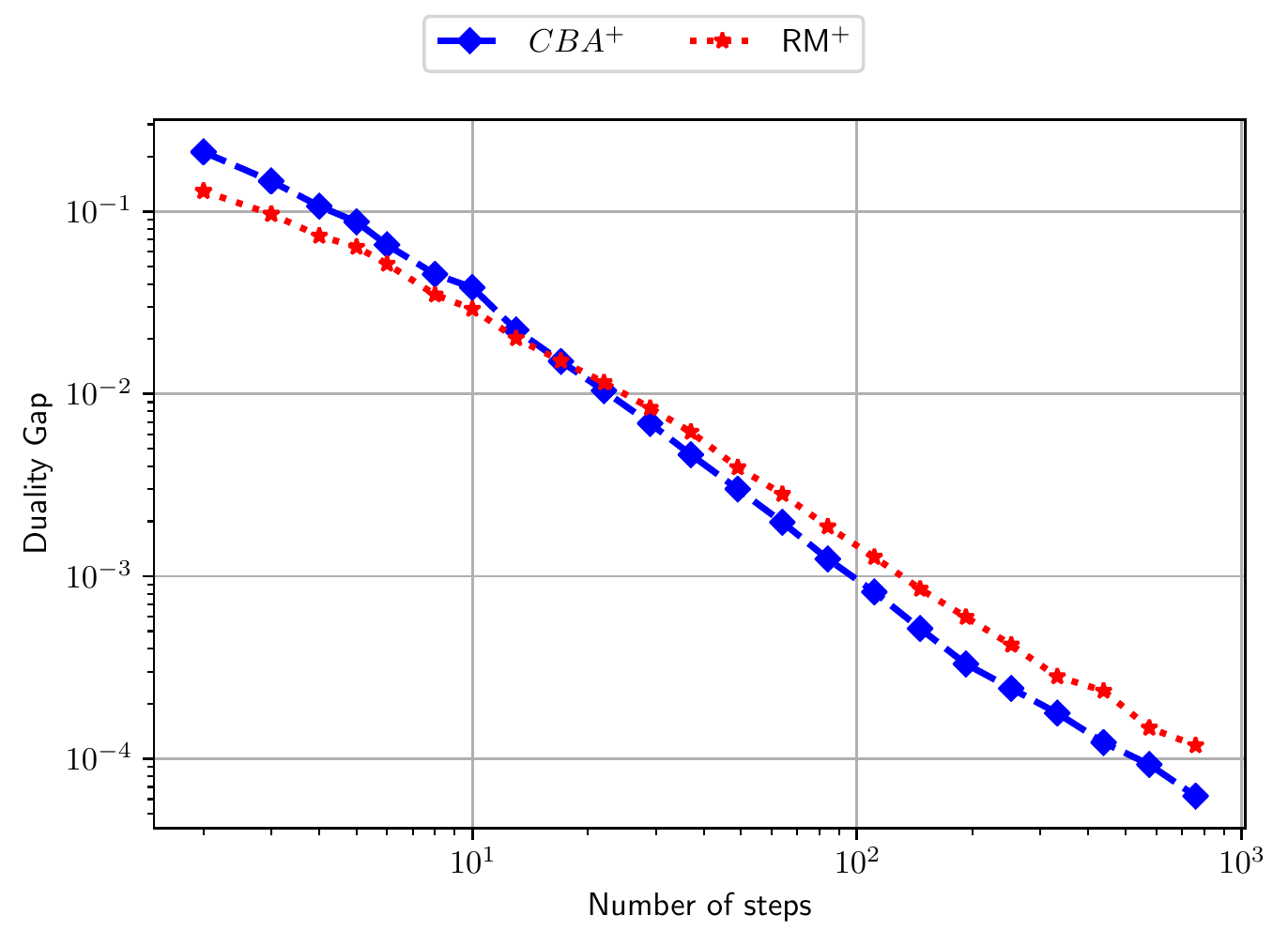}
         \caption{With alternation and linear averaging.}\label{fig:subfig:comp-simplex-with-lin-avg-uniform}
  \end{subfigure}
  \caption{Comparison of \rmm, \rmp, \cba{} and \cbap{} for $\XX = \Delta(n),\YY=\Delta(m)$ and random matrices, with and without alternations, and with and without linear averaging. We choose $n,m=10$ and $A_{ij} \sim U[0,1]$ over $70$ instances.}
  \label{fig:two-player-simplex-uniform}
\end{figure}

\begin{figure}[H]
 \begin{subfigure}{0.24\textwidth}
\centering
         \includegraphics[width=1.0\linewidth]{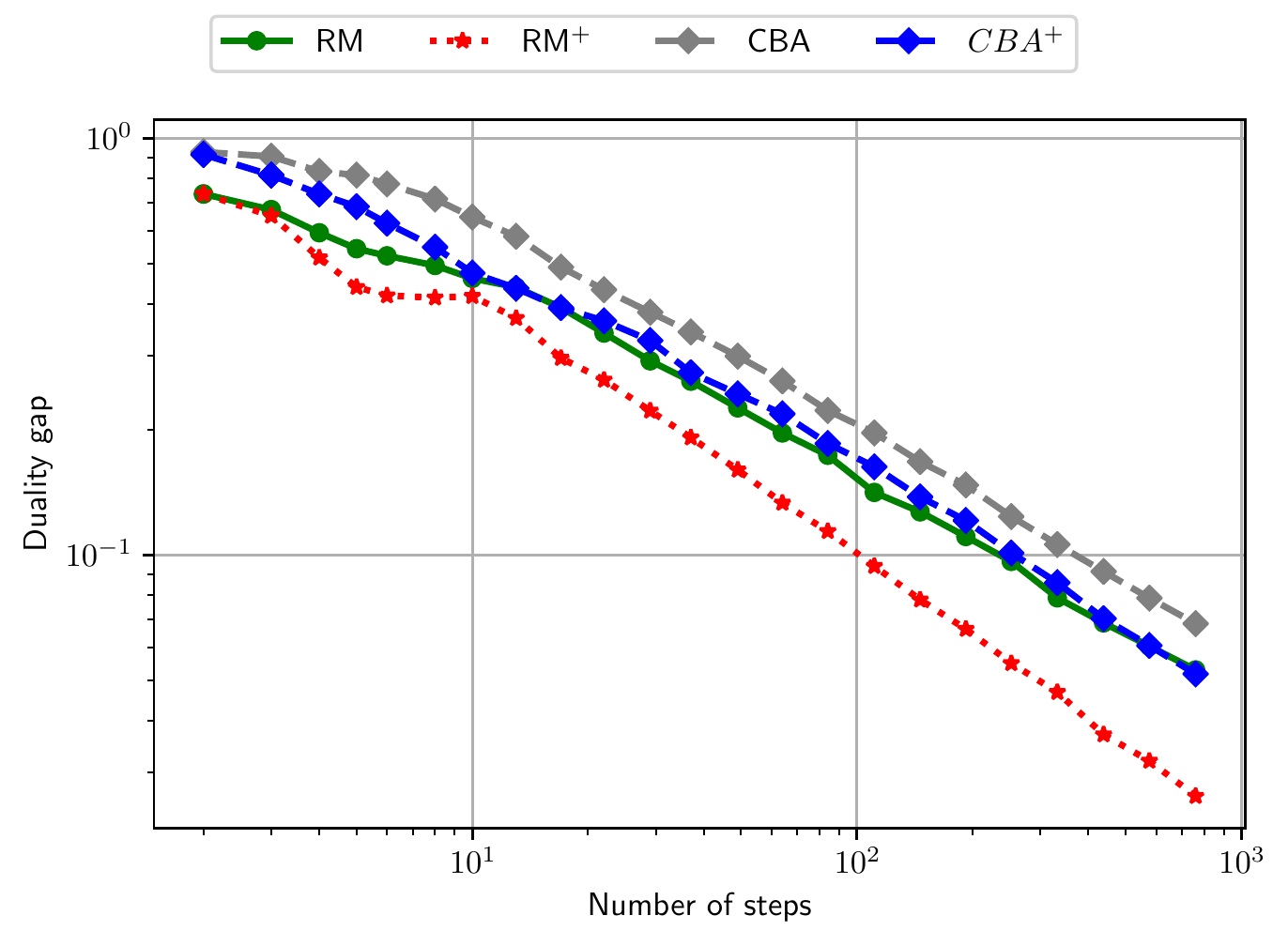}
         \caption{Without alternation nor linear averaging.}
                  \label{fig:subfig:comp-simplex-without-alt-normal}
  \end{subfigure}
              \begin{subfigure}{0.24\textwidth}
\centering
         \includegraphics[width=1.0\linewidth]{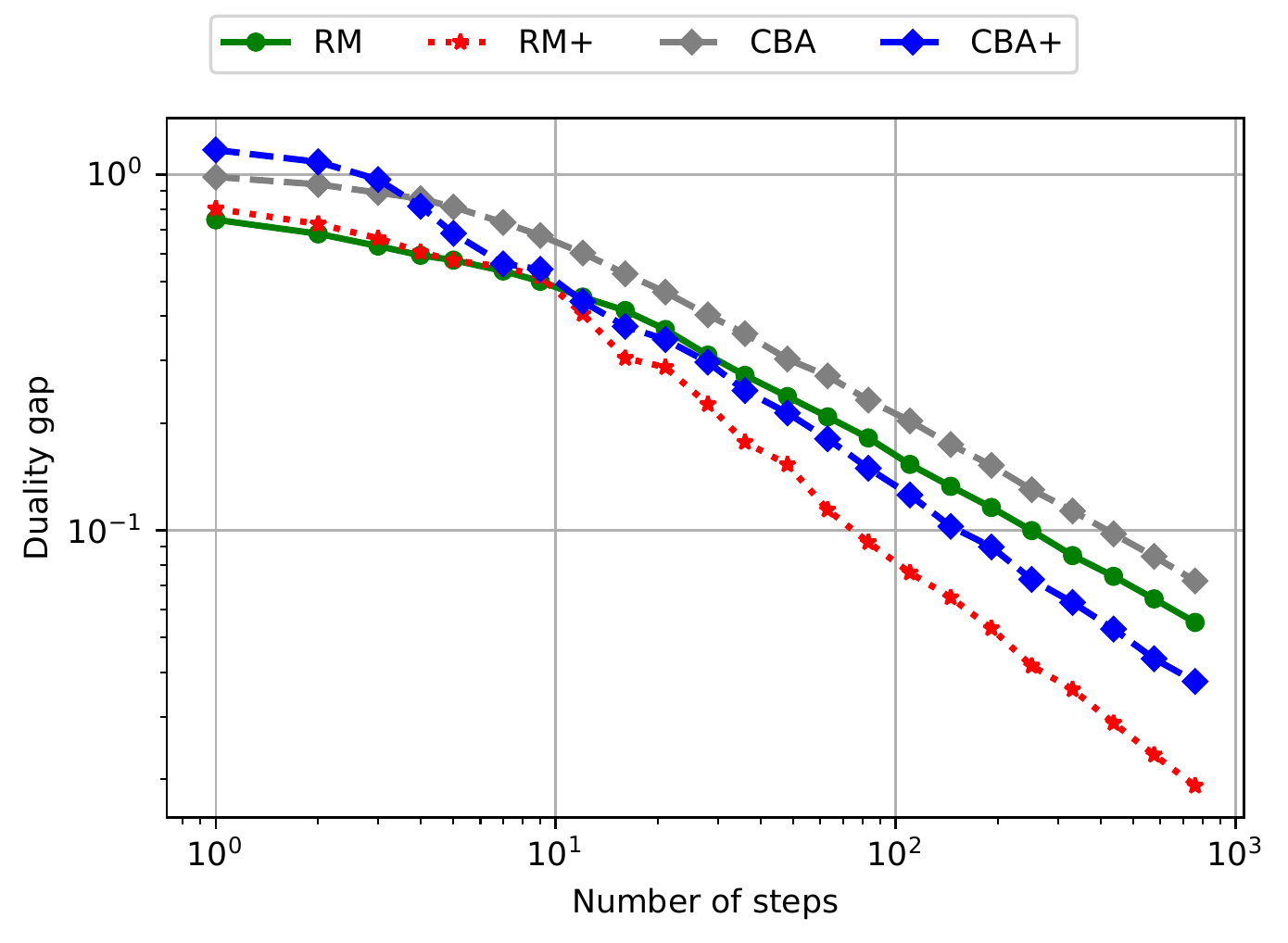}
         \caption{Without alternation but with linear averaging (only for \rmp{} and \cbap).}
          \label{fig:subfig:comp-simplex-without-alt-with-linavg-normal}
  \end{subfigure}
     \begin{subfigure}{0.24\textwidth}
\centering
         \includegraphics[width=1.0\linewidth]{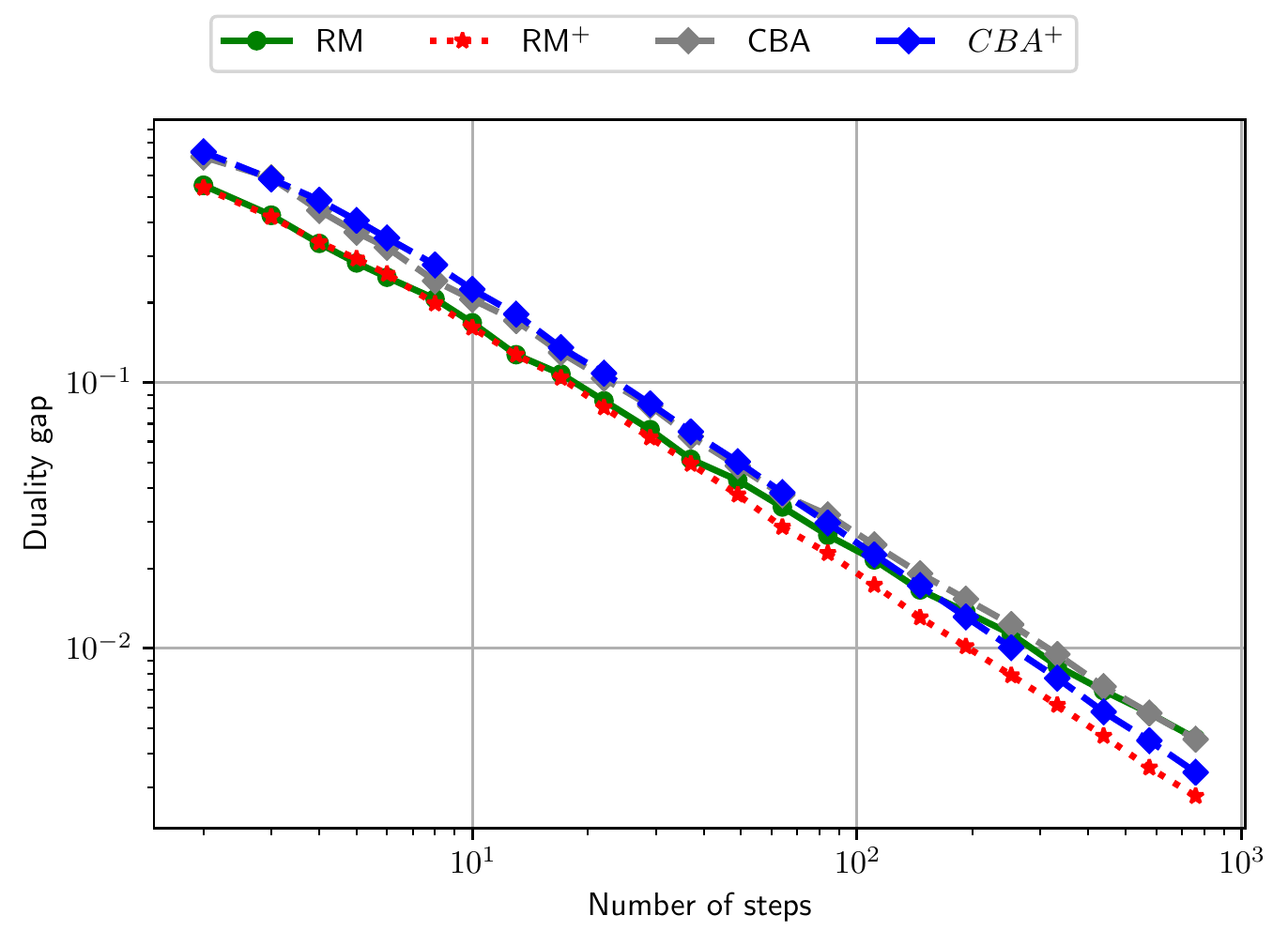}
         \caption{With alternation but no linear averaging.}\label{fig:subfig:comp-simplex-with-alt-normal}
           \end{subfigure}
     \begin{subfigure}{0.24\textwidth}
\centering
         \includegraphics[width=1.0\linewidth]{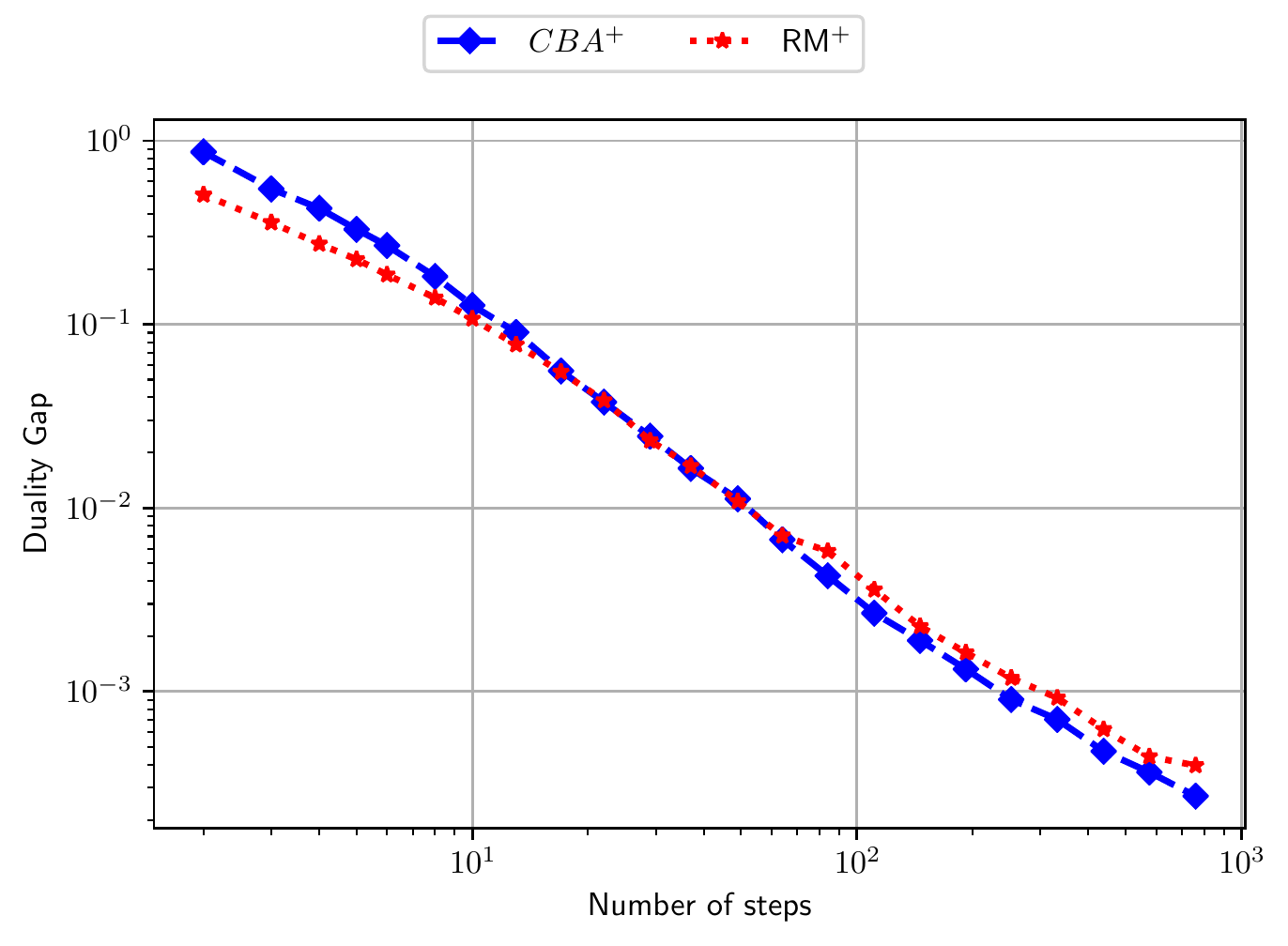}
         \caption{With alternation and linear averaging.}\label{fig:subfig:comp-simplex-with-lin-avg-normal}
  \end{subfigure}
  \caption{Comparison of \rmm, \rmp, \cba{} and \cbap{} for $\XX = \Delta(n),\YY=\Delta(m)$ and random matrices, with and without alternations, and with and without linear averaging. We choose $n,m=10$ and $A_{ij} \sim N(0,1)$ over $70$ instances.}
  \label{fig:two-player-simplex-normal}
\end{figure}

\tb{\subsection{Additional numerical experiments for matrix games}\label{app:simu-matrix-games}
We have seen in Appendix \ref{app:comparing-rm-cba} that \rmp{} and \cbap{} with alternation and linear averaging are outperforming \rmm{} and \cba{}. In Figure \ref{fig:subfig:matrix-uniform-parameters}, we have compared both \rmp{} and \cbap{} with AdaFTRL~\citep{orabona2015scale} and AdaHedge~\citep{de2014follow}, two algorithms that also enjoy the desirable \textit{scale-free} property, i.e., their sequences of decisions remain invariant when the losses are scaled by a constant factor. In the next figure, we provide additional comparisons of \rmp, \cbap, AdaHedge and AdaFTRL when the coefficients of the matrix of payoff are normally distributed. We found that \rmp{} and \cbap{} are both outperforming AdaHedge and AdaFTRL, a situation similar to the case of uniform payoffs (Figure \ref{fig:subfig:matrix-uniform-parameters}.)
\begin{figure}[H]
\center
\includegraphics[scale=0.5]{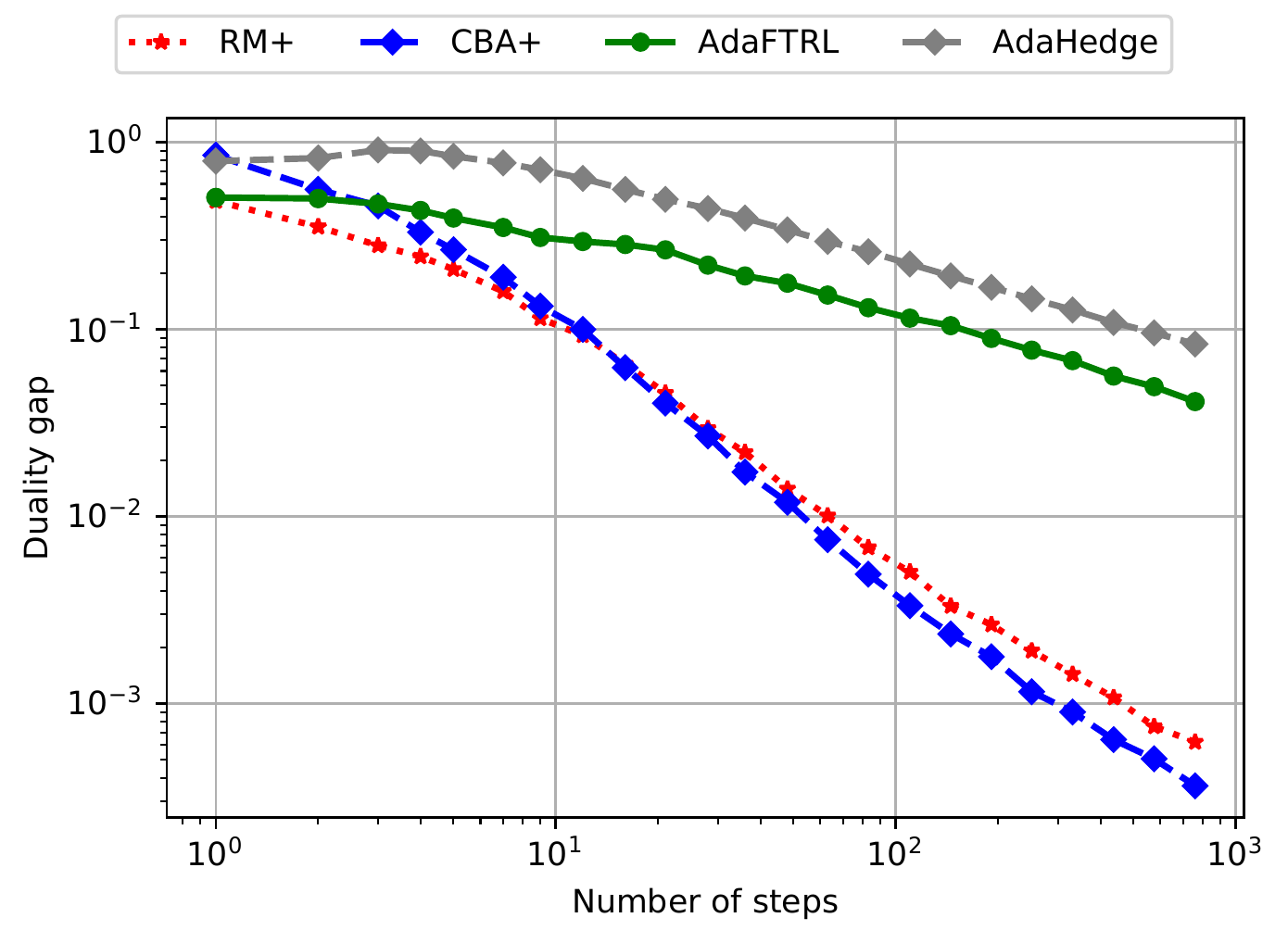}
\caption{Comparison of \rmp, \cbap, AdaHedge and AdaFTRL for $\XX = \Delta(n),\YY=\Delta(m)$ and random matrices. We choose $n,m=10$ and $A_{ij} \sim N(0,1)$ over $70$ instances.}
\label{fig:comparison-rmp-cbap-adahedge-adaftrl}
\end{figure}}
\subsection{Additional numerical experiments for EFGs}\label{app:simu-EFG} In Section \ref{sec:simu-bilinear-games-on-simplex} , we have compared \cbap{} (using alternation and linear averaging) and CFR$^{+}$ on various EFGs instances. We present in Figure \ref{fig:EFG-appendices} additional simulations where \cbap{} and CFR$^+$ performs similarly.  A description of the games can be found in \cite{farina2021faster}. On the $y$-axis we show the duality gap of the current averaged decisions $(\bar{\bm{x}}_{T},\bar{\bm{y}}_{T})$. On the $x$-axis we show the number of iterations. 

\begin{figure}[H]
 \begin{subfigure}{0.30\textwidth}
\centering
         \includegraphics[width=1.0\linewidth]{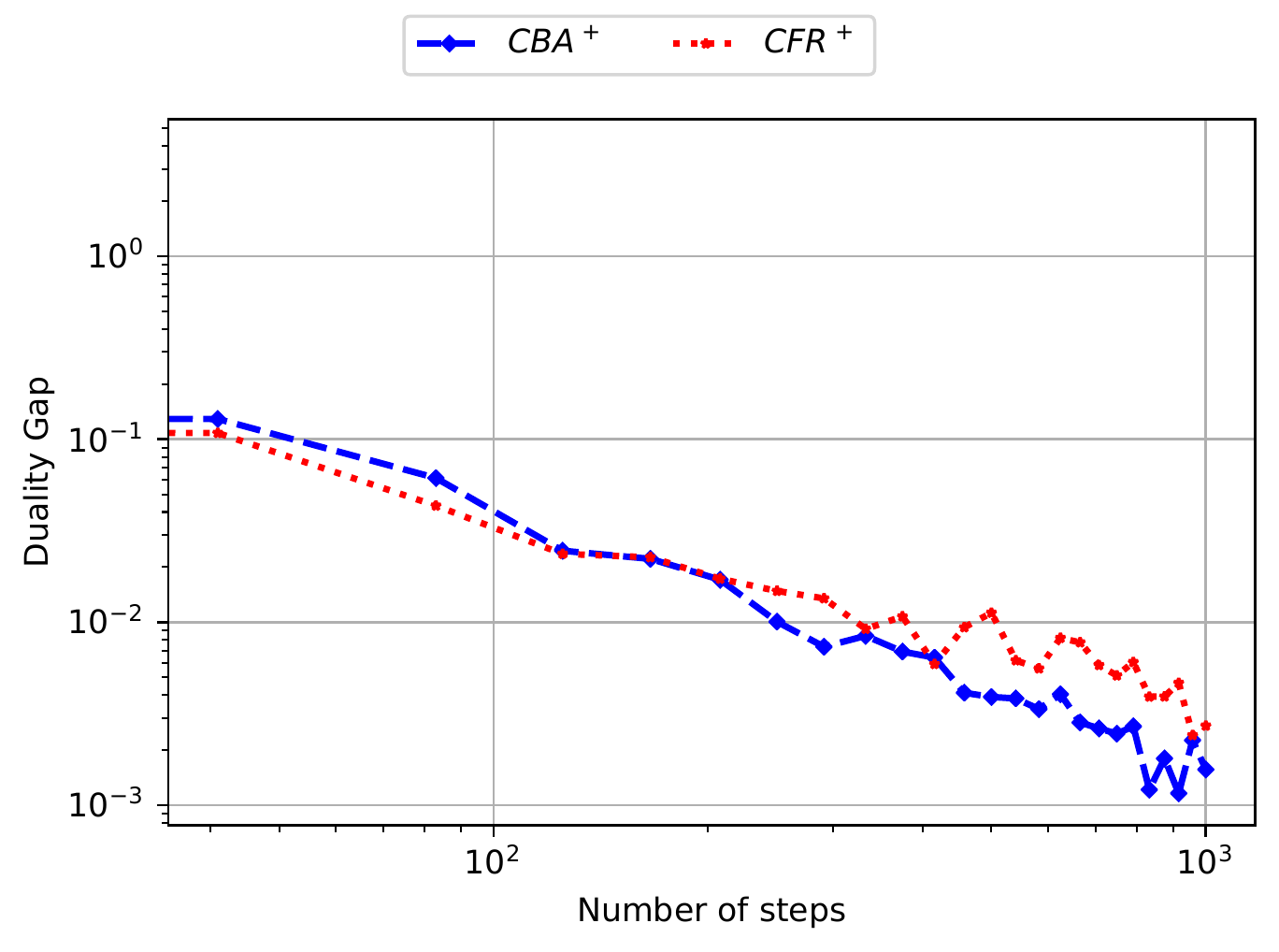}
         \caption{Search game, 4 turns.}
  \end{subfigure}
   \begin{subfigure}{0.30\textwidth}
\centering
         \includegraphics[width=1.0\linewidth]{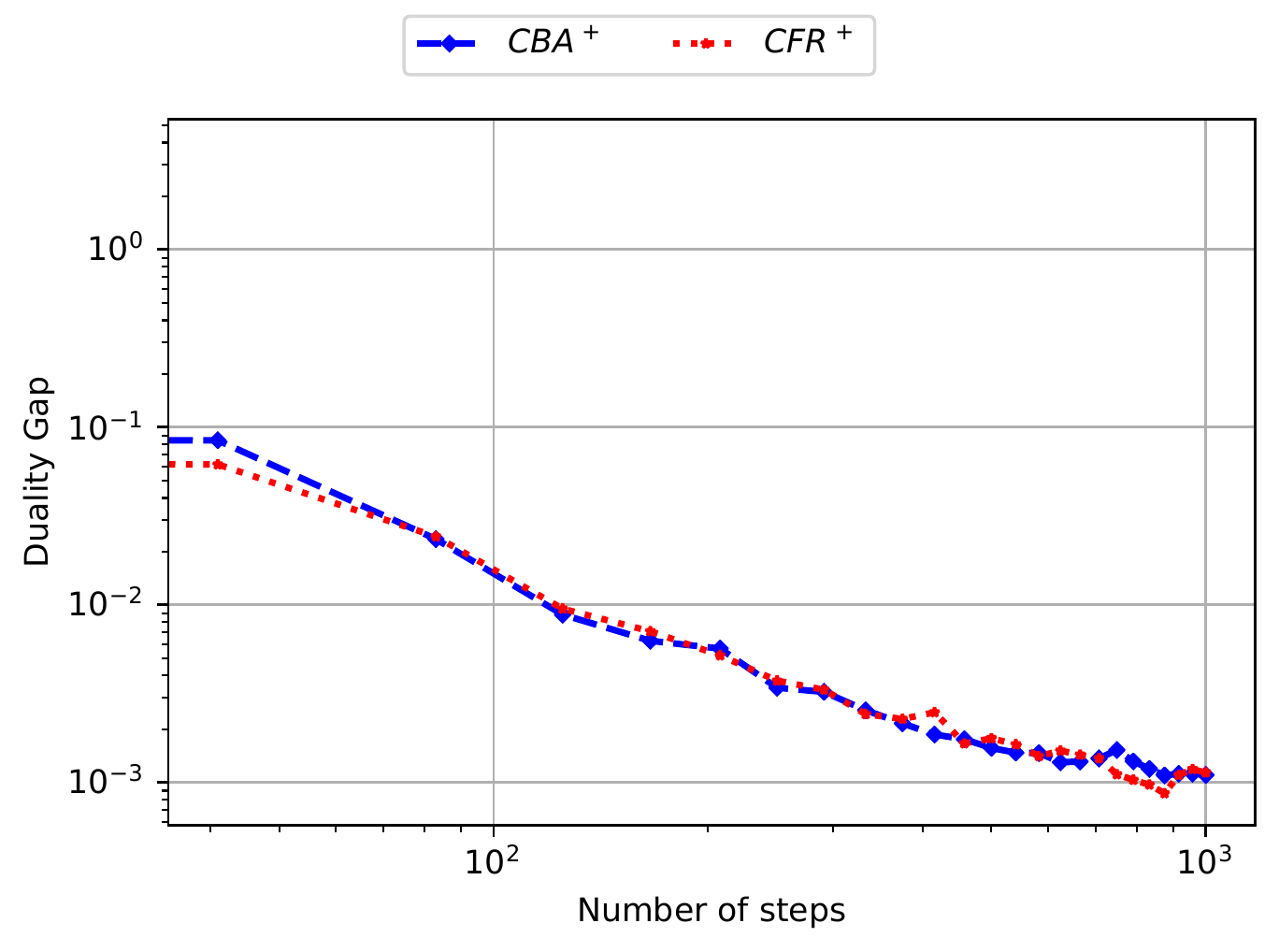}
         \caption{Search game, 5 turns.}
  \end{subfigure}
     \begin{subfigure}{0.30\textwidth}
\centering
         \includegraphics[width=1.0\linewidth]{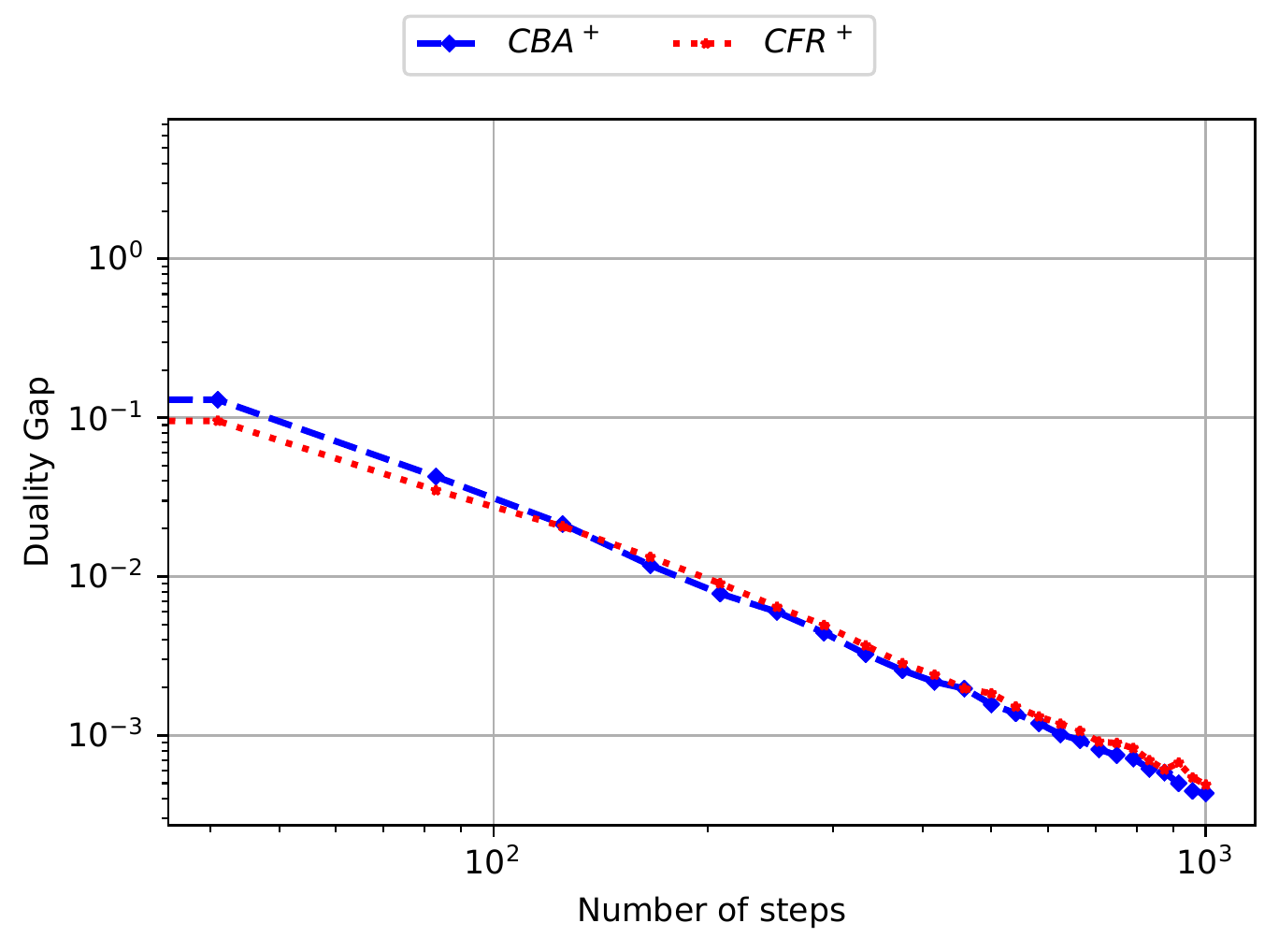}
         \caption{Leduc, 2 players, 3 ranks.}
  \end{subfigure}
       \begin{subfigure}{0.30\textwidth}
\centering
         \includegraphics[width=1.0\linewidth]{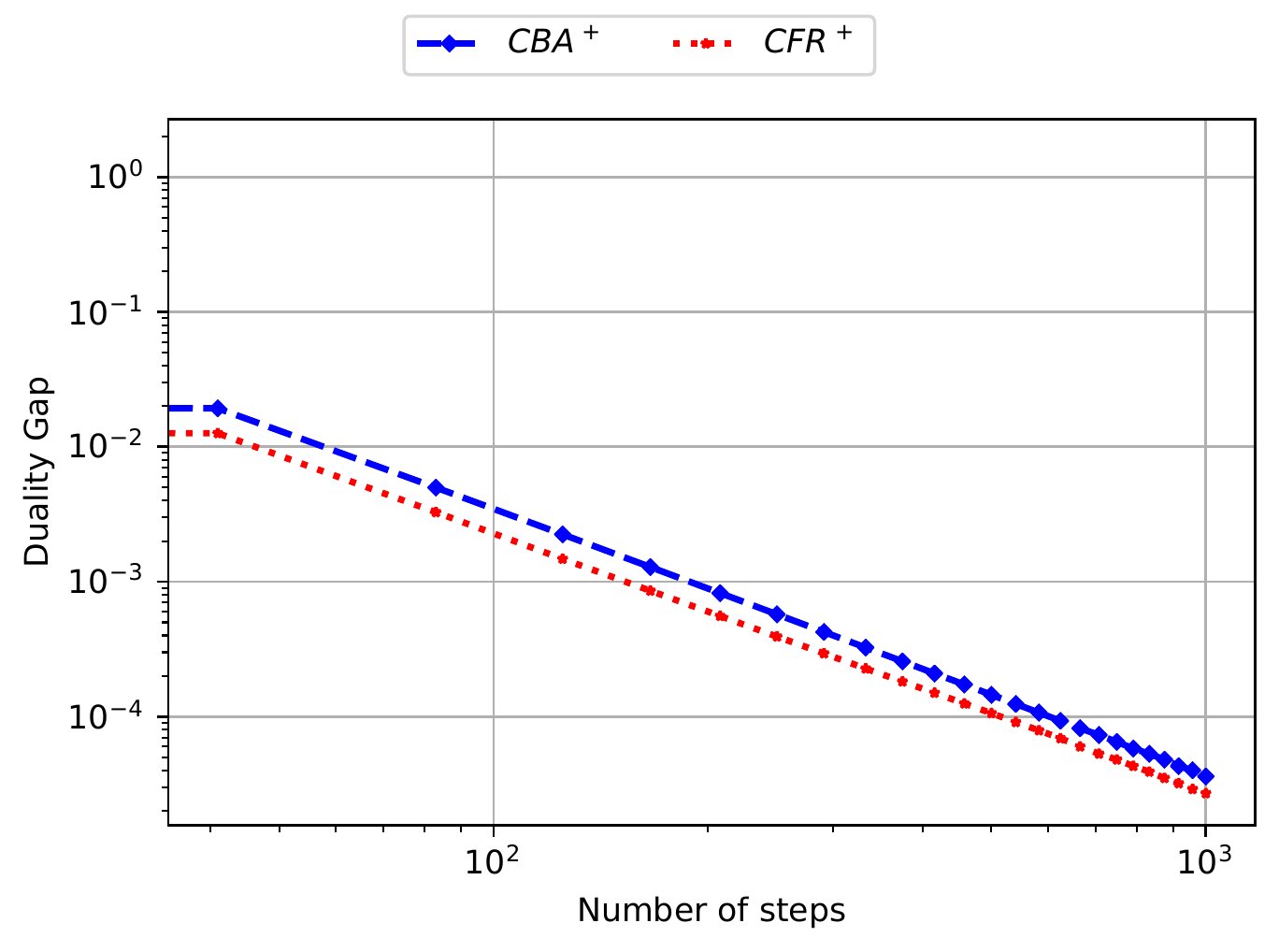}
         \caption{Leduc, 2 players, 6 faces.}
  \end{subfigure}
  \caption{Comparison of \cbap{} and CFR$^+$ for various Extensive-Form Games (EFG) instances.}
  \label{fig:EFG-appendices}
\end{figure} 
\section{OMD, FTRL and optimistic variants}\label{app:OMD-etc}
\subsection{Algorithms}
Let us fix some step size $\eta >0$.
For solving our instances of distributionally robust optimization, we compare Algorithm \cbap{} with the following four state-of-the-art algorithms: 
\begin{enumerate}
\item Follow-The-Regularized-Leader (FTRL) \citep{abernethy2009competing,mcmahan2011follow}:
\begin{equation}\label{alg:FTRL}\tag{FTRL}
\bm{x}_{t+1} \in \arg \min_{ \bm{x} \in \XX } \; \langle \sum_{\tau=1}^{t} \bm{f}_{\tau},\bm{x}\rangle + \dfrac{1}{ \eta } \| \bm{x} \|_{2}^{2}.
\end{equation}
Optimistic FTRL \citep{rakhlin2013online}: given estimation $\bm{m}^{t+1}$ of loss at period $t+1$, choose
\begin{equation}\label{alg:pred-FTRL}\tag{O-FTRL}
\bm{x}_{t+1} \in \arg \min_{ \bm{x} \in \XX } \; \langle \sum_{\tau=1}^{t} \bm{f}_{\tau} + \bm{m}^{t+1},\bm{x}\rangle + \dfrac{1}{ \eta } \|  \bm{x} \|_{2}^{2}.
\end{equation}
\item Online Mirror Descent (OMD) \citep{nemirovsky1983problem,beck2003mirror}:
\begin{equation}\label{alg:OMD}\tag{OMD}
\bm{x}_{t+1} \in \min_{\bm{x} \in \XX } \langle \bm{f}_{t},\bm{x} \rangle + \dfrac{1}{ \eta } \| \bm{x}-\bm{x}_{t}\|_{2}^{2}.
\end{equation}
 Optimistic OMD \citep{chiang2012online}: given estimation $\bm{m}^{t+1}$ of loss at period $t+1$, 
 \begin{equation}\label{alg:pred-OMD}\tag{O-OMD}
  \begin{aligned}
 &\bm{z}_{t+1}  \in \min_{\bm{z} \in \XX } \langle \bm{m}_{t+1},\bm{z} \rangle + \dfrac{1}{ \eta } \| \bm{z}-\bm{x}_{t}\|_{2}^{2}, \\
 & \text{Observe the loss }\bm{f}_{t+1} \text{ related to } \bm{z}_{t+1}, \\
 & \bm{x}_{t+1}  \in \min_{\bm{x} \in \XX } \langle \bm{f}_{t+1},\bm{x} \rangle + \dfrac{1}{ \eta} \| \bm{x} - \bm{x}_{t}\|_{2}^{2}.
 \end{aligned}
 \end{equation}
\end{enumerate}
Note that a priori these algorithms can be written more generally using Bregman divergence (e.g., \cite{BenTal-Nemirovski}). We choose to work with $\| \cdot \|_{2}$ instead of \textit{Kullback-Leibler divergence} as this $\ell_{2}$-setup is usually associated with faster empirical convergence rates \citep{ChambollePock16,GKG20}. Additionally,  following \cite{chiang2012online,rakhlin2013online}, we use the last observed loss as the predictor for the next loss, i.e., we set $\bm{m}^{t+1} = \bm{f}_{t}$.
\subsection{Implementations}\label{app:OMD-implementation}
When $\XX$ is the simplex or a ball based on the $\ell_{2}$-distance and centered at $\bm{0}$, there is a closed-form solution to the proximal updates for \ref{alg:FTRL}, \ref{alg:OMD}, \ref{alg:pred-FTRL} and \ref{alg:pred-OMD}. However, it is not clear how to compute these proximal updates for different settings, e.g., when $\XX$ is a subset of the simplex or an $\ell_{p}$-ball. We present the details of our implementation below. The results in the rest of this section are reminiscient to the novel tractable proximal setups presented in \cite{grand2020first,grand2020scalable}.

\paragraph{Computing the projection steps for min-player}
For $\XX = \{ \bm{x} \in \R^{n} \; | \; \| \bm{x}  - \bm{x}_{0} \|_{2} \leq \epsilon_{x}\}$, $\bm{c}, \bm{x}' \in \R^{n}$ and a step size $\eta>0$, the prox-update becomes 
\begin{equation}\label{eq:prox-update-x-player}
\min_{\| \bm{x} -\bm{x}_{0} \|_{2} \leq \epsilon_{x}} \langle \bm{c}, \bm{x} \rangle + \dfrac{1}{2 \eta} \| \bm{x} - \bm{x}'\|_{2}^{2}.
\end{equation}
This is the same arg min as
\[\min_{\| \bm{x} -\bm{x}_{0} \|_{2} \leq \epsilon_{x}} \| \bm{x} - (\bm{x}' - \eta \bm{c})\|_{2}^{2}.\]
We can change $\bm{x}$ by $\bm{z} = \left( \bm{x}-\bm{x}_{0} \right)/\epsilon_{x}$ to solve the equivalent program
\[ \min_{\| \bm{z} \|_{2} \leq 1 } \| \bm{z} - \dfrac{1}{\epsilon_{x}} \left( \bm{x}'-\eta \bm{c} - \bm{x}_{0} \right) \|_{2}^{2}.\]
The solution to the above program is
\[ \bm{z} = \dfrac{\bm{x}'-\eta \bm{c} - \bm{x}_{0}}{\max\{\epsilon_{x},\| \bm{x}'-\eta \bm{c} - \bm{x}_{0} \|_{2} \}}.\]
From $\bm{x} = \bm{x}_{0} + \epsilon_{x}\bm{z}$ we obtain $\bm{x}^{*}$ the solution to \eqref{eq:prox-update-x-player}
\[\bm{x}^{*} = \bm{x}_{0} + \epsilon_{x} \dfrac{\bm{x}'-\eta \bm{c} - \bm{x}_{0}}{\max\{\epsilon_{x},\| \bm{x}'-\eta \bm{c} - \bm{x}_{0} \|_{2} \}}.\]
\paragraph{Computing the projection steps for max-player}
For $\YY = \{ \bm{y} \in \Delta(m) \; | \; \| \bm{y}  - \bm{y}_{0} \|_{2} \leq \epsilon_{y}\}$, the proximal update of the max-player from a previous point $\bm{y}'$ and a step size of $\eta>0$ becomes
\begin{equation}\label{eq:prox-update-y-player}
\min_{\| \bm{y} -\bm{y}_{0} \|_{2} \leq \epsilon_{y}, \bm{y} \in \Delta(m)} \langle \bm{c}, \bm{y} \rangle + \dfrac{1}{2 \eta} \| \bm{y} - \bm{y}'\|_{2}^{2}.
\end{equation}
If we dualize the $\ell_{2}$ constraint with a Lagrangian multiplier $\mu \geq 0$ we obtain the relaxed problem $q(\mu)$ where
\begin{equation}\label{eq:q-mu}
q(\mu) = - (1/2) \epsilon_{y}^{2} \mu + \min_{\bm{y} \in \Delta(m)} \langle \bm{c}, \bm{y} \rangle + \dfrac{1}{2 \eta} \| \bm{y} - \bm{y}'\|_{2}^{2} + \dfrac{\mu}{2}\| \bm{y} - \bm{y}_{0} \|_{2}^{2}.
\end{equation}
Note that the $\arg \min$ in 
\[ \min_{\bm{y} \in \Delta(m)} \langle \bm{c}, \bm{y} \rangle + \dfrac{1}{2 \eta} \| \bm{y} - \bm{y}'\|_{2}^{2} + \dfrac{\mu}{2}\| \bm{y} - \bm{y}_{0} \|_{2}^{2} \]
is the same $\arg \min$ as in 
\begin{equation}\label{eq:intermediary-q-mu}
 \min_{\bm{y} \in \Delta(m)} \| \bm{y} - \dfrac{\eta}{\eta \mu + 1} \left( \dfrac{1}{\eta} \bm{y}' + \mu \bm{y}_{0} - \bm{c} \right) \|_{2}^{2}.
\end{equation}
Note that \eqref{eq:intermediary-q-mu} is an orthogonal projection onto the simplex.  Therefore, it can be solved efficiently \citep{euclidean-projection}. We call $\bm{y}(\mu)$ an optimal solution of \eqref{eq:intermediary-q-mu}. Then $q(\mu)$ can be rewritten
\[ q(\mu) = - (1/2) \epsilon_{y}^{2} \mu + \langle \bm{c}, \bm{y}(\mu) \rangle + \dfrac{1}{2 \eta} \| \bm{y}(\mu) - \bm{y}'\|_{2}^{2} + \dfrac{\mu}{2}\| \bm{y}(\mu) - \bm{y}_{0} \|_{2}^{2}.\]
We can therefore binary search $q(\mu)$ as in the previous expression.
An upper bound $\bar{\mu}$ for $\mu^{*}$ can be computed as follows. Note that 
\[ q(\mu) \leq - (1/2) \epsilon_{y}^{2} \mu + \langle \bm{c}, \bm{y}_{0} \rangle + \dfrac{1}{2 \eta} \| \bm{y}_{0}- \bm{y}'\|_{2}^{2}.\]
Since $\mu \mapsto q(\mu)$ is concave we can choose $\bar{\mu}$ such that $q(\mu) \leq q(0)$. Using the previous inequality this yields
\[ \bar{\mu} = \frac{2}{\epsilon_{y}^{2}} \left(  \langle \bm{c}, \bm{y}_{0} \rangle + \dfrac{1}{2 \eta} \| \bm{y}_{0}- \bm{y}'\|_{2}^{2} - q(0) \right).\]
We choose a precision of $\epsilon=0.001$ in our simulations. Note that these binary searches make \ref{alg:OMD}, \ref{alg:FTRL}, \ref{alg:pred-FTRL} and \ref{alg:pred-OMD} slower than \cbap{} in terms of running times, since the updates in \cbap{} only requires to compute the projection $\pi_{\C}(\bm{u})$, and we have shown in Proposition \ref{prop:complexity-computing-ortho-proj-norm-p} and Appendix \ref{app:proof-ell-p-ball} how to compute this in $O(n)$ when $\XX$ is an $\ell_{2}$ ball $\XX = \{ \bm{x} \in \Delta(n) \; | \; \| \bm{x} - \bm{x}_{0} \|_{2} \leq \epsilon_{x} \}$.

\paragraph{Computing the theoretical step sizes} We now give details about the choice of choice of theoretical step sizes.  In theory (e.g., \cite{BenTal-Nemirovski}),  for a player with decision set $\XX$, we can choose $\eta_{\sf th} = \sqrt{2} \Omega/L\sqrt{T}$ with $\Omega = \max_{\bm{x},\bm{x}' \in \XX} \| \bm{x}-\bm{x}'\|_{2}$, and $L$ an upper bound on the norm of any observed loss $ \bm{f}_{t}$: $ \| \bm{f}_{t} \|_{2} \leq L, \forall \; t \geq 1$. Note that this requires to know 1) the number of steps $T$, and 2) the upper bound $L$ on the norm of any observed loss $\bm{f}_{t}$, before the losses are generated.  
We now show how to compute $L_{x}$ and $L_{y}$ (for the $x$-player and the $y$-player) for an instance of the distributionally robust optimization problem \eqref{eq:DRO-empirical}.
\begin{enumerate}
\item For the $y$-player, the loss$\bm{f}_{t}$ is $\bm{f}_{t} = \left( \ell_{i}(\bm{x}_{t}) \right)_{i \in [1,m]}$, with $\ell_{i}(\bm{x}) = \log(1+\exp(-b_{i}\bm{a}^{\top}_{i}\bm{x}))$. 
For each $i \in[1,m]$ we have $| \ell_{i} | \leq \log(1+\exp( |b_{i}| R \| \bm{a}_{i} \|_{2} ))$ so that
\[L_{y} = \sqrt{\sum_{i=1}^{m} \log(1+\exp( |b_{i}| R \| \bm{a}_{i} \|_{2} ))^{2} }.\]
\item For the $x$-player we have $\bm{f}_{t} = \bm{A}^{t}\bm{y}_{t}$, where $\bm{A}^{t}$ is the matrix of subgradients of $\bm{x} \mapsto F(\bm{x},\bm{y}_{t})$ at $\bm{x}_{t}$:
\[ A^{t}_{ij} = \dfrac{-1}{1+\exp(b_{i}\bm{a}_{i}^{\top}\bm{x}_{t})}b_{i}a_{i,j}, \forall \; (i,j) \in \{ 1,...,m\} \times \{1,...,n\}.\]
Therefore, $\| \bm{f}_{t} \|_{2} \leq \| \bm{A}^{t} \|_{2} \| \bm{y}_{t} \|_{2} \leq \| \bm{A}^{t} \|_{2}$, because $\bm{y} \in \Delta(m)$. Now we have $\| \bm{A}^{t} \|_{2} \leq \| \bm{A}^{t} \|_{F} = \sqrt{\sum_{i,j} |A_{ij}^{t}|^{2}}$.  
From $|A^{t}_{ij}| \leq |b_{i}a_{i,j}|$ we use
\[ L_{x} = \sqrt{\sum_{i,j} |b_{i}a_{i,j}|^{2} }.\]
\end{enumerate}
%
\section{Additional details and numerical experiments for distributionally robust optimization}\label{app:details-simu}
We compare \cbap{} with alternation and linear averaging, \ref{alg:OMD},\ref{alg:FTRL},\ref{alg:pred-OMD} and \ref{alg:pred-OMD} for various step sizes $\eta$ where $\eta = \alpha \eta_{\sf th}$ for $\alpha \in \{1,100,1,000,10,000\}$, on additional synthetic and real data sets. We also add a comparison with adaptive step sizes.
\paragraph{Data sets} We present here the characteristics of the data sets that we use in our DRO simulations.  All data sets can be downloaded from the {\sf libsvm} classification libraries\footnote{https://www.csie.ntu.edu.tw/$\sim$cjlin/libsvmtools/datasets/}
\begin{itemize}
\item {\em Adult} data set: two classes,  $m=1,605$ samples with $n=123$ features.
\item {\em Australian} data set: two classes,  $m=690$ samples with $n=14$ features.
\item {\em Madelon} data set: two classes,  $m=2,000$ samples with $n=500$ features.
\item {\em Splice} data set: two classes,  $m=1,000$ samples with $n=60$ features.
\end{itemize}
\paragraph{Additional experiments with fixed step sizes}
In this section we present additional numerical experiments for solving distributionally robust optimization instances in Figure \ref{fig:simu-DRO-2}.
We use a synthetic data set, where we sample the features $a_{ij}$ as uniform random variables in $[0,1]$.   We also present results for the \textit{adult} and the \textit{australian} data sets from libsvm.  We vary the aggressiveness of the step sizes $\eta = \alpha \eta_{\sf th}$ by multiplying the theoretical step sizes $\eta_{\sf th}$ by a multiplicative step factor $\alpha$. The empirical setting is the same as in Section \ref{sec:simu}.
We note that our algorithm still outperforms or performs on par with the classical approaches after $10^{2}$ iterations, without requiring a single choice of parameter. 
\begin{figure}
\includegraphics[width=1.0\linewidth]{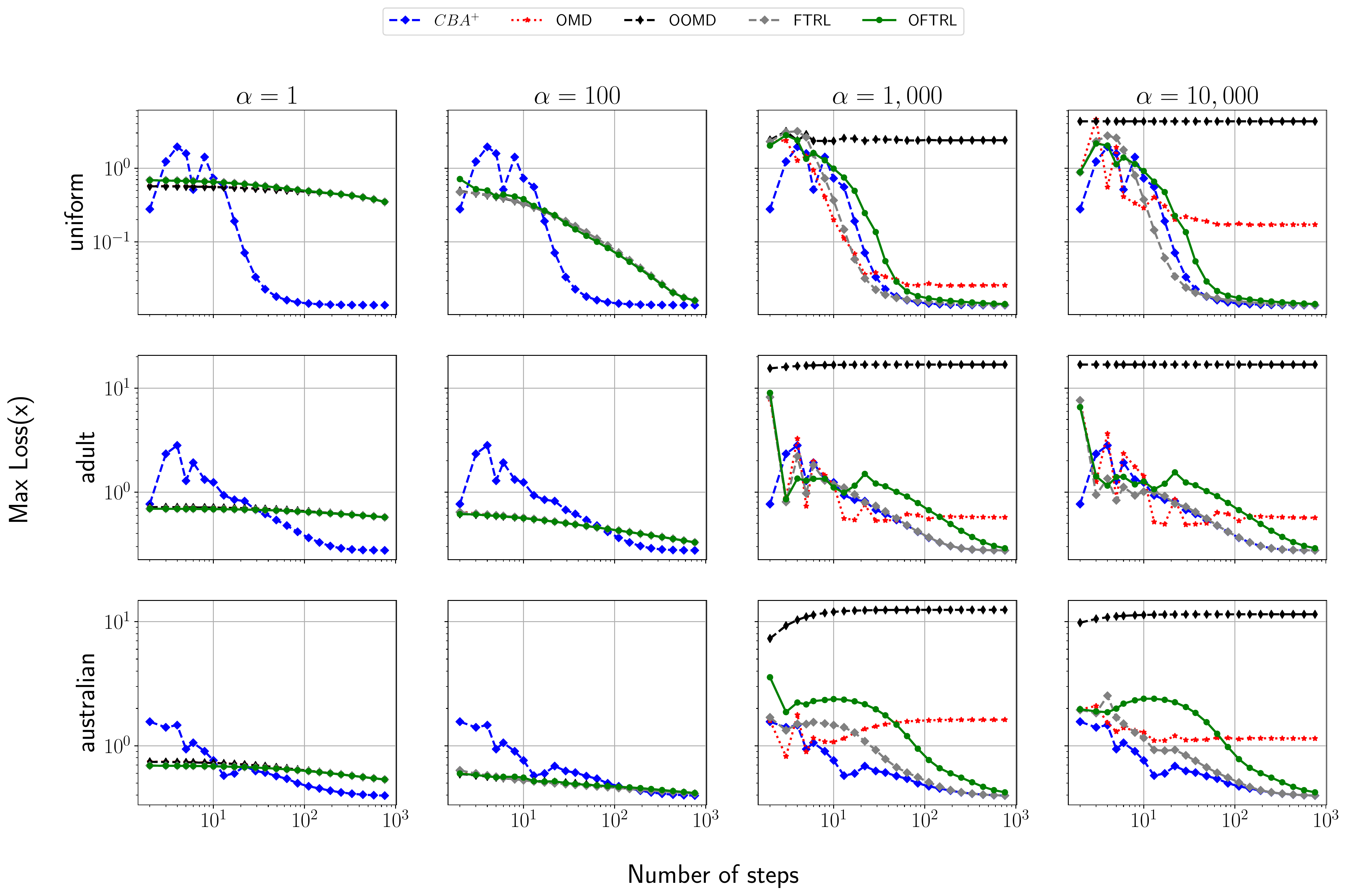}
\caption{Comparisons of the performances of \cbap{} with alternation and linear averaging, \ref{alg:OMD},\ref{alg:FTRL},\ref{alg:pred-OMD} and \ref{alg:pred-FTRL} on synthetic (with \textit{uniform} distribution) and real data sets (\textit{adult} and \textit{australian}). We use fixed step sizes $\eta = \alpha \eta_{\sf th}$, where $\eta_{\sf th}$ is the theoretical step size that guarantees convergence.}\label{fig:simu-DRO-2}
\end{figure}
\paragraph{Additional experiments with adaptive step sizes} We present our additional results with adaptive step sizes in Figure \ref{fig:DRO-adaptive}.  Given $\bm{v}_{t}$ the payoff observed by the player at period $t$, and following~\cite{orabona2019modern}, we choose the step sizes $\left( \eta_{t} \right)_{t \geq 1}$ as
\begin{equation}\label{eq:eta-adaptive}
\eta_{t} = 1/\sqrt{\sum_{\tau=1}^{t} \| \bm{v}_{\tau} \|_{2}^{2}}.
\end{equation}
 We note that \cbap{} still outperforms, or performs on par, with the state-of-the-art approaches.

\begin{figure}[hbt]
   \begin{subfigure}{0.3\textwidth}
\centering
         \includegraphics[width=1.0\linewidth]{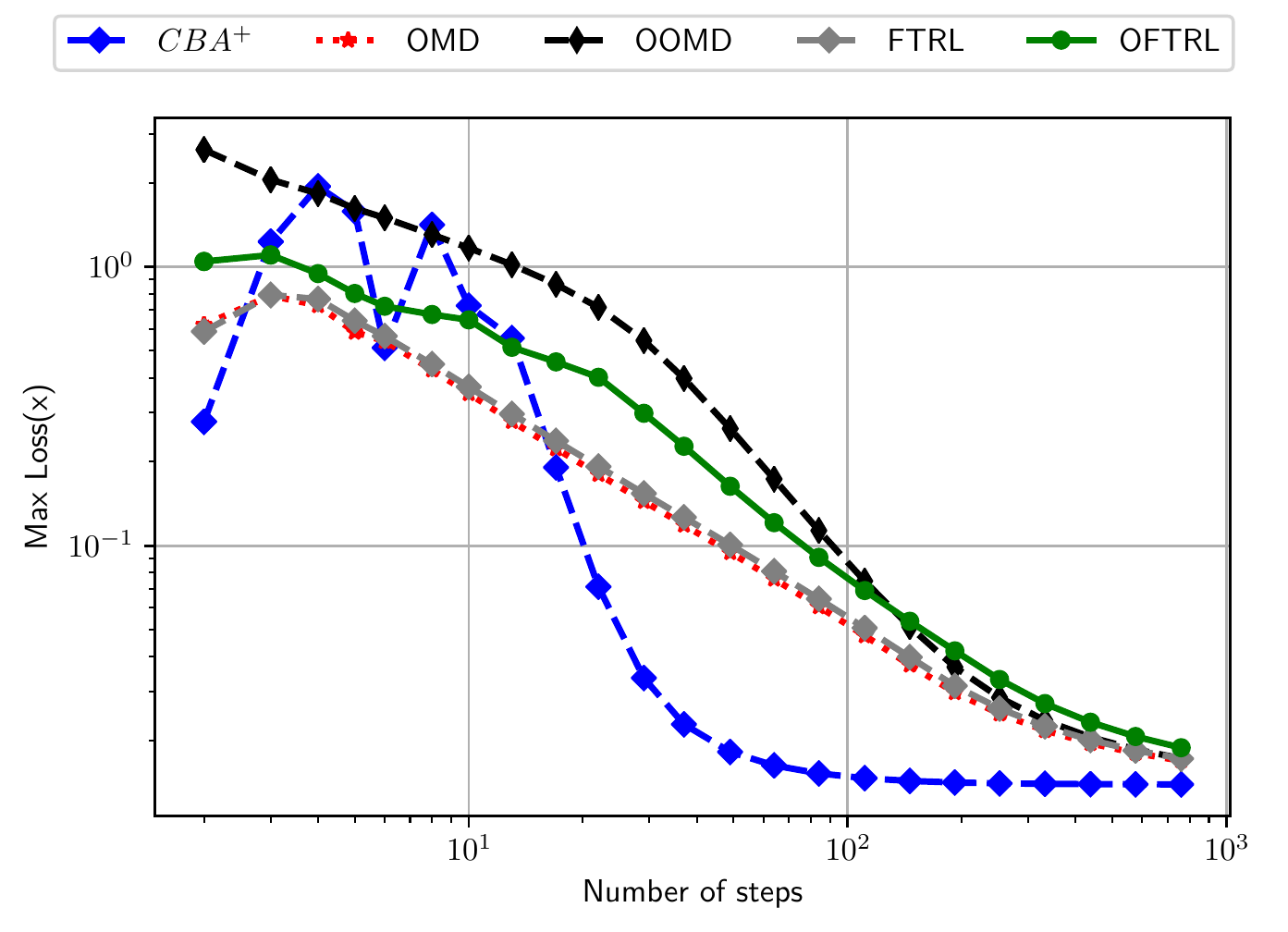}
         \caption{Uniform.}
  \end{subfigure}
   \begin{subfigure}{0.3\textwidth}
\centering
         \includegraphics[width=1.0\linewidth]{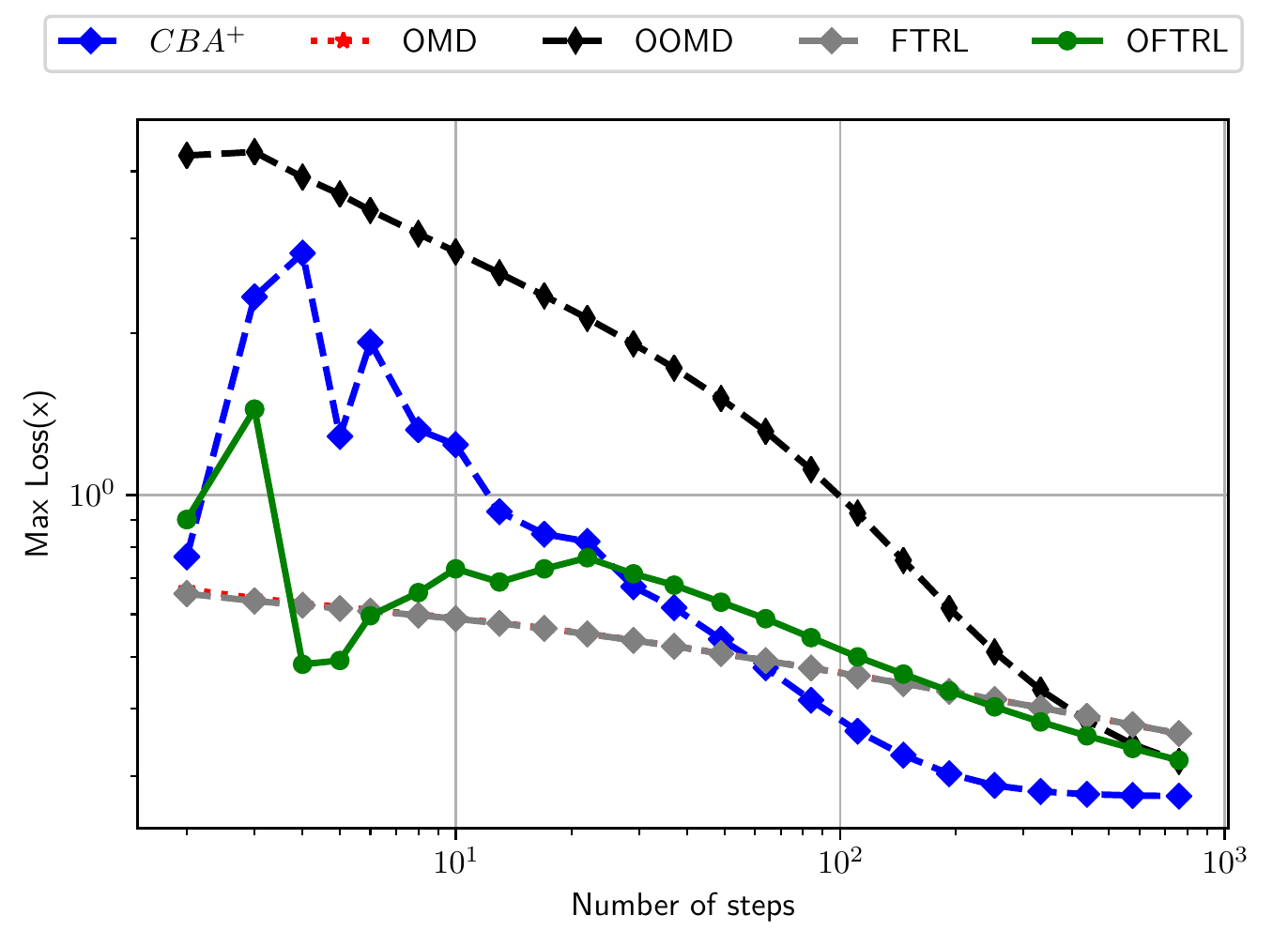}
         \caption{Adult}
  \end{subfigure}
   \begin{subfigure}{0.3\textwidth}
\centering
         \includegraphics[width=1.0\linewidth]{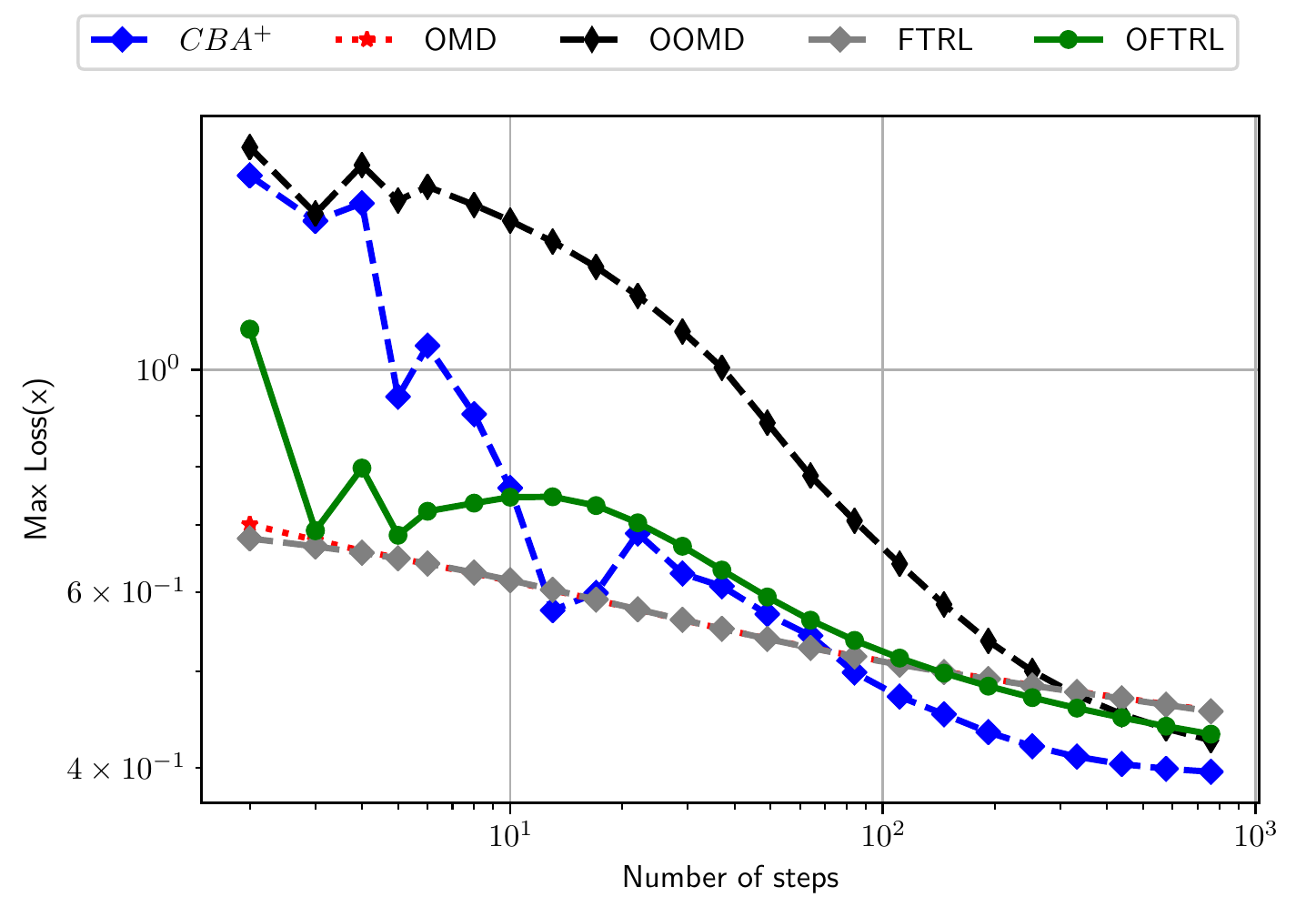}
         \caption{Australian}
  \end{subfigure}
  \caption{Comparisons of the performances of \cbap{} with alternation and linear averaging, \ref{alg:OMD},\ref{alg:FTRL},\ref{alg:pred-OMD} and \ref{alg:pred-FTRL} on synthetic (with \textit{uniform} distribution) and real data sets (\textit{adult} and \textit{australian}).  For the non-parameter free algorithms, we use the adaptive step sizes as in \eqref{eq:eta-adaptive}.}
  \label{fig:DRO-adaptive}
\end{figure}

%% file: main_conic_blackwell.bbl
\begin{thebibliography}{53}
\providecommand{\natexlab}[1]{#1}
\providecommand{\url}[1]{\texttt{#1}}
\expandafter\ifx\csname urlstyle\endcsname\relax
  \providecommand{\doi}[1]{doi: #1}\else
  \providecommand{\doi}{doi: \begingroup \urlstyle{rm}\Url}\fi

\bibitem[Abernethy et~al.(2011)Abernethy, Bartlett, and
  Hazan]{abernethy2011blackwell}
Jacob Abernethy, Peter~L Bartlett, and Elad Hazan.
\newblock Blackwell approachability and no-regret learning are equivalent.
\newblock In \emph{Proceedings of the 24th Annual Conference on Learning
  Theory}, pages 27--46. JMLR Workshop and Conference Proceedings, 2011.

\bibitem[Abernethy et~al.(2009)Abernethy, Hazan, and
  Rakhlin]{abernethy2009competing}
Jacob~D Abernethy, Elad Hazan, and Alexander Rakhlin.
\newblock Competing in the dark: An efficient algorithm for bandit linear
  optimization.
\newblock 2009.

\bibitem[Beck and Teboulle(2003)]{beck2003mirror}
Amir Beck and Marc Teboulle.
\newblock Mirror descent and nonlinear projected subgradient methods for convex
  optimization.
\newblock \emph{Operations Research Letters}, 31\penalty0 (3):\penalty0
  167--175, 2003.

\bibitem[Ben-Tal and Nemirovski(2001)]{BenTal-Nemirovski}
Aharon Ben-Tal and Arkadi Nemirovski.
\newblock \emph{Lectures on modern convex optimization: analysis, algorithms,
  and engineering applications}, volume~2.
\newblock Siam, 2001.

\bibitem[Ben-Tal et~al.(2013)Ben-Tal, Den~Hertog, De~Waegenaere, Melenberg, and
  Rennen]{ben2013robust}
Aharon Ben-Tal, Dick Den~Hertog, Anja De~Waegenaere, Bertrand Melenberg, and
  Gijs Rennen.
\newblock Robust solutions of optimization problems affected by uncertain
  probabilities.
\newblock \emph{Management Science}, 59\penalty0 (2):\penalty0 341--357, 2013.

\bibitem[Ben-Tal et~al.(2015)Ben-Tal, Hazan, Koren, and Mannor]{ben2015oracle}
Aharon Ben-Tal, Elad Hazan, Tomer Koren, and Shie Mannor.
\newblock Oracle-based robust optimization via online learning.
\newblock \emph{Operations Research}, 63\penalty0 (3):\penalty0 628--638, 2015.

\bibitem[Bertsimas et~al.(2019)Bertsimas, den Hertog, and
  Pauphilet]{bertsimas2019probabilistic}
Dimitris Bertsimas, Dick den Hertog, and Jean Pauphilet.
\newblock Probabilistic guarantees in robust optimization.
\newblock 2019.

\bibitem[Blackwell(1956)]{blackwell1956analog}
David Blackwell.
\newblock An analog of the minimax theorem for vector payoffs.
\newblock \emph{Pacific Journal of Mathematics}, 6\penalty0 (1):\penalty0 1--8,
  1956.

\bibitem[Bowling et~al.(2015)Bowling, Burch, Johanson, and
  Tammelin]{bowling2015heads}
Michael Bowling, Neil Burch, Michael Johanson, and Oskari Tammelin.
\newblock Heads-up limit hold’em poker is solved.
\newblock \emph{Science}, 347\penalty0 (6218):\penalty0 145--149, 2015.

\bibitem[Brown and Sandholm(2018)]{brown2018superhuman}
Noam Brown and Tuomas Sandholm.
\newblock Superhuman {AI} for heads-up no-limit poker: {Libratus} beats top
  professionals.
\newblock \emph{Science}, 359\penalty0 (6374):\penalty0 418--424, 2018.

\bibitem[Brown and Sandholm(2019)]{brown2019superhuman}
Noam Brown and Tuomas Sandholm.
\newblock Superhuman {AI} for multiplayer poker.
\newblock \emph{Science}, 365\penalty0 (6456):\penalty0 885--890, 2019.

\bibitem[Burch et~al.(2019)Burch, Moravcik, and Schmid]{burch2019revisiting}
Neil Burch, Matej Moravcik, and Martin Schmid.
\newblock Revisiting cfr+ and alternating updates.
\newblock \emph{Journal of Artificial Intelligence Research}, 64:\penalty0
  429--443, 2019.

\bibitem[Chambolle and Pock(2011)]{ChambollePock2011}
Antonin Chambolle and Thomas Pock.
\newblock A first-order primal-dual algorithm for convex problems with
  applications to imaging.
\newblock \emph{Journal of mathematical imaging and vision}, 40\penalty0
  (1):\penalty0 120--145, 2011.

\bibitem[Chambolle and Pock(2016)]{ChambollePock16}
Antonin Chambolle and Thomas Pock.
\newblock On the ergodic convergence rates of a first-order primal--dual
  algorithm.
\newblock \emph{Mathematical Programming}, 159\penalty0 (1-2):\penalty0
  253--287, 2016.

\bibitem[Chiang et~al.(2012)Chiang, Yang, Lee, Mahdavi, Lu, Jin, and
  Zhu]{chiang2012online}
Chao-Kai Chiang, Tianbao Yang, Chia-Jung Lee, Mehrdad Mahdavi, Chi-Jen Lu, Rong
  Jin, and Shenghuo Zhu.
\newblock Online optimization with gradual variations.
\newblock In \emph{Conference on Learning Theory}, pages 6--1. JMLR Workshop
  and Conference Proceedings, 2012.

\bibitem[Combettes and Reyes(2013)]{combettes2013moreau}
Patrick~L Combettes and Noli~N Reyes.
\newblock Moreau’s decomposition in banach spaces.
\newblock \emph{Mathematical Programming}, 139\penalty0 (1):\penalty0 103--114,
  2013.

\bibitem[De~Rooij et~al.(2014)De~Rooij, Van~Erven, Gr{\"u}nwald, and
  Koolen]{de2014follow}
Steven De~Rooij, Tim Van~Erven, Peter~D Gr{\"u}nwald, and Wouter~M Koolen.
\newblock Follow the leader if you can, hedge if you must.
\newblock \emph{The Journal of Machine Learning Research}, 15\penalty0
  (1):\penalty0 1281--1316, 2014.

\bibitem[Duchi et~al.(2008)Duchi, Shalev-Shwartz, Singer, and
  Chandra]{euclidean-projection}
John Duchi, Shai Shalev-Shwartz, Yoram Singer, and Tushar Chandra.
\newblock Efficient projections onto the {L}-1 ball for learning in high
  dimensions.
\newblock In \emph{Proceedings of the 25th international conference on Machine
  learning}, pages 272--279, 2008.

\bibitem[Duchi et~al.(2021)Duchi, Glynn, and Namkoong]{duchi2021statistics}
John~C Duchi, Peter~W Glynn, and Hongseok Namkoong.
\newblock Statistics of robust optimization: A generalized empirical likelihood
  approach.
\newblock \emph{Mathematics of Operations Research}, 2021.

\bibitem[Egozcue et~al.(2003)Egozcue, Pawlowsky-Glahn, Mateu-Figueras, and
  Barcelo-Vidal]{egozcue2003isometric}
Juan~Jos{\'e} Egozcue, Vera Pawlowsky-Glahn, Gl{\`o}ria Mateu-Figueras, and
  Carles Barcelo-Vidal.
\newblock Isometric logratio transformations for compositional data analysis.
\newblock \emph{Mathematical Geology}, 35\penalty0 (3):\penalty0 279--300,
  2003.

\bibitem[Farina et~al.(2019{\natexlab{a}})Farina, Kroer, and
  Sandholm]{farina2019online}
Gabriele Farina, Christian Kroer, and Tuomas Sandholm.
\newblock Online convex optimization for sequential decision processes and
  extensive-form games.
\newblock In \emph{Proceedings of the AAAI Conference on Artificial
  Intelligence}, volume~33, pages 1917--1925, 2019{\natexlab{a}}.

\bibitem[Farina et~al.(2019{\natexlab{b}})Farina, Kroer, and
  Sandholm]{farina2019optimistic}
Gabriele Farina, Christian Kroer, and Tuomas Sandholm.
\newblock Optimistic regret minimization for extensive-form games via dilated
  distance-generating functions.
\newblock In \emph{Advances in Neural Information Processing Systems}, pages
  5222--5232, 2019{\natexlab{b}}.

\bibitem[Farina et~al.(2019{\natexlab{c}})Farina, Kroer, and
  Sandholm]{farina2019regret}
Gabriele Farina, Christian Kroer, and Tuomas Sandholm.
\newblock Regret circuits: Composability of regret minimizers.
\newblock In \emph{International Conference on Machine Learning}, pages
  1863--1872, 2019{\natexlab{c}}.

\bibitem[Farina et~al.(2021)Farina, Kroer, and Sandholm]{farina2021faster}
Gabriele Farina, Christian Kroer, and Tuomas Sandholm.
\newblock Faster game solving via predictive blackwell approachability:
  Connecting regret matching and mirror descent.
\newblock In \emph{Proceedings of the AAAI Conference on Artificial
  Intelligence}. AAAI, 2021.

\bibitem[Gao et~al.(2019)Gao, Kroer, and Goldfarb]{GKG20}
Yuan Gao, Christian Kroer, and Donald Goldfarb.
\newblock Increasing iterate averaging for solving saddle-point problems.
\newblock \emph{arXiv preprint arXiv:1903.10646}, 2019.

\bibitem[Gordon(2007)]{gordon2007no}
Geoffrey~J Gordon.
\newblock No-regret algorithms for online convex programs.
\newblock In \emph{Advances in Neural Information Processing Systems}, pages
  489--496. Citeseer, 2007.

\bibitem[Goyal and Grand-Cl{\'e}ment(2018)]{GGC}
Vineet Goyal and Julien Grand-Cl{\'e}ment.
\newblock Robust {M}arkov decision process: Beyond rectangularity.
\newblock \emph{arXiv preprint arXiv:1811.00215}, 2018.

\bibitem[Grand-Cl{\'e}ment and Kroer(2020{\natexlab{a}})]{grand2020first}
Julien Grand-Cl{\'e}ment and Christian Kroer.
\newblock First-order methods for {W}asserstein distributionally robust {MDP}.
\newblock \emph{arXiv preprint arXiv:2009.06790}, 2020{\natexlab{a}}.

\bibitem[Grand-Cl{\'e}ment and Kroer(2020{\natexlab{b}})]{grand2020scalable}
Julien Grand-Cl{\'e}ment and Christian Kroer.
\newblock Scalable first-order methods for robust mdps.
\newblock \emph{arXiv preprint arXiv:2005.05434}, 2020{\natexlab{b}}.

\bibitem[Hart and Mas-Colell(2000)]{hart2000simple}
Sergiu Hart and Andreu Mas-Colell.
\newblock A simple adaptive procedure leading to correlated equilibrium.
\newblock \emph{Econometrica}, 68\penalty0 (5):\penalty0 1127--1150, 2000.

\bibitem[Iyengar(2005)]{Iyengar}
Garud Iyengar.
\newblock Robust dynamic programming.
\newblock \emph{Mathematics of Operations Research}, 30\penalty0 (2):\penalty0
  257--280, 2005.

\bibitem[Jin and Sidford(2020)]{jin2020efficiently}
Yujia Jin and Aaron Sidford.
\newblock Efficiently solving mdps with stochastic mirror descent.
\newblock In \emph{International Conference on Machine Learning}, pages
  4890--4900. PMLR, 2020.

\bibitem[Kroer(2020)]{kroer2020ieor8100}
Christian Kroer.
\newblock Ieor8100: Economics, ai, and optimization lecture note 5: Computing
  {Nash} equilibrium via regret minimization.
\newblock 2020.

\bibitem[Kroer et~al.(2018)Kroer, Farina, and Sandholm]{kroer2018solving}
Christian Kroer, Gabriele Farina, and Tuomas Sandholm.
\newblock Solving large sequential games with the excessive gap technique.
\newblock In \emph{Advances in Neural Information Processing Systems}, pages
  864--874, 2018.

\bibitem[Kroer et~al.(2020)Kroer, Waugh, K{\i}l{\i}n{\c{c}}-Karzan, and
  Sandholm]{kroer2018faster}
Christian Kroer, Kevin Waugh, Fatma K{\i}l{\i}n{\c{c}}-Karzan, and Tuomas
  Sandholm.
\newblock Faster algorithms for extensive-form game solving via improved
  smoothing functions.
\newblock \emph{Mathematical Programming}, pages 1--33, 2020.

\bibitem[McMahan(2011)]{mcmahan2011follow}
Brendan McMahan.
\newblock Follow-the-regularized-leader and mirror descent: Equivalence
  theorems and l1 regularization.
\newblock In \emph{Proceedings of the Fourteenth International Conference on
  Artificial Intelligence and Statistics}, pages 525--533. JMLR Workshop and
  Conference Proceedings, 2011.

\bibitem[Morav{\v{c}}{\'\i}k et~al.(2017)Morav{\v{c}}{\'\i}k, Schmid, Burch,
  Lis{\`y}, Morrill, Bard, Davis, Waugh, Johanson, and
  Bowling]{moravvcik2017deepstack}
Matej Morav{\v{c}}{\'\i}k, Martin Schmid, Neil Burch, Viliam Lis{\`y}, Dustin
  Morrill, Nolan Bard, Trevor Davis, Kevin Waugh, Michael Johanson, and Michael
  Bowling.
\newblock Deepstack: Expert-level artificial intelligence in heads-up no-limit
  poker.
\newblock \emph{Science}, 356\penalty0 (6337):\penalty0 508--513, 2017.

\bibitem[Namkoong and Duchi(2016)]{namkoong2016stochastic}
Hongseok Namkoong and John~C Duchi.
\newblock Stochastic gradient methods for distributionally robust optimization
  with f-divergences.
\newblock In \emph{NIPS}, volume~29, pages 2208--2216, 2016.

\bibitem[Nemirovski(2004)]{nemirovski2004prox}
Arkadi Nemirovski.
\newblock Prox-method with rate of convergence {O}(1/t) for variational
  inequalities with lipschitz continuous monotone operators and smooth
  convex-concave saddle point problems.
\newblock \emph{SIAM Journal on Optimization}, 15\penalty0 (1):\penalty0
  229--251, 2004.

\bibitem[Nemirovski and Yudin(1983)]{nemirovsky1983problem}
Arkadi Nemirovski and David Yudin.
\newblock \emph{Problem complexity and method efficiency in optimization.}
\newblock 1983.

\bibitem[Orabona(2019)]{orabona2019modern}
Francesco Orabona.
\newblock A modern introduction to online learning.
\newblock \emph{arXiv preprint arXiv:1912.13213}, 2019.

\bibitem[Orabona and P{\'a}l(2015)]{orabona2015scale}
Francesco Orabona and D{\'a}vid P{\'a}l.
\newblock Scale-free algorithms for online linear optimization.
\newblock In \emph{International Conference on Algorithmic Learning Theory},
  pages 287--301. Springer, 2015.

\bibitem[Rahimian and Mehrotra(2019)]{rahimian2019distributionally}
Hamed Rahimian and Sanjay Mehrotra.
\newblock Distributionally robust optimization: A review.
\newblock \emph{arXiv preprint arXiv:1908.05659}, 2019.

\bibitem[Rakhlin and Sridharan(2013)]{rakhlin2013online}
Alexander Rakhlin and Karthik Sridharan.
\newblock Online learning with predictable sequences.
\newblock In \emph{Conference on Learning Theory}, pages 993--1019. PMLR, 2013.

\bibitem[Shimkin(2016)]{shimkin2016online}
Nahum Shimkin.
\newblock An online convex optimization approach to blackwell's
  approachability.
\newblock \emph{The Journal of Machine Learning Research}, 17\penalty0
  (1):\penalty0 4434--4456, 2016.

\bibitem[Sidford and Tian(2018)]{sidford2018coordinate}
Aaron Sidford and Kevin Tian.
\newblock Coordinate methods for accelerating l-infinity regression and faster
  approximate maximum flow.
\newblock In \emph{2018 IEEE 59th Annual Symposium on Foundations of Computer
  Science (FOCS)}, pages 922--933. IEEE, 2018.

\bibitem[Syrgkanis et~al.(2015)Syrgkanis, Agarwal, Luo, and
  Schapire]{syrgkanis2015fast}
Vasilis Syrgkanis, Alekh Agarwal, Haipeng Luo, and Robert~E Schapire.
\newblock Fast convergence of regularized learning in games.
\newblock \emph{arXiv preprint arXiv:1507.00407}, 2015.

\bibitem[Tammelin et~al.(2015)Tammelin, Burch, Johanson, and
  Bowling]{tammelin2015solving}
Oskari Tammelin, Neil Burch, Michael Johanson, and Michael Bowling.
\newblock Solving heads-up limit {Texas} hold'em.
\newblock In \emph{Twenty-Fourth International Joint Conference on Artificial
  Intelligence}, 2015.

\bibitem[Tseng(1995)]{tseng1995linear}
Paul Tseng.
\newblock On linear convergence of iterative methods for the variational
  inequality problem.
\newblock \emph{Journal of Computational and Applied Mathematics}, 60\penalty0
  (1-2):\penalty0 237--252, 1995.

\bibitem[{von}~Stengel(1996)]{stengel1996efficient}
Bernhard {von}~Stengel.
\newblock Efficient computation of behavior strategies.
\newblock \emph{Games and Economic Behavior}, 14\penalty0 (2):\penalty0
  220--246, 1996.

\bibitem[Wei et~al.(2020)Wei, Lee, Zhang, and Luo]{wei2020linear}
Chen-Yu Wei, Chung-Wei Lee, Mengxiao Zhang, and Haipeng Luo.
\newblock Linear last-iterate convergence in constrained saddle-point
  optimization.
\newblock In \emph{International Conference on Learning Representations}, 2020.

\bibitem[Wiesemann et~al.(2013)Wiesemann, Kuhn, and Rustem]{Kuhn}
W.~Wiesemann, D.~Kuhn, and B.~Rustem.
\newblock Robust {M}arkov decision processes.
\newblock \emph{Operations Research}, 38\penalty0 (1):\penalty0 153--183, 2013.

\bibitem[Zinkevich et~al.(2007)Zinkevich, Johanson, Bowling, and
  Piccione]{zinkevich2007regret}
Martin Zinkevich, Michael Johanson, Michael Bowling, and Carmelo Piccione.
\newblock Regret minimization in games with incomplete information.
\newblock In \emph{Advances in neural information processing systems}, pages
  1729--1736, 2007.

\end{thebibliography}
